\documentclass[10pt, fullpage]{article}

\usepackage{graphicx}
\usepackage{verbatim}
\usepackage{fullpage}
\usepackage{amssymb}
\usepackage{amsmath}
\usepackage{amsthm}
\usepackage{amsfonts}
\usepackage{algorithm}
\usepackage{ifthen}
\usepackage[noend]{algorithmic}
\usepackage{color}

\newtheorem{theorem}{Theorem}[section]

\newtheorem{lemma}[theorem]{Lemma}
\newtheorem{claim}[theorem]{Claim}

\newtheorem{problem}{Problem}
\newtheorem{corollary}[theorem]{Corollary}

\newtheorem{fact}[theorem]{Fact}

\def\reals{\mathbb{R}}
\def\R{\reals}
\def\N{\mathbb{N}}

\def\eps{\epsilon}

\def\poly{\mathrm{poly}}

\def\B{\mathbb{B}}
\def\C{\mathbb{C}}
\def\Z{\mathbb{Z}}

\def\argmax{\mathrm{argmax}}
\def\cum{\mathrm{cum}}
\def\Var{\mathrm{Var}}

\newcommand{\abs}[1]{\left|#1\right|}
\newcommand{\norm}[1]{\left\|#1\right\|}

\newcommand{\prob}[1]{{\sf Pr}\left(#1\right)}
\newcommand{\probsub}[2]{{\sf Pr}_{#1}\left(#2\right)}

\newcommand{\E}[1]{\mathbb{E}\left(#1\right)}

\newcommand{\ip}[2]{\left<{ #1},{#2}\right>}
\newcommand{\pard}[2]{\frac{\partial #1 }{\partial #2 }}

\newcommand{\EE}[2]{\mathbb{E}_{#1}\left(#2\right)}

\newcommand{\diag}[1]{\mathrm{diag}\left( #1 \right)}

\renewcommand{\det}[1]{\mathrm{det}\left( #1 \right)}
\renewcommand{\vec}[1]{\mathrm{vec}\left( #1 \right)}
\newcommand{\re}[1]{\operatorname{Re}\left( #1 \right)}
\newcommand{\im}[1]{\operatorname{Im}\left( #1 \right)}
\newcommand{\colspan}[1]{\operatorname{\mathsf{colspan}}\left( #1 \right)}
\newcommand{\set}[1]{\left\{ #1 \right\}}

\newcommand{\bE}[1]{\bar{\mathbb{E}}\left(#1\right)}
\newcommand{\bMm}{\bar{M}_{\mu}}
\newcommand{\bMl}{\bar{M}_{\lambda}}
\newcommand{\bP}{\bar{P}}

\newcommand{\placeholder}{K}

\newcommand{\Vr}[1]{\mathrm{Var}\left(#1\right)}
\newcommand{\Conjug}[1]{}
\newcommand{\suchthat}{\;\ifnum\currentgrouptype=16 \middle\fi|\;}

\newcommand{\angles}[1]{\left\langle #1 \right\rangle}

\newcommand{\expn}[1]{\mathrm{exp}\left(#1\right)}
\newcommand{\mytilde}{\raise.17ex\hbox{$\scriptstyle\mathtt{\sim}$}}
\newcommand{\nE}{\mathbb{E}}

\title{\LARGE\bf Fourier PCA and Robust Tensor Decomposition}
\author{Navin Goyal
\thanks{Microsoft Research India}
\and 
Santosh Vempala
\thanks{Schools of Computer Science, Georgia Tech and CMU}
\and 
Ying Xiao
\thanks{School of Computer Science, Georgia Tech}
}

\begin{document}
\maketitle

\begin{abstract}
  Fourier PCA is Principal Component Analysis of a matrix obtained from
higher order derivatives of the logarithm of the Fourier transform of a 
distribution.  We make this method algorithmic by 
 developing a tensor decomposition method for a pair of tensors sharing the 
same vectors in rank-$1$ decompositions. Our main application is 
the first provably polynomial-time algorithm for underdetermined ICA, i.e., learning an $n \times m$  matrix $A$ from
observations $y=Ax$ where $x$ is drawn from an unknown product distribution with arbitrary non-Gaussian components. The number of component distributions $m$ can be arbitrarily higher than the dimension $n$ and the columns of $A$ only need to satisfy a natural and efficiently verifiable nondegeneracy condition.  As a second application, we give 
an alternative algorithm for learning mixtures of spherical Gaussians with linearly independent means. These results also hold in the presence of Gaussian noise.
\end{abstract}

\thispagestyle{empty} 

\newpage

\setcounter{page}{1}
\section{Introduction}

Principal Component Analysis~\cite{pearson1901liii} is often an
``unreasonably effective'' heuristic in practice, and some of its effectiveness can be explained rigorously as well (see, e.g.,
\cite{kannan2009spectral}). It consists of computing the eigenvectors
of the empirical covariance matrix formed from the data; the
eigenvectors turn out to be directions that locally maximize second
moments. The following example illustrates the power and limitations of PCA: given random independent points from a rotated cuboid in $\R^n$ with distinct axis lengths, PCA will identify the axes of the cuboid and their lengths as the eigenvectors and eigenvalues of the covariance matrix. However, if instead of a rotation, points came from a linear transformation of a cuboid, then PCA does not work. 

To handle this and similar hurdles,  higher moment extensions of PCA
have been developed in the literature e.g., \cite{AnandkumarTensorDecomp,hsu2008spectral,MosselR05,
  anandkumar2012method,anandkumar2012spectral,HsuK13,vx} and shown to be provably
effective for a wider range of unsupervised learning problems, including  
special cases of Independent Component Analysis (ICA), Gaussian mixture models, learning latent topic models etc. 
ICA is the classic signal recovery problem of learning a linear transformation $A$ from i.i.d. samples $x=As$ where $s \in \R^n$ has an unknown product distribution. The example above, namely learning a linearly transformed cuboid, is a special case of this problem. Although PCA fails, one can use it to first apply a transformation (to a sample) that makes the distribution isotropic, i.e., effectively making the distribution a rotation of a cube. At this point, eigenvectors give no further information, but as observed in the signal processing literature \cite{ComonJutten, ICAbook}, directions that locally maximize fourth moments reveal the axes of the cube, and undoing the isotropic transformation yields the axes of the original cuboid. Using this basic idea, Frieze et al.~\cite{FriezeJK96} and
subsequent papers give provably efficient algorithms assuming that the linear
transformation $A$ is full-dimensional and the components of the product distribution
differ from one-dimensional Gaussians in their fourth moment. This leaves open the important general case of {\em underdetermined} ICA, namely where $A$ is not necessarily square or full-dimensional, i.e., the observations $x$ are projections to a lower-dimensional space;  in the case of the cuboid example, we only see samples from an (unknown) projection of a transformed cuboid. 

In this paper, we
give a polynomial-time algorithm that (a) works for any transformation $A$ provided
the columns of the linear transformation satisfy a natural extension
of linear independence, (b) does not need the fourth moment assumption, and (c) is robust to Gaussian noise. 
As far as we know, this is the first
polynomial-time algorithm for underdetermined ICA. The central object of our study is a higher
derivative tensor of suitable functions of the Fourier transform of the
distribution. Our main algorithmic technique is an 
efficient tensor decomposition method for pairs of tensors that share the
same vectors in their respective rank-$1$ decompositions. We call our
general technique \emph{Fourier PCA}. This is motivated by the fact that in the base case of second derivatives, 
it reduces to PCA of a reweighted covariance matrix.

As a second application, Fourier PCA gives an alternative algorithm for learning a mixture of spherical Gaussians, 
under the assumption that the means of the component Gaussians are
linearly independent. Hsu and Kakade \cite{HsuK13} already gave an algorithm for this problem based on third moments; our algorithm has the benefit of being robust to Gaussian noise.

We now discuss these problems and prior work, then present our results in more detail.

\paragraph{ICA}
Blind source separation is a fundamental problem in diverse areas
ranging from signal processing to neuroscience to machine learning. In
this problem, a set of source signals are mixed in an unknown way, and
one would like to recover the original signals or understand how they
were mixed given the observations of the mixed signals. Perhaps the most influential formalization of this problem
is ICA (see the books \cite{ComonJutten, ICAbook} for comprehensive
introductions).  In the basic formulation of ICA, one has a random
vector $s \in \R^n$ (the source signal) whose components are
independent random variables with unknown distributions. Let $s^{(1)},
s^{(2)}, \ldots$ be independent samples of $s$.  One observes mixed
samples $As^{(1)}, As^{(2)}, \ldots $ obtained by mixing the
components of $s$ by an unknown invertible $n \times n$ mixing matrix
A. The goal is to recover $A$ to the extent possible, which would then
also give us $s^{(1)}, s^{(2)}, \ldots,$ or some approximations
thereof. One cannot hope to recover $A$ in
case more than one of the $s_i$ are Gaussian; in this case any set of 
orthogonal directions in this subspace would also be consistent with
the model. Necessarily, then, all ICA algorithms must require that the
component distributions differ from being Gaussian in some fashion.

A number of algorithms have been devised in the ICA
community for this problem. The literature is vast and we refer to \cite{ComonJutten} for
a comprehensive account. 
The ICA problem has been studied rigorously in theoretical computer
science in several previous papers \cite{FriezeJK96, NguyenR09,
  AroraGMS12, BelkinRV12, AnandkumarTensorDecomp}.  All of these
algorithms either assume that the component distribution is a very
specific one \cite{NguyenR09, AroraGMS12}, or assume that its kurtosis
(fourth cumulant) is bounded away from $0$, in effect 
assuming that its fourth moment is bounded away from that of a
Gaussian. The application of tensor decomposition to ICA 
has its
origins in work by Cardoso~\cite{Cardoso89}, and similar ideas were
later discovered by Chang \cite{chang1996full} in the context of
phylogenetic reconstruction and developed further in several works,
e.g. Mossel and Roch~\cite{MosselR05}, Anandkumar et
al.~\cite{anandkumar2012spectral}, Hsu and Kakade~\cite{HsuK13} for
various latent variable models. 
Arora et al.~\cite{AroraGMS12} and Belkin et
al.~\cite{BelkinRV12} show how to make the
algorithm resistant to unknown Gaussian noise. 

Underdetermined ICA, 
where the transformation matrix $A$ is not square
or invertible (i.e., it includes a projection),  is an important general problem and there are a number of algorithms proposed for it in the signal processing literature, many of them quite sophisticated. However, none of them is known to have rigourous guarantees on the sample or time complexity, even for special distributions.
See e.g. Chapter 9 of \cite{ComonJutten} for a review of 
existing algorithms and identifiability conditions for underdetermined ICA.  
For example, 
FOOBI \cite{cardoso91, dcc07} uses fourth-order correlations,
and its analysis is done only for the {\em exact} setting without
analyzing the robustness of the algorithm when applied to a sample,
and bounding the sample and time complexity for a desired level of
error. In addition, the known sufficient condition for the success of
FOOBI is stronger than ours (and more elaborate). 
We mention two other related papers \cite{ComonRajih, BIOME}.

\paragraph{Gaussian mixtures}

Gaussian mixtures are a popular model in statistics. A
distribution $F$ in $\R^n$ is modeled as a convex combination of
unknown Gaussian components. Given i.i.d. samples from $F$, the goal is
to learn all its parameters, i.e., the means, covariances and mixing
weights of the components. A classical result in statistics says that
Gaussian mixtures with distinct parameters are uniquely identifiable, i.e., as the number
of samples goes to infinity, there is unique decomposition of $F$ into
Gaussian components. It has been established that the sample
complexity grows exponentially in $m$, the number of components
\cite{BelkinCOLT10, BelkinFOCS10,
  kalai2010efficiently,moitra2010settling}. In a different direction,
under assumptions of separable components, a mixture is learnable in
time polynomial in all parameters
\cite{vempala2004spectral,dasgupta1999learning,
  sanjeev2001learning,dasgupta2007probabilistic,
  chaudhuri2008learning,brubaker2008isotropic}.  Our work here is
motivated by Hsu and Kakade's algorithm \cite{HsuK13}, which uses a
tensor constructed from the first three moments of the distribution and works 
for a mixture of spherical Gaussians with linearly independent means.

\paragraph{Robust tensor decomposition}
As a core subroutine for all problems above, we develop a general theory of efficient tensor decompositions for pairs
of tensors, which allows us to
recover a rank-$1$ tensor decomposition from two homogeneous tensor
relations. As noted in the literature, such a pair of tensor equations can be obtained from one tensor equation by applying two random vectors to the original equation, losing one in the order of the tensor. Our tensor decomposition
``flattens'' these tensors to matrices and performs an eigenvalue
decomposition. The matrices in question are not
Hermitian or even normal, and hence we use more general methods 
for eigendecomposition (in particular, tensor power
iterations cannot be used to find the desired decompositions). The algorithm for tensor decomposition via simultaneous tensor diagonalization is essentially due to Leurgans et al \cite{Leur1993}; to the best of our knowledge, ours is the first robust analysis. In subsequent work, Bhaskara et al. \cite{Bhaskara2013} have outlined a similar robustness analysis with a different application.

\subsection{Results}
 We begin with fully determined ICA.  Unlike
most of the literature on ICA, which employs moments, we do not
require that our underlying random variables $s_i$ differ from a
Gaussian at the fourth moment. In fact, our algorithm can deal with
differences from being Gaussian at any moment, though the
computational and sample complexities are higher when the differences
are at higher moments. We will use \emph{cumulants} as a notion of
difference from being a Gaussian. The cumulant of random variable $x$
at order $r$, denoted by $\cum_r(x)$, is the $r^{th}$ moment with some
additional subtractions of polynomials of lower moments. The following
is a short statement of our result for fully-determined ICA (i.e. the
mixing matrix $A$ is invertible); the full statement appears later as
Theorem~\ref{thm:ICA}. The algorithm takes as input the
samples generated according to the ICA model and parameters $\epsilon,
\Delta, M, k$ and an estimate of $\sigma_n(A)$.
\begin{theorem}
  Let $x \in \R^n$ be given by an ICA model $x=As$ where $A \in \R^{ n
    \times n}$ columns of $A$ have unit norm and let $\sigma_n(A) >
  0$. Suppose that for each $s_i$, there exists a $k_i \le k$ such
  that $\abs{\cum_{k_i}(s_i)} \ge \Delta > 0$ and $\E{\abs{s_i}^k} \le
  M$. Then, one can recover the columns of $A$ up to signs to $\eps$
  accuracy in polynomial time using $\poly( n^{k^2}, M^k, 1/\Delta^k,
  1/\sigma_n(A)^k, 1/\eps)$\\ samples with high probability.
\end{theorem}

 In the simplest setting, roughly speaking,
our algorithm computes the covariance matrix of the data reweighted
according to a random Fourier coefficient $e^{i u^T x}$ where $u \in
\R^n$ is picked according to a Gaussian distribution.  Our ideas are
inspired by the work of Yeredor \cite{DBLP:journals/sigpro/Yeredor00}, who presented such an algorithm for fully determined ICA (without finite sample guarantees).  

The reweighted covariance matrix can also be viewed as the Hessian of the logarithm of the Fourier transform of the distribution.
Using this perspective, we extend the method to underdetermined 
instances---problems where the apparent number of degrees of freedom seems
higher than the measurement system can uniquely fix, by studying 
higher derivative tensors of the Fourier transform.  The use of Fourier weights has the added advantage that the
resulting quantities always exist (this is the same phenomenon that
for a probability distribution the characteristic function always
exists, but the moment generating function may not) and thus works
for all random variables and not just in the case of having all
finite moments.

We then extend this to the setting where the source signal $s$ has
more components than the number of measurements (Section
\ref{sec:underdetermined}); recall that in this case, the transformation $A$ is a
map to a lower-dimensional space. 
Finding provably efficient algorithms for underdetermined ICA has been
an open problem. Tensor decomposition techniques, such as power iteration, which are known to work in the
fully determined case cannot possibly generalize to the
underdetermined case~\cite{aghkt12}, as they require linear
independence of the columns of $A$, which means that they can handle
at most $n$ source variables.

Our approach is based on tensor decomposition, usually defined as follows: given a tensor $T \in
\R^{n \times \cdots \times n}$ which has the following rank-$1$
expansion:
\begin{align*}
  T = \sum_{i=1}^m \mu_i A_i \otimes \cdots \otimes A_i,
\end{align*}
compute the vectors $A_i \in \R^n$. Here $\otimes$ denotes the usual
outer product so that entry-wise $[v \otimes \cdots \otimes
v]_{i_1,\ldots,i_r} = v_{i_1} \cdots v_{i_r}$). Our main idea here is
that we do not attempt to decompose a single tensor into its rank-$1$
components. This is an NP-hard problem in general, and to make it tractable, previous work uses additional informaton 
and structural assumptions, which do not hold in the underdetermined setting or place strong
restrictions on how large $m$ can be as a function of $n$. Instead, we consider 
\emph{two} tensors which share the same rank-$1$ components and
compose the tensors in a specific way, thereby extracting the desired
rank-$1$ components. In the following $\vec{ A_i^{\otimes d/2}}$
denotes the tensor $A_i^{\otimes d/2}$ flattened into a vector.
The algorithm's input consists of: tensors $T_\mu, T_\lambda$,  and parameters $n, m , d, \Delta, \epsilon$
as explained in the following theorem. 
\begin{theorem}[Tensor decomposition]\label{thm:decomposition}
  Let $A$ be an $n \times m$ matrix with $m > n$ and columns with unit
  norm, and let $T_{\mu},T_{\lambda} \in \R^{n \times \cdots \times
    n}$ be order $d$ tensors such that $d \in 2 \N$ and
  \begin{align*}
  T_{\mu} = \sum_{i=1}^m \mu_i A_i^{\otimes d} \qquad T_{\lambda}
  = \sum_{i=1}^m \lambda_i A_i^{\otimes d},
  \end{align*}
  where $\vec{ A_i^{\otimes d/2}}$ are linearly independent, $\mu_i,
  \lambda_i \ne 0$ and $\abs{\frac{\mu_i}{\lambda_i} -
    \frac{\mu_j}{\lambda_j}} > \Delta$ for all $i,j$ and $\Delta >
  0$. Then, algorithm TensorDecomposition$(T_\mu, T_\lambda)$ outputs
  vectors $A'_1, \ldots, A'_m$ with the following property. There is a
  permutation $\pi: [m] \to [m]$ and signs $\alpha: [m] \to
  \set{-1,1}$ such that for $i \in [m]$ we have
\begin{align*}
\norm{\alpha_i A'_{\pi(i)}-A_i}_2 \leq \epsilon.
\end{align*}
The running time  is $\mathrm{poly}\left(n^d, \frac{1}{\epsilon}, \frac{1}{\Delta}, \frac{1}{\sigma_{\min}(A^{\odot d/2})}\right)$. 
\end{theorem} 
The polynomial in the running time above can be made explicit. It
basically comes from the time complexity of SVD and eigenvector
decomposition of diagonalizable matrices. 
We note that 
in contrast to previous work on tensor decompositions ~\cite{harshman1970foundations,de2000best,carroll1970analysis,smilde2005multi}, our method has
provable finite sample guarantees. We give a robust version of the above, stated as Theorem \ref{thm:robustdecomposition}.

To apply this to underdetermined ICA, we 
form tensors associated with the ICA distribution as inputs to our
pairwise tensor decomposition algorithm. The particular tensors that
we use are the derivative tensors of the second characteristic
function evaluated at random points. 


Our algorithm can handle  
unknown Gaussian noise. The ICA model with Gaussian noise is given by
\begin{align*}
x = As + \eta,
\end{align*}
where $\eta \sim N(0, \Sigma)$ is independent Gaussian noise with
unknown general covariance matrix $\Sigma \in \R^{n \times n}$.  Also,
our result does not need full independence of the $s_i$, it is
sufficient to have $d$-wise independence.  The following is a short
statement of our result for underdetermined ICA; the full statement
appears later as Theorem~\ref{thm:main} (but without noise).  Its extension to handling
Gaussian noise is in Sec.~\ref{subsec:noise}.
The input to the algorithms, apart from the samples generated according to the unknown noisy underdetermined ICA model, consists of several parameters whose meaning will be clear in the theorem statement below: A tensor order parameter $d$, number of signals $m$, accuracy parameter $\epsilon$, confidence parameter $\delta$, bounds on moments and cumulants $M$ and $\Delta$, an estimate of the conditioning parameter $\sigma_m$, and moment order $k$. The notation 
$A^{\odot d}$ used below is explained in the preliminaries section; briefly, it's a $n^d \times m$ matrix
with each column obtained by flattening $A_i^{\otimes d}$ into a vector.

\begin{theorem}\label{thm:UICA_noisy}
Let $x \in \R^n$ be given by
  an underdetermined ICA model with unknown Gaussian noise $x= As + \eta$ where $A \in \R^{n \times m}$
  with unit norm columns and the covariance matrix $\Sigma \in \R^{n \times n}$
are unknown. Let $d \in 2 \N$ be such that $\sigma_m(A^{\odot d/2}) > 0$. Let $M_k, M_d, M_{2d}$ and 
$k > d$ be such that
 for each $s_i$, there is a $k_i$ satisfying $d < k_i < k$ and $\abs{\cum_{k_i}(s_i)} \ge
  \Delta$, and $\E{ \abs{s_i}^{k_i}} , \E{\sigma_1(\Sigma)^k}\le M_k$, $\E{\abs{s_i}^d} \leq M_d$, 
and $\E{\abs{s_i}^{2d}} \leq M_{2d}$.  
Then one can recover the columns of $A$ up to $\eps$ accuracy in 2-norm and up to the sign 
using $\text{poly}\left( m^{k}, M_d^k, M_{2d}, 1/ \Delta, 1/\sigma_m(A^{\odot d/2})^{k}, 1/\eps,  1/\sigma^k \right)$ samples and with similar 
polynomial time complexity with probability at least $3/4$, where 
$0 < \sigma < \frac{\Delta}{M_k} \text{poly}(\sigma_m(m^k, A^{\odot d/2})^k,1/k^k)$.


\end{theorem}

The probability of success of the algorithm can be boosted from $3/4$ to $1-\delta$ for any $\delta > 0$ by taking $O(\log (1/\delta))$  independent runs of the algorithm and using an adaptation of the ``median" trick (see e.g., Thm 2.8 in\cite{LV07}).
To our knowledge, this is the first polynomial-time
algorithm for underdetermined ICA with provable finite sample
guarantees. It works under mild assumptions on the input
distribution and nondegeneracy assumptions on the mixing matrix
$A$. The assumption says that the columns of the matrix when tensored up individually are linearly independent. 
For example, with $d=4$, suppose that every $s_i$ differs from a Gaussian in the fifth or higher moment
by $\Delta$, then we can recover all the components as
long as $\vec{ A_i A_i^T}$ are linearly independent. Thus, the number of components that can be recovered can be as high as 
$m = n(n+1)/2$. Clearly, this is a weakening of the standard assumption that the columns of $A$ are linearly independent. 
This assumption can be regarded as a certain incoherence type
assumption. Moreover, in a sense it's a necessary and sufficient condition: the ICA problem is solvable for matrix
$A$ if and only if any two columns are linear independent~\cite{ComonJutten}, and this turns out to be equivalent 
to the existence of a
finite $d$ such that $A^{\odot d}$ has full column rank. 
A well-known condition in the literature on tensor
decomposition is Kruskal's condition \cite{kruskal1977three}. Unlike
that condition it is easy to check if a matrix satisfies our
assumption (for a fixed $d$).  Our assumption is true {\em
  generically}: For a randomly chosen matrix $A \in \R^{n \times {n
    \choose d}}$ (e.g. each entry being i.i.d. standard Gaussian), we
have $\sigma_{min}(A^{\odot d}) > 0$ with probability $1$.  In a
similar vein, for a randomly chosen matrix $A \in \R^{n \times {n
    \choose d}}$ its condition number is close to $1$ with high
probability; see Theorem~\ref{thm:vershynin_condition} for a precise statement and proof.
Moreover, our assumption is robust also in the smoothed sense~\cite{AndersonGMM}: If we
start with an arbitrary matrix $M \in \R^{n \times {n \choose 2}}$ and perturb it with a noise matrix
$N \in \R^{n \times {n \choose 2}}$ with each entry independently chosen from $N(0, \sigma^2)$, then we have
$\sigma_{min}((M+N)^{\odot 2}) = \sigma^2/n^{O(1)}$ with probability
at least $1-1/n^{\Omega(1)}$, and a similar generalization holds for higher powers. 
This follows from a simple application
of the anti-concentration properties of polynomials in independent
random variables; see \cite{AndersonGMM} for a proof. See also~\cite{Bhaskara2013}.

As in the fully-determined ICA setting, we require that our random
variables have some cumulant different from a Gaussian. 
One aspect of our result is that using the $d^{th}$ derivative,
one loses the ability to detect non-Gaussian cumulants at order $d$ and
lower; on the other hand, a theorem of Marcinkiewicz
\cite{marcinkiewicz1939propriete} implies that this does not cause
a problem.
\begin{theorem}[Marcinkiewicz]
  Suppose that random variable $x \in \R$ has only a finite number of
  non-zero cumulants, then $x$ is distributed according to a Gaussian,
  and every cumulant of order greater than 2 vanishes.
\end{theorem}
Thus, even if we miss the difference in cumulants at order $i \le d$,
there is some higher order cumulant which is nonzero, and hence
non-Gaussian. Note also that for many specific instances of the ICA
problem studied in the literature, \emph{all} cumulants differ from
those of a Gaussian \cite{FriezeJK96, NguyenR09, AroraGMS12}.  

We remark that apart from direct practical interest of ICA in signal recovery, recently some new applications of ICA as an algorithmic primitive have been discovered. Anderson et al.~\cite{agr12} reduce some special cases of the 
problem of learning a convex body (coming from a class of convex bodies such as simplices), given uniformly
distributed samples from the body, to fully-determined ICA. Anderson et al.~\cite{AndersonGMM} solve the problem of learning Gaussian
mixture models in regimes for which there were previously no efficient algorithms known. This is done by reduction to underdetermined ICA using the results of our paper. 

Our final result applies the same method to learning mixtures of
spherical Gaussians (see the full version). Using Fourier PCA, we
recover the result of Hsu and Kakade \cite{HsuK13}, and extend it
to the setting of noisy mixtures, where the noise itself is an unknown
arbitrary Gaussian.  Our result can be viewed as saying that
reweighted PCA gives an alternative algorithm for learning such
mixtures.
\begin{theorem}\label{thm:lin-ind-mixtures}
  Fourier PCA for Mixtures applied to a mixture of $k < n$ spherical
  Gaussians $N(\mu_i, \sigma_i^2 I_n)$ 
  recovers the parameters of the mixture to desired accuracy $\eps$
  using time and samples polynomial in $k, n, 1/\eps$ with high
  probability, assuming that the means $\mu_i$ are linearly
  independent.
\end{theorem}


Thus, overall, our contributions can be viewed as two-fold. The first part is a robust, efficient tensor decomposition technique. 
The second is the analysis of the spectra of matrices/tensors arising from Fourier derivatives. In particular, showing that the eigenvalue
gaps are significant based on anticoncentration of polynomials in Gaussian space; and that these matrices, even when obtained from samples, remain diagonalizable.


\section{Preliminaries}

For positive integer $n$, the set $\{1, \ldots, n\}$ is denoted by
$[n]$. The set of positive even numbers is denoted by $2 \N$.

We assume for simplicity and without real loss of generality 
that $\E{s_j}=0$ for all $j$. We can ensure this by working with
samples $x^i-\bar{x}$ instead of the original samples $x^i$ (here
$\bar{x}$ is the empirical average of the samples). There is a slight
loss of generality because using $\bar{x}$ (as opposed to using $\E(x)$) 
introduces small
errors. These errors can be analysed along with the rest of the errors
and do not introduce any new difficulties.

\paragraph{Probability.}
For a random variable $x \in \R^n$ and $u \in \R^n$, its \emph{characteristic function}
$\phi: \R \to \C$ is defined by $\phi_x(u) = \EE{x}{e^{iu^Tx}}$. Unlike the
moment generating function, the characteristic function is
well-defined even for random variables without all moments finite. 
The \emph{second characteristic function} of $x$ is defined by
$\psi_x(u) := \log \phi_x(u)$, where we take that branch of the complex
logarithm that makes $\psi(0)=0$. In addition to random variable $x$ above we will also 
consider random variable $s \in \R^m$ related to $x$ via $x = As$ for $A \in \R^{n \times m}$
and the functions associated with it: the characteristic function $\phi_s(t) = \EE{s}{e^{it^Ts}}$
and the second characteristic function $\psi_s(t) = \log \phi_s(t)$.

Let $\mu_j := \E{x^j}$. 
\emph{Cumulants} of
$x$ are polynomials in the moments of $x$ which we now define. 
For $j \geq 1$, the $j$th cumulant is denoted $\cum_j(x)$.  Some
examples: $\cum_1(x) = \mu_1, \cum_2(x) = \mu_2 - \mu_1^2, \cum_3(x) =
\mu_3 - 3 \mu_2\mu_1 + 2\mu_1^3$. As can be seen from these examples
the first two cumulants are the same as the expectation and the
variance, resp. Cumulants have the property that for two independent
r.v.s $x, y$ we have $\cum_j(x+y) = \cum_j(x) + \cum_j(y)$ (assuming
that the first $j$ moments exist for both $x$ and $y$). The first two
cumulants of the standard Gaussian distribution have value $0$ and
$1$, and all subsequent cumulants have value $0$. Since ICA is not
possible if all the independent component distributions are Gaussians,
we need some measure of distance from the Gaussians of the component
distributions. A convenient measure turns out to be the distance from
$0$ (i.e. the absolute value) of the third or higher cumulants. If all the moments of $x$ exist,
then the second characteristic function admits a Taylor expansion in
terms of cumulants
\begin{align*}
\psi_x(u) = \sum_{j \geq 1} \cum_j(x) \frac{(iu)^j}{j!}.
\end{align*}
This can also be used to define cumulants of all orders. 

\paragraph{Matrices.} For a vector $\mu = (\mu_1, \ldots, \mu_n)$, let
$\diag{\mu}$ and $\diag{\mu_j}$, where $j$ is an index variable, denote 
the $n \times n$ diagonal matrix with the diagonal
entries given by the components of $\mu$. The singular values of an $m \times n$
matrix will always be ordered in the decreasing order: $\sigma_1 \geq
\sigma_2 \geq \ldots \geq \sigma_{\min(m, n)}$. Our matrices will often have rank $m$, and thus
the non-zero singular values will often, but not always, be $\sigma_1,
\ldots, \sigma_m$. The span of the columns vectors of a matrix $A$
will be denoted $\colspan{A}$.  The columns of a matrix $A$ are
denoted $A_1, A_2, \ldots$. The potentially ambiguous but convenient notation $A_i^T$ means $(A_i)^T$.  
The condition number of a matrix $A$ is
$\kappa(A):= \sigma_{\max}(A)/\sigma_{\min}(A)$, where 
$\sigma_{\max}(A) := \sigma_1(A)$ and $\sigma_{\min}(A) := \sigma_{\min(m, n)}(A)$.

\paragraph{Tensors and tensor decomposition.}
Here we introduce various tensor notions that we need; these are discussed in detail in the review
paper~\cite{KoldaBader}. 
An order $d$ tensor $T$ is an array indexed by $d$ indices each with
$n$ values (e.g., when $d=2$, then $T$ is
simply a matrix of size $n \times n$). Thus, it has $n^d$ entries. Tensors considered in this paper are symmetric,
i.e. $T_{i_1, ..., i_d}$ is invariant under permutations of $i_1, \ldots, i_d$. In the sequel we will generally not 
explicitly mention that our tensors are symmetric. We also note that symmetry of tensors is not essential for 
our results but 
for our application to ICA it suffices to look at only symmetric tensors and the results generalize easily to the
general case, but at the cost of additional notaton.

We can also view a tensor as a
degree-$d$ homogeneous form over vectors $u \in \R^n$ defined by $
T(u,\ldots,u) = \sum_{i_1,\ldots,i_d} T_{i_1,\ldots,i_d} u_{i_1}.
\cdots u_{i_d}$.  This is in analogy with matrices, where every matrix
$A$ defines a quadratic form, $u^T A u = A(u,u) = \sum_{i,j} A_{i,j}
u_i u_j$.  We use the outer product notation
\begin{align*}
  v^{\otimes d} = \underbrace{v \otimes \cdots \otimes
    v}_{\text{$d$ copies}},
\end{align*}
where entrywise we have $[v \otimes \cdots \otimes v]_{j_1,
  \ldots,j_d} = v_{j_1} \cdots v_{j_d}$. 
A (symmetric) rank-$1$ decomposition of a tensor $T_{\mu}$ is defined by
\begin{align}\label{eqn:tensor}
  T_{\mu} = \sum_{i=1}^m \mu_i A_i^{\otimes d},
\end{align}
where the $\mu_i \in \R$ are nonzero and the $A_i \in \R^n$ are
vectors which are not multiples of each other. Such a decomposition always exists for all 
symmetric tensors with $m < n^d$ (better bounds are known but we won't need them).
For example, for a symmetric
matrix, by the spectral theorem we have
\begin{align*}
  M = \sum_{i=1}^n \lambda_i v_i \otimes v_i.
\end{align*}
We will use the notion of flattening of tensors. Instead of giving a formal definition
it's more illuminating to give examples. E.g. for $d=4$, 
construct a bijection $\tau: [n^2] \to [n] \times [n]$ as $\tau(k) =
(\lfloor k/n \rfloor, k-\lfloor k/n \rfloor)$ and $\tau^{-1}(i,j) =
ni+j$. We then define a packing of a matrix $B \in \R^{n \times n}$ to
a vector $p \in \R^{n^2}$ by $ B_{\tau(k)} = p_{k}$.  For convenience
we will say that $B = \tau(p)$ and $p = \tau^{-1}(B)$. We also define
a packing of $T \in \R^{n \times n \times n \times n}$ to a matrix $M
\in \R^{n^2 \times n^2}$ by $ M_{a,b} = T_{\tau(a),\tau(b)}$, for $a, b \in [n^2]$. Note
that $M$ is symmetric because $T$ is symmetric with respect to all
permutations of indices: $ M_{a,b} = T_{\tau(a), \tau(b)} =
T_{\tau(b), \tau(a)} = M_{b,a}$. The definition of $\tau$ depends on the dimensions
and order of the tensor and what it's being flattened into; this will be
clear from the context and will not be further elaborated upon. 
Finally, to simplify the notation, we
will employ the Khatri-Rao power of a matrix: 
$A^{\odot d} := \left[\vec{A_1^{\otimes d}} | \vec{A_2^{\otimes d}} | \ldots |\vec{A_m^{\otimes d}}\right]$, 
where recall that $\vec{T}$ for a tensor $T$ is a 
flattening of $T$, i.e. we arrange the entries of $T$ in a single column vector.

\paragraph{Derivatives.} For $g: \R^n \rightarrow \R$ we will use abbreviation 
$\partial_{u_i} g(u_1, \ldots, u_n)$ for
$\pard{g(u_1, \ldots, u_n)}{u_i}$; when the variables are clear from
the context, we will further shorten this to $\partial_ig$. Similarly,
$\partial_{i_1, \ldots, i_k} g$ denotes $\partial_{i_1}(\ldots
(\partial_{i_k}g)\ldots)$, and for multiset $S = \{i_1, \ldots,
i_k\}$, this will also be denoted by $\partial_S g$, which makes sense because $\partial_{i_1, \ldots, i_k} g$ is 
invariant under reorderings of $i_1, \ldots, i_k$. We will not use any special notation for multisets;
what is meant will be clear from the context.

$D_u g(u)$ denotes the gradient vector $(\partial_{u_1}g(u), \ldots, \partial_{u_n}g(u))$, and 
$D_u^2g(u)$ denotes the Hessian matrix $[\partial_{u_i}\partial_{u_j}g(u)]_{ij}$. More generally, 
$D_u^dg(u)$ denotes the order $d$ derivative tensor given by 
$[D_u^dg(u)]_{i_1, \ldots, i_d}=\partial_{u_{i_1}}\ldots \partial_{u_{i_d}}g(u)$.

\paragraph{Derivatives and linear transforms.} We are particularly interested in how the derivative
tensor changes under linear transform of the arguments. We state things over the real field, but 
everything carries over to the complex field as well. 
Let $g: \R^n \rightarrow \R$ and 
$f: \R^m \rightarrow \R$ be two functions such that all the derivatives that we consider below exist. 
Let $A \in \R^{n \times m}$ and let variables $t \in \R^m$ and $u \in \R^n$ be related by linear
relation $t = A^T u$, and let the function $f$ and $g$ be related by $g(u) = f(A^Tu) = f(t)$. 
Then for $j \in [n]$

\begin{align*} 
\partial_{u_j}g(u) 
&= \partial_{u_j}f((A_1)^Tu, \ldots, (A_m)^Tu) \\
&= \frac{\partial (A_1)^Tu}{\partial u_j} \partial_{t_1} f(t) 
+ \ldots + \frac{\partial (A_m)^Tu}{\partial u_j} \partial_{t_m} f(t) \\
&= A_{j1}\partial_{t_1}f(t) + \ldots + A_{jm} \partial_{t_m}f(t) \\
&= \sum_{k \in [m]} A_{jk} \partial_{t_k}f(t).
\end{align*}

Applying the previous eqution twice for $i, j \in [n]$ gives
\begin{align*}
\partial_{u_i}\partial_{u_j} g(u) &= \partial_{u_i}\left(\sum_{k \in [m]} A_{jk}\partial_{t_k}f(t)\right) \\
&= \sum_{k \in [m]} A_{jk}\partial_{t_k}(\partial_{u_i}f(t)) \\
&= \sum_{k \in [m]}A_{jk} \sum_{\ell \in [m]}A_{il} \partial_{t_k}\partial_{t_\ell}f(t)\\
&=\sum_{\ell, k \in [m]} A_{i\ell}A_{jk}\partial_{t_\ell}\partial_{t_k}f(t),
\end{align*}
and applying it four times for $i_1, i_2, i_3, i_4 \in [n]$ gives
\begin{align}\label{eqn:fourth_derivative_linear_transform}
\partial_{u_{i_1}}\partial_{u_{i_2}}\partial_{u_{i_3}}\partial_{u_{i_4}}g(u) = 
\sum_{k_1, k_2, k_3, k_4 \in [m]} A_{i_1k_1}A_{i_2k_2}A_{i_3k_3}A_{i_4k_4} \partial_{t_{k_1}}\partial_{t_{k_2}}\partial_{t_{k_3}}\partial_{t_{k_4}}f(t). 
\end{align}

This can be written more compactly as a matrix equation
\begin{align*}
D_u^4g(u) = A^{\otimes 2} (D_t^4f(t)) (A^{\otimes 2})^T,
\end{align*}
where we interpret both $D_u^4g(u)$ and $D_t^4f(t)$ as appropriately flattened into matrices. 

A useful special case of this occurs when $f$ has the property that 
$\partial_{t_i}\partial_{t_j} f(t) = 0$ whenever $i \neq j$. 
In this case \eqref{eqn:fourth_derivative_linear_transform} can be rewritten as 
\begin{align*}
\partial_{u_{i_1}}\partial_{u_{i_2}}\partial_{u_{i_3}}\partial_{u_{i_4}}g(u) = 
\sum_{k \in [m]} A_{i_1k}A_{i_2k}A_{i_3k}A_{i_4k}\partial^4_{t_k}f(t),
\end{align*}
and in matrix notation
\begin{align*}
D^4_ug(u) = A^{\odot 2} \diag{\partial^4_{t_1}f(t), \ldots, \partial^4_{t_m}f(t)} (A^{\odot 2})^T, 
\end{align*}
where again we interpret $D^4_ug(u)$ as flattened into a matrix.

The previous equations readily generalize to higher derivatives. For $d \geq 1$, 
interpreting the tensors $D_u^{2d}g(u)$ and $D_t^{2d}f(t)$ as flattened into matrices, we have
\begin{align} 
D_u^{2d}g(u) = A^{\otimes d} (D_t^{2d}f(t)) (A^{\otimes d})^T, 
\end{align}
and if $f$ has the property that $\partial_{t_i}\partial_{t_j} f(t) = 0$ whenever $i \neq j$ then
\begin{align} 
D_u^{2d}g(u) = A^{\odot d} \diag{\partial^{2d}_{t_1}f(t), \ldots, \partial^{2d}_{t_m}f(t)} (A^{\odot d})^T. 
\end{align}

In our applications we will need to use the above equations for the case when $g(u) = \psi_x(u)$
and $f(t) = \psi_s(t)$ where these notions were defined at the beginning of this section. 
The above equations then become
\begin{align} \label{eqn:derivative_linear_transform_general}
D_u^{2d} \psi_x(u) = A^{\otimes d} (D_t^{2d}\psi_s(t)) (A^{\otimes d})^T.
\end{align}

In the special case when the components of $s$ are indpendent we have
$f(t) = \log \E{e^{it_1 s_1}} + \ldots + \log \E{e^{it_m s_m}}$ and so we have the property 
$\partial_{t_i}\partial_{t_j} \psi_s(t) = 0$ whenever $i \neq j$ and this gives
\begin{align} \label{eqn:derivative_linear_transform_diagonal}
D_u^{2d} \psi_x(u) 
= A^{\odot d} \diag{\partial^{2d}_{t_1}\psi_s(t), \ldots, \partial^{2d}_{t_1}\psi_s(t)} (A^{\odot d})^T.
\end{align}

\section{Algorithms}\label{sec:algorithms}
In this section, we present our main new algorithms and outline their analysis. For the reader's convenience, we will restate these algorithms in the sections where their analysis appears. As mentioned in the introduction, our ICA algorithm 
is based on a certain tensor decocmposition algorithm.

\subsection{Tensor decomposition} \label{subsec:outline_tensor_decomposition}
A fundamental result of linear algebra is that every
symmetric matrix has a spectral decomposition, which allows us to
write it as the sum of outer products of vectors: $A = \sum_{i=1}^n
\lambda_i v_i v_i^T$, and such representations are efficiently
computable.
Our goal, in analogy
with spectral decomposition for matrices, is to recover (symmetric) rank-$1$ decompositions of tensors. Unfortunately, there are no known
algorithms with provable guarantees when $m > n$, and in fact this
problem is NP-hard in general \cite{Bru09,hillar2009most}. It is an
open research question to characterize, or even give interesting
sufficient conditions, for when a rank-$1$ decomposition of a tensor 
$T$ as in \eqref{eqn:tensor} is unique 
and computationally tractable.
For the case $d=2$, a necessary and sufficient condition for uniqueness is that the
eigenvalues of $T$ are distinct. Indeed, when eigenvalues repeat,
rotations of the $A_i$ in the degenerate eigensubspaces with repeated
eigenvalues lead to the same matrix $M$. 

For $d>2$, if the $A_i$ are orthogonal, then the expansion in
\eqref{eqn:tensor} is unique and can be computed efficiently. The
algorithm is power iteration that recovers one $A_i$ at a time
(see e.g. \cite{AnandkumarTensorDecomp}). The requirement that the $A_i$ are
orthogonal is necessary for this algorithm, but if one also has access
to the order-2 tensor (i.e., matrix) in addition, $M = \sum_{i=1}^m
A_i \otimes A_i$, and the $A_i$ are linearly indepenent, then one can
arrange for the orthogonality of the $A_i$ by a suitable linear
transformation. However, the fundamental limitation remains that we
must take $m \le n$ simply because we can not have more than $n$
orthogonal vectors in $\R^n$.

Here we consider a modified setting where we are
allowed some additional information: suppose we have access to two
tensors, both of order $d$, which share the same rank-1 components,
but have different coefficients:
\begin{align*}
  T_{\mu} = \sum_{i=1}^m \mu_i A_i^{\otimes d}, \qquad T_{\lambda}
  = \sum_{i=1}^m \lambda_i A_i^{\otimes d}.
\end{align*}
Given such a \emph{pair} of
tensors $T_\mu$ and $T_\lambda$, can we recover the rank-1 components
$A_i$?

We answer this question in the affirmative for even orders $d \in 2
\N$, and give a provably good algorithm for this problem assuming
that the ratios
$\mu_i/\lambda_i$ are distinct. Additionally, we assume that the $A_i$
are not scalar multiples of each other, a necessary assumption. We make this quantitative via 
the $m^{th}$ singular value of the matrix with columns given by $A_i^{\odot d/2}$.

Our algorithm works by flattening tensors $T_{\mu}$ and $T_{\lambda}$
into matrices $M_\mu$ and $M_\lambda$ which have the following form:
\begin{align*}
  M_\mu = (A^{\odot d/2}) \diag{ \mu_i} (A^{\odot d/2})^T, \qquad 
  M_\lambda = (A^{\odot d/2}) \diag{ \lambda_i} (A^{\odot d/2})^T.
\end{align*}
Taking the
product $M_\mu M_{\lambda}^{-1}$ yields a matrix whose eigenvectors
are the columns of $A^{\odot d/2}$ and whose eigenvalues are $\mu_i /
\lambda_i$:
\begin{align*}
  M_{\mu} M_{\lambda}^{-1} & = (A^{\odot d/2}) \diag{ \mu_i} (A^{\odot
    d/2})^T ((A^{\odot d/2})^T)^{-1} \diag{ \lambda_i}^{-1} (A^{\odot
    d/2})^{-1} \\
  & = (A^{\odot d/2}) \diag{ \mu_i / \lambda_i} (A^{\odot d/2})^{-1}.
\end{align*}

Actually, for the last equation to make sense one needs that $A^{\odot d/2}$ be invertible which will 
generally not be the case as $A^{\odot d/2}$ is not even a square matrix in general. 
We handle this by restricting $M_\mu$ and $M_\lambda$ to linear transform
from their pre-image to the image. This is the reason for 
introducing matrix $W$ in algorithm {\bf Diagonalize$(M_{\mu}, M_{\lambda})$} below.

The main algorithm below is {\bf Tensor Decomposition$(T_{\mu},
  T_{\lambda})$} which flattens the tensors and calls subroutine {\bf
  Diagonalize$(M_{\mu}, M_{\lambda})$} to get estimates of the
$A_i^{\odot d/2}$, and from this information recovers the $A_i$
themselves.  In our application it will be the case that $\mu, \lambda
\in \C^m$ and $A_i \in \R^n$. The discussion below is tailored to this
situation; the other interesting cases where everything is real or
everything is complex can also be dealt with with minor modifications.

\begin{figure}[hbtp]
\begin{center}
\fbox{\parbox{\textwidth}{
\begin{minipage}{6in}
\vspace{0.1in}
{\bf Diagonalize$(M_{\mu}, M_{\lambda})$}
\begin{enumerate}
  \item Compute the SVD of $M_\mu = V \Sigma U^T$, and let $W$ be the
    left singular vectors (columns of $V$) corresponding to the $m$ largest
    singular values. Compute the matrix $M = (W^T M_\mu W)(W^T M_\lambda W)^{-1}$.
  \item Compute the eigenvector decomposition $M = PDP^{-1}$.
  \item Output columns of $WP$. 
\end{enumerate}
\end{minipage}
}}
\end{center}
\end{figure}

\begin{figure}[hbtp]
\begin{center}
\fbox{\parbox{\textwidth}{
\begin{minipage}{6in}
\vspace{0.1in}
{\bf Tensor Decomposition$(T_{\mu}, T_{\lambda})$}
\begin{enumerate}
  \item Flatten the tensors to square matrices to obtain $M_\mu = \tau^{-1}(T_\mu)$ and $M_\lambda
    = \tau^{-1}(T_\lambda)$.
  \item $WP = Diagonalize(M_\mu, M_\lambda)$. 
  \item For each column $C_i$ of $WP$, let $C'_i := \re{e^{i \theta^\ast}
      C_i}/\norm{\re{e^{i \theta^\ast} C_i}}$ where $\theta^\ast =
    \argmax_{\theta \in [0,2 \pi]}\left( \norm{ \re{ e^{i \theta} C_i}}\right)$. 
  \item For each column $C'_i$, let $v_i \in \R^n$ be such that $v_i^{\otimes d/2}$ is the best 
rank-1 approximation to $\tau(C'_i)$.
\end{enumerate}
\end{minipage}
}}
\end{center}
\end{figure}

The columns $C_i = WP_i$ are eigenvectors computed in subroutine \textbf{Diagonalize}. 
Ideally, we would 
like these to equal $A_i^{\odot d/2}$. We are going to have errors introduced because of sampling,
but in addition, since we are 
working in the complex field we do not have control over the phase of $C_i$ (the output of
\textbf{Diagonalize} obtained in Step 3 of \textbf{Tensor Decomposition}), and for 
$\rho \in \C$ with $\abs{\rho}=1$, $\rho C_i$ is also a valid output of \textbf{Diagonalize}. 
In Step 3 of \textbf{Tensor Decomposition}, we recover the correct phase of
$C_i \in \C^n$ which will give us a vector in $C'_i \in \R^n$. 
We do this by choosing the phase maximizing the norm of the real part.

In Step 4, we have $v^{\otimes d} + E$, where $E$ is an error tensor, and we want to recover
$v$. We can do this approximately when $\norm{E}_F$ is sufficiently small just by reading off
a one-dimensional slice (e.g. a column in the case of matrices) of $v^{\otimes d} + E$ 
(say the one with the maximum norm). 


For the computation of eigenvectors of diagonalizable (but not normal)
matrices over the complex numbers, we can employ any of the several
algorithms in the literature (see for example
\cite{GolubBook, press2007numerical} for a number of algorithms used in
practice). In general, these algorithms employ the same atomic
elements as the normal case (Jacobi iterations, Householder
transformations etc.), but in more sophisticated ways. The
perturbation analysis of these algorithms is substantially more
involved than for normal matrices; in particular, it is not
necessarily the case that a (small) perturbation to a diagonalizable
matrix results in another diagonalizable matrix. We contend with all
these issues in Section \ref{subsec:correctness}. In particular we note that while
{\em exact} analysis is relatively straightforward (Theorem \ref{thm:exact}), a robust 
version that recovers the common decomposition of the given tensors
takes considerable additional care (Theorem \ref{thm:robustdecomposition}).

\subsection{Underdetermined ICA}\label{subsec:icaalgorithm}
For underdetermined ICA we compute the higher derivative
tensors of the second characteristic function $\psi_x(u) = \log(
\phi_x(u))$ at two random points and run the tensor decomposition algorithm from 
the previous section. 

\begin{center}
\fbox{\parbox{\textwidth}{
\begin{minipage}{5in}
\vspace{0.1in}
{\bf Underdetermined ICA}($\sigma$)

\begin{enumerate}

\item (Compute derivative tensor) Pick independent random vectors 
$\alpha, \beta \sim N(0, \sigma^2 I_n)$. For even $d$ estimate the $d^{th}$ derivative tensors 
 of $\psi_x(u)$ at $\alpha$ and $\beta$ as $T_\alpha = D^{d}_u\psi_x(\alpha)$ and
$T_\beta = D^{d}_u\psi_x(\beta)$. 

\item (Tensor decomposition) Run \textbf{Tensor Decomposition$(T_\alpha,T_\beta)$}.





\end{enumerate}
\end{minipage}
}}
\end{center} 
To estimate the
$2d^{th}$ derivative tensor of $\psi_x(u)$ empirically, one simply
writes down the expression for the derivative tensor, and then
estimates each entry from samples using the naive estimator.

The analysis roughly proceeds as follows: By 
\eqref{eqn:derivative_linear_transform_diagonal} for tensors flattened into matrices we 
have
$D_u^{2d} \psi_x(\alpha) 
= A^{\odot d} \diag{\partial^{2d}_{t_1}\psi_s(A^T\alpha), \ldots, \partial^{2d}_{t_1}\psi_s(A^T\alpha)} (A^{\odot d})^T$ and 
$D_u^{2d} \psi_x(\beta) 
= A^{\odot d} \diag{\partial^{2d}_{t_1}\psi_s(A^T\beta), \ldots, \partial^{2d}_{t_1}\psi_s(A^T\beta)} (A^{\odot d})^T$.

Thus we
have two tensors with shared rank-1 factors as in the tensor
decomposition algorithm above. For our tensor decomposition to work,
we require that all the ratios
$(\partial^{2d}_{t_j}\psi_s(A^T\alpha))/(\partial^{2d}_{t_j}\psi_s(A^T\beta))$ 
for $j \in [m]$ be different
from each other as otherwise the eigenspaces in the flattened forms will
mix and we will not be able to uniquely recover the columns $A_i$. To
this end, we will express $\partial^{2d}_{t_j}\psi_s(A^T\alpha)$ 
as a low degree polynomial plus error
term (which we will control by bounding the derivatives of $\psi_s$). The
low degree polynomials will with high probability take on sufficiently
different values for $A^Tu$ and $A^Tv$, which in turn guarantees that
their ratios, even with the error terms, are quite different.

Our analysis for both parts might be of interest for
other problems. On the way to doing this in full generality for
underdetermined ICA, we first consider the special case of $d=2$,
which will already involve several of these concepts and the algorithm
itself is just PCA reweighted with Fourier weights. 

\section{Fully determined independent component analysis}\label{sec:ica}
We begin with the case of standard or fully determined ICA where the
transformation matrix $A \in \R^{n \times n}$ is full rank. 
With a slight
loss of generality, we assume that $A$ is unitary. If $A$ is not unitary, 
we can simply make it approximately so by placing the entire sample in
isotropic position. Rigorously arguing about this will require an additional 
error analysis; we will omit such details for the sake of
clarity. In any case, our algorithm for underdetermined ICA does not (and cannot) make any
such assumption.
Our algorithm computes the eigenvectors of a covariance
matrix reweighted according to random Fourier coefficients.

\begin{figure}[hbtp]
\begin{center}
\fbox{\parbox{\textwidth}{
\begin{minipage}{5in}
\vspace{0.1in}
{\bf Fourier PCA}($\sigma$)
\begin{enumerate}
\item (Isotropy) Get a sample $S$ from the input distribution and use them to find an isotropic transformation $B^{-1}$ with
\[
B^2 = \frac{1}{|S|}\sum_{x \in S}(x-\bar{x})(x-\bar{x})^T.
\]
\item (Fourier weights) Pick a random vector $u$ from $N(0, \sigma^2
  I_n)$. For every $x$ in a new sample $S$, compute $y = B^{-1}x$, and
  its Fourier weight
\[
w(y) = \frac{e^{iu^T y }}{\sum_{y \in S} e^{iu^T y }}.
\] 
\item (Reweighted Covariance) Compute the covariance matrix of the points $y$ reweighted by $w(y)$
\[
\mu_u = \frac{1}{\abs{S}}\sum_{y \in S} w(y) y \quad \mbox{ and }
\quad \Sigma_u = \frac{1}{\abs{S}}\sum_{y \in
  S}w(y)(y-\mu_u)(y-\mu_u)^T.
\]
\item Compute the eigenmatrix $V$ of $\Sigma_u$ and output $BV$.\\
\end{enumerate}
\end{minipage}
}}
\end{center}
\end{figure}

Formally, this algorithm is subsumed by our work on underdetermined
ICA in Section \ref{sec:underdetermined}, but both the algorithm and
its analysis are substantially simpler than the general case, but we
retain the essential elements of our technique -- fourier transforms,
polynomial anti-concentration and derivative truncation. On the other
hand, it does not require the machinery of our tensor decomposition in
Section \ref{sec:decompositions}.

We make the following comments regarding the efficient realisation of
this algorithm. The matrix $\Sigma_u$ in the algorithm is complex and
symmetric, and thus is not Hermitian; its eigenvalue decomposition is
more complicated than the usual Hermitian/real-symmetric case. 
It can
be computed in one of two ways. One is to compute the SVD of
$\Sigma_u$ (i.e., compute the eigenvalue decomposition of
$\Sigma_u\Sigma_u^*$ which is a real symmetric matrix). Alternatively,
we can exploit the fact that the real and complex parts have the same
eigenvectors, and hence by carefully examining the real and imaginary
components, we can recover the eigenvectors. We separate $\Sigma_u =
\re{\Sigma_u}+i \im{\Sigma_u}$ into its real part and imaginary part,
and use an SVD on $\re{\Sigma_u}$ to partition its eigenspace into
subspaces with close eigenvalues, and then an SVD of $\im{\Sigma_u}$
in each of these subspaces. Both methods need some care to ensure that
eigenvalue gaps in the original matrix are preserved, an important
aspect of our applications. We complete the algorithm description for
ICA by giving a precise method for determining the eigenmatrix $V$ of
the reweighted sample covariance matrix $\Sigma_u$. This subroutine
below translates a gap in the complex eigenvalues of $\Sigma_u$ into
observable gaps in the real part.
\begin{enumerate}
\item Write $\Sigma_u = \re{\Sigma_u}+i \im {\Sigma_u}$. Note that
  both the component matrices are real and symmetric.
\item Compute the eigendecomposition of $\re{\Sigma_u} =
  U\diag{r_i}U^T$.
\item Partition $r_1, \ldots, r_n$ into blocks $R_1, \ldots, R_l$ so
  that each block contains a subsequence of eigenvalues and the gap
  between consecutive blocks is at least $\eps_0$, i.e., $\min_{r \in
    R_j, s \in R_{j+1}} r-s \ge \eps_0$. Let $U_j$ be the eigenvectors
  corresponding to block $R_j$.
\item For each $1 \le j \le l$, compute the eigenvectors of
  $U_j^T\im{\Sigma_u}U_j$ and output $V$ as the full set of
  eigenvectors (their union).
\end{enumerate}

\begin{lemma}\label{lem:svd}
  Suppose $\Sigma_u$ has eigenvalues $\lambda_1, \ldots, \lambda_n$
  and $\eps = \min_{i\neq j} \min \{ \re{\lambda_i} - \re{\lambda_j},
  \im{\lambda_i} - \im{\lambda_j}\}$. Then, with $\eps_0 = \eps/n$,
  the above algorithm will recover the eigenvectors of $\Sigma_u$.
\end{lemma}
\begin{proof}
  The decomposition of the matrix $\re{\Sigma_u}$, will accurately
  recover the eigensubspaces for each block (since their eigenvalues
  are separated). Moreover, for each block $U_j
  \diag{r_i} U_j^T$, the real eigenvalues $r_i$ are within a range
  less than $\eps$ (since each consecutive pair is within $r_i -
  r_{i+1} < \eps/n$). Thus, for each pair $i, i+1$ in this block, we
  must have a separation of at least $\eps$ in the imaginary parts of
  $\lambda_i, \lambda_{i+1}$, by the definition of $\eps$. Therefore
  the eigenvalues of $Q_j = U_j^T\im{\Sigma_u}U_j$ are separated by at
  least $\eps$ and we will recover the original eigenvectors
  accurately.
\end{proof}

To perform ICA, we simply apply Fourier PCA to samples from the input
distribution. We will show that for a suitable choice of $\sigma$ and
sample size, this will recover the independent components to any
desired accuracy.  The main challenge in the analysis is showing that
the reweighted covariance matrix will have all its eigenvalues spaced
apart sufficiently (in the complex plane). This eigenvalue spacing
depends on how far the component distributions are from being
Gaussian, as measured by cumulants. Any non-Gaussian distribution will
have a nonzero cumulant, and in that sense this is a complete method.
We will quantify the gaps in terms of the cumulants to get an
effective bound on the eigenvalue spacings. The number of samples is
chosen to ensure that the gaps remain almost the same, and we can
apply eigenvector perturbation theorems Davis-Kahan or Wedin to
recover the eigenvectors to the desired accuracy.

\subsection{Overview of analysis}
Our main theorem in the analysis of this algorithm is as follows:
\begin{theorem}\label{thm:ICA}
  Let $x \in \R^n$ be given by an ICA model $x=As$ where $A \in \R^{ n
    \times n}$ is unitary 
  and the $s_i$ are independent, $\E{s_i^4}
  \le M_4$ for some constant, and for each $s_i$ there exists a $k_i
  \le k$ such that $\abs{\cum_{k_i}(s_i)} \ge \Delta$ (one of the
  first $k$ cumulants is large) and $\E{ \abs{s_i}^k} \le M_k$. For
  any $\eps > 0$, with the following setting of $\sigma$,
  \begin{align*}
    \sigma = \frac{\Delta}{2k!} \left( \frac{ \sqrt{2 \pi}}{4(k-1)n^2}
    \right)^k \cdot \frac{1}{(2e)^{k+1} M_k \log(4n)^{k+1}},
  \end{align*}
  \textbf{Fourier PCA} will recover vectors $\{b_1, \ldots, b_n\}$
  such that there exists signs $a_i = \pm 1$ satisfying
  \begin{align*}
    \norm{A_i - b_i} \le \eps
  \end{align*}
  with high probability, using $(ckn)^{2k^2+2}(M_k/\Delta)^{2k+2}
  M_4^2/ \eps^2 $ samples.
\end{theorem}

Our analysis proceeds via the analysis of the Fourier transform: for a
random vector $x \in \R^n$ distributed according to $f$, the
characteristic function is given by the Fourier transform
\begin{align*}
  \phi(u)  = \E{ e^{iu^Tx}} = \int f(x) e^{iu^Tx} dx.
\end{align*}
We favour the Fourier transform or characteristic function over the
Laplace transform (or moment generating function) for the simple
reason that the Fourier transform always exists, even for very heavy
tailed distributions. In particular, the trivial bound $\abs{e^{itx}}
= 1$ means that once we have a moment bound, we can control the
Fourier transform uniformly.

We will actually employ the \emph{second characteristic function} or
\emph{cumulant generating function} given by $\psi(u) = \log(
\phi(u))$. Note that both these definitions are with respect to
observed random vector $x$: when $x$ arises from an ICA model $x= As$,
we will also define the component-wise characteristic functions with
respect to the underlying $s_i$ variables $\phi_i(u_i) = \E{ e^{i u_i
    s_i}}$ and $\psi_i(u_i) = \log( \phi_i(u_i))$. Note that both
these functions are with respect to the underlying random variables
$s_i$ and not the observed random variables $x_i$. For convenience, we
shall also write $g_i = \psi_i''$.

Note that the reweighted covariance matrix in our algorithm is
precisely the Hessian second derivative matrix $D^2 \psi$:
\begin{align*}
  \Sigma_u & = D^2 \psi = \frac{\E{ (x- \mu_u)(x- \mu_u)^T
      e^{iu^Tx}}}{\E{ e^{iu^Tx}}},
\end{align*}
where $\mu_u = \E{ x e^{iu^Tx}} / \E{ e^{iu^Tx}}$. This matrix $D^2
\psi$ has a very special form; suppose that $A=I_n$:
\begin{align*}
  \psi(u)  = \log \left( \E{ e^{iu^Ts}} \right)
   = \log \left( \E{ \prod_{i=1}^n e^{iu_is_i}}\right) 
   = \sum_{i=1}^n \log( \E{ e^{iu_is_i}} )
   = \sum_{i=1}^n \psi_i(u_i).
\end{align*}
Taking a derivative will leave only a single term
\begin{align}
  \frac{ \partial \psi}{\partial u_i} = \psi'_i(u_i).
\end{align}
And taking a second derivative will leave only the diagonal terms
\begin{align*}
  D^2 \psi = \diag{ \psi_i''(u_i)} = \diag{ g_i( u_i)}.
\end{align*}
Thus, diagonalizing this matrix will give us the columns of $A = I_n$,
provided that the eigenvalues of $D^2 \psi$ are non-degenerate. In the
general case when $A \ne I_n$, we can first place $x$ in isotropic
position whence $A$ will be unitary. We perform the change of basis
carefully in Section \ref{sec:eigenvalues}, obtaining that $D^2 \psi =
A \diag{ \psi_i''( (A^Tu)_i)}A^T$. Now $D^2 \psi$ is symmetric, but
not Hermitian, and its eigenvalues are complex, but nonetheless a
diagonalization suffices to give the columns of $A$.

To obtain a robust algorithm, we rely on the eigenvalues of $D^2 \psi$
being adequately spaced (so that the error arising from sampling does
not mix the eigenspaces, hence columns of $A$). Thus, we inject some
randomness by picking a random Fourier coefficient, and hope that the
$g_i(u_i)$ are sufficiently anti-concentrated. To this end, we will
truncate the Taylor series of $g_i$ to $k^{th}$ order, where the
$k^{th}$ cumulant is one that differs from a gaussian
substantially. The resulting degree $k$ polynomial will give us the
spacings of the eigenvalues via polynomial anti-concentration
estimates in Section \ref{sec:anti}, and we will control the remaining
terms from order $k+1$ and higher by derivative estimates in Section
\ref{sec:truncation}. Notably, the further that $s_i$ is from being a
gaussian (in cumulant terms), the stronger anti-concentration. We will
pick the random Fourier coefficient $u$ according to a Gaussian $N(0,
\sigma^2 I_n)$ and we will show that with high probability for all
pairs $i,j$ we have
\begin{align*}
  \abs{ g_i( (A^Tu)_i) - g_j( (A^Tu)_j)} \ge \delta.
\end{align*}
Critical to our analysis is the fact that $(A^Tu)_i$ and $(A^Tu)_j$
are both independent Gaussians since the columns of $A$ are
independent by our assumption of isotropic position (Section
\ref{sec:spacings}). We then go onto compute the sample complexity
required to maintain these gaps in Section \ref{sec:samples} and
conclude with the proof of correctness for our algorithm in Section
\ref{sec:proofica}.

In case more than one of the variables are (standard) Gaussians, then
a quick calculation will verify that $\psi_i''(u_i) = 1$. Thus, in the
presence of such variables the eigenvectors corresponding to the
eigenvalue 1 are degenerate and we can not resolve between any linear
combination of such vectors. Thus, the model is \emph{indeterminate}
when some of the underlying random variables are too gaussian. To deal
with this, one typically hypothesizes that the underlying variables
$s_i$ are different from Gaussians. One commonly used way is to
postulate that for each $s_i$ the fourth moment or cumulant differs
from that of a Gaussian. We weaken this assumption, and only require
that \emph{some} moment is different from a Gaussian.

The Gaussian function plays an important role in harmonic analysis as
the eigenfunction of the Fourier transform operator, and we exploit
this property to deal with additive Gaussian noise in our model in
Section \ref{subsec:noise}.

\subsection{Eigenvalues}\label{sec:eigenvalues}
As noted previously, when $A=I_n$, we have $D^2 \psi = \diag{
  \psi_i''(u_i)}$. When $A \ne I_n$ we have
\begin{lemma}\label{lemma:firstderivative}
  Let $ x \in \R^n$ be given by an ICA model $x = As$ where $A \in
  \R^{n \times n}$ is a unitary matrix and $s\in \R^n$ is an
  independent random vector. Then
  \begin{align*}
    D^2 \psi =  A \diag{ \psi_i''( (A^Tu)_i)} A^T.
  \end{align*}
\end{lemma}

This lemma is standard chain rule for multivariate functions. A more general 
version applying to higher derivative
tensors is proved later in Section~\ref{sec:underdetermined}.

\subsection{Anti-concentration for polynomials}\label{sec:anti}
The main result of this section is an anti-concentration inequality
for univariate polynomials under a Gaussian measure. 
While this inequality appears very similar to the inequality of Carbery--Wright~\cite{Carbery}
(cf. ~\cite{MOS}, Corollary 3.23),
we are not able to derive our inequlity from it. The hypothesis of our inequality is weaker in 
that it only requires the polynomial to be monic instead of requiring the polynomial to have 
unit variance as required by Carbery--Wright; on the other hand it applies only to univariate 
polynomials.


\begin{theorem}[Anti-concentration of a polynomial in Gaussian space]\label{thm:anti}
  Let $p(x)$ be a degree $d$ monic polynomial over $\R$. Let $x \sim
  N(0,\sigma^2)$, then for any $t \in \R$ we have
  \begin{align*}
    \prob{ \abs{p(x) - t} \le \eps} \le \frac{4d \eps^{1/d}
    }{\sigma \sqrt{2 \pi}}.
  \end{align*}
\end{theorem}
For most of the proof we will work with the Lebesgue measure; the proof for the Gaussian 
measure will follow immediately. 
Our starting point is the following lemma which
can be derived from the properties of Chebyshev polynomials
(\cite{borwein}, Sec 2.1, Exercise 7); we include a proof for completeness. 
For interval $[a,b]$, define the supremum norm 
on real-valued functions $f$ defined on $[a,b]$ by
\begin{align*}
  \norm{f(x)}_{[a,b]} := \norm{f(x) \chi_{[a,b]}}_{\infty} = \sup_{x \in [a,b]}
    \abs{f(x)}.
\end{align*}
Then we have
\begin{lemma}
  The unique degree $d$ monic polynomial minimising
  $\norm{p(x)}_{[a,b]}$ is given by
  \begin{align}\label{eqn:chebyshev}
    p(x) = 2 \left( \frac{b-a}{4} \right)^d T_d \left(
      \frac{2x-a-b}{b-a} \right),
  \end{align}
  where $T_d$ is the $d^{th}$ Chebyshev polynomial.
\end{lemma}
\begin{proof}
  We already know (see \cite{borwein}) that for the interval $[-1,1]$ the unique monic
  polynomial of degree $d$ which minimizes $\norm{p(x)}_{[-1,1]}$ is
  given by $2^{1-d} T_d(x)$. Map the interval $[a,b]$ to $[-1,1]$
  using the affine map $f(x) = (2x - a -b)/(b-a)$ which satisfies
  $f(a) = -1$ and $f(b) = 1$. Then $((b-a)/2)^d 2^{1-d} T_d(x) = 2((b-a)/4)^d T_d(x)$ 
  is the unique monic polynomial minimizing $\norm{\cdot}_{[a,b]}$. For if it were not, we could
use such a polynomial to construct another monic polynomial (by reversing the above transformation)
which contradicts the fact that Chebyshev polynomials are the unique monic minimizers of 
$\norm{\cdot}_{[-1,1]}$.
\end{proof}
From this we have the following lemma. 
\begin{lemma}\label{lemma:difference}
  Let $p(x)$ be a degree $d$ monic polynomial over $\R$. Fix $\eps >0$, then for every
  $x$, there exists an $x'$ where $\abs{x-x'} \le \eps$ and $\abs{p(x)
    - p(x')} \ge 2(\eps / 2)^{d}$.
\end{lemma}
\begin{proof}
  We will translate the polynomial $p$ to obtain the polynomial $q(y)$
  as follows:
  \begin{align*}
    q(y) = p(y+x) - p(x).
  \end{align*}
  Observe that $q(y)$ is a monic polynomial and $q(0) = 0$. Now
  suppose that for all points  $x' \in [x-\eps,x+\eps]$, we
  have $\abs{p(x) - p(x')} < (\eps / 2)^d$, then for all $y
  \in [-\eps,\eps]$, we must have $\abs{q(y)} < 2(\eps/2)^d$.

  But, from the previous lemma, we know that for the interval
  $[-\eps,\eps]$, the minimum $L^\infty$-norm on the interval for
  any monic polynomial is attained by $r(y) = 2 (\eps/2)^d T_d (y /
  \eps)$. The value of this minimum is $2(\eps/2)^d$.
\end{proof}

We can use the above lemma to given an upper bound on the measure of the set where a
polynomial stays within a constant sized band:
\begin{lemma}\label{lemma:intervals}
  Let $p(x)$ be a degree $d$ monic polynomial. Then for any interval
  $[a,b]$ where $b-a = \eps$ we have 
  \begin{align*}
    \mu ( x \in \R: p(x) \in [a,b] ) \le 4d \left( \frac{\eps}{2}
    \right)^{1/d},
  \end{align*}
  where $\mu$ denotes the usual Lebesgue measure over $\R$.
\end{lemma}
\begin{proof}
  Since $p$ is a continuous function so, $p^{-1}([a,b]) = \cup_{i}
  I_i$ where $I_i$ are disjoint closed intervals. There are at most
  $d$ such intervals: every time $p(x)$ exits and re-enters the interval $[a,b]$
 there must be a change of sign in the derivative $p'(x)$ at some point in between. 
Since $p'(x)$ is a degree $d-1$ polynomial, there can only be $d-1$ changes of
  sign. 

  Next, suppose that $\abs{I_i} > 4 (\eps/2)^{1/d}$, then there
  exists an interval $[x - 2 (\eps'/2)^{1/d}, x+2(\eps'/2)^{1/d}]
  \subseteq I_i$, where $\eps' > \eps$. Then, by applying Lemma~\ref{lemma:difference}, there
  exists a point $x'$ such that $\abs{x-x'} \le 2(\eps'/2)^{1/d}$  but
  \begin{align*}
    \abs{p(x) - p(x')} & \ge 2 \left[ \frac{1}{2} \cdot 2 \left(
        \frac{\eps'}{2} \right)^{1/d} \right]^d \\
    & \ge \eps' > \eps.
  \end{align*}
 This implies that $x' \notin [a,b]$, which is a
  contradiction. Hence we must have $\abs{I_i} \le 4 (\eps/2)^{1/d}$
  and
  \begin{align*}
    \sum_i \abs{I_i} \le d \max_i \abs{I_i} \le 4 d (\eps/2)^{1/d},
  \end{align*}
  as required.
\end{proof}
We can now give the proof for Theorem \ref{thm:anti}:
\begin{proof}[Proof of Theorem \ref{thm:anti}]
  We know that the Lebesgue measure of the set for which $p(x) \in [t-\eps,t+\eps]$ is
  given by Lemma \ref{lemma:intervals}. Then multiplying by the
  maximum density of a Gaussian $1/\sigma \sqrt{2 \pi}$ gives us the
  desired bound.
\end{proof}

\subsection{Truncation Error}\label{sec:truncation}
Let us expand out the function $g_i = \psi''_i$ as a Taylor
series with error estimate:
\begin{theorem}[Taylor's theorem with remainder]\label{thm:taylor}
  Let $f: \R \to \R$ be a $C^n$ continuous function over some interval
  $I$. Let $a,b \in I$, then
  \begin{align*}
    f(b) = \sum_{k=1}^{n-1} \frac{f^{(k)} (a)}{k!}(b-a)^k +
    \frac{f^{(n)}(\xi)}{n!} (b-a)^n,
  \end{align*}
  for some $\xi \in [a,b]$.
\end{theorem}
To this end, we write
\begin{align}\label{eqn:taylor}
  g_i(u_i) = p_i(u_i) + \frac{g^{(k)}(\xi)}{k!} u_i^k,
\end{align}
where $ \xi \in [0,u_i]$ and $p_i$ is a polynomial of degree $(k-1)$.

To bound the error term in  \eqref{eqn:taylor}, we observe that
it suffices to bound $[\log( \phi_i)]^{(k)}(u_i)$ using the
following lemma. 
\begin{lemma}\label{lemma:errorestimate}
  Let $x \in \R$ be a random variable with finite $k$ absolute
  moments, and let $\phi(u)$ be the associated characteristic
  function. Then
  \begin{align*}
    \abs{ \left[ \log( \phi) \right]^{(k)}(u)} \le \frac{2^{k-1}
      (k-1)!\; \E{ \abs{x}^k}}{\abs{ \phi(u)}^k}.
  \end{align*}
\end{lemma}
\begin{proof}
  We will compute the derivatives of $\psi(u) = \log \phi(u)$ as follows: we proceed
  recursively with $\psi'(u) = \phi'(u) / \phi(u)$ as our base
  case. Let $\psi^{(d)}$ be given by the ratio of two functions, a
  numerator function $N(u;d)$ and a denominator function $D(u;d)$,
  with no common factors and $N(u;d)$ is the sum of terms of the form $\prod_{j=1}^d
  \phi^{(i_j)}(u)$ where the coefficient of each term is $\pm 1$. 
Some useful properties of functions $N(u;d), D(u;d)$ are summarized in the following claim.
  \begin{claim}\label{claim:errorestimate}
    For $d \geq 1$, functions $N(u;d)$ and $D(u;d)$ satisfy
    \begin{enumerate}
    \item $D(u;d) = \phi(u)^d$.
    \item For each term of $N(u;d)$, $\sum_{j=1}^d i_j \le d$.
    \item For each term of $N(u;d)$, the total number of factors
      of $\phi$ and its derivatives is at most $d$.
    \item For $d \ge 1$, there are at most $2^{d-1} (d-1)!$ terms
      in $N(u;d)$.
    \end{enumerate}
  \end{claim}
  \begin{proof}
    We will prove all these via induction over $d$. Clearly these are
    all true for the base case $d=1$. Assume that all four facts are
    true for some $d$, we will now examine the case for $d+1$.

Writing $\psi^{(d+1)}(u)$ as the derivative of $\psi^{(d)}(u)= N(u;d)/D(u;d)$ and canceling 
common factors gives
    \begin{align}\label{eqn:quotient}
      \psi^{(d+1)}(u) =& \frac{ N'(u;d) D(u;d) - N(u;d)
        D'(u;d)}{D(u;d)^2} \notag\\
      & =\frac{ N'(u;d) \phi(u)^d - N(u;d) d
        \phi(u)^{d-1}\phi'(u)}{\phi(u)^{2d}}  \notag\\
      & = \frac{ N'(u;d) \phi(u) - d \phi'(u) N(u;d)}{\phi(u)^{d+1}}.
    \end{align}
    Observing that there is always a term in $N(u;d) = \phi'(u)^d$, we
    can not cancel any further factors of $\phi(u)$. Hence $D(u;d) =
    \phi(u)^d$, proving the first part of the claim.

The second and third parts of the claim follow immediately from the final expression for 
$\psi^{(d+1)}(u)$ above and our inductive hypothesis. 



    To prove the fourth part, let $T(d)$ denote the total number of
    terms in $N(u;d)$, then by part 3 and the expansion in 
    \eqref{eqn:quotient}, we have $T(d+1) \le d T(d) + d T(d) \le
    2d T(d)$.  From this $T(d+1) \le 2^d d!$ follows immediately.
  \end{proof}

  Returning to the proof of Lemma~\ref{lemma:errorestimate}, for $d \leq k$ we
  observe that
  \begin{align*}
    \abs{\phi^{(d)}(u)} = \abs{ \E{ (ix)^d e^{iu^T x}}} \le \E{ \abs{(ix)^d
      e^{iu^T x}}} \le \E{ \abs{ x}^d}.
  \end{align*}
  Thus, for each term of $N(u;d)$:
  \begin{align*}
    \abs{\prod_{j=1}^d \phi^{(i_j)}(u)} \le \prod_{j=1}^d \abs{
      \phi^{(i_j)}(u)} \le \prod_{j=1}^d \E{ \abs{x}^{i_j}} 
      \overset{\mathrm{Fact \ref{fact:holder}}}{\leq} \E{\abs{x}^{ \sum_{j=1}^d i_j}} \le \E{\abs{x}^d}. 
  \end{align*}

  Combining Claim~\ref{claim:errorestimate} with the previous equation, and noticing that
 we never need to consider absolute moments of order higher than $k$ (which are guaranteed to 
exist by our hypothesis), gives
  the desired conclusion. 
\end{proof}

To conclude this section, we observe that if the distribution of $x \in \R$ is isotropic
then for $u \in (-1,1)$ we have
\begin{align*}
  \phi_x(u) = \phi_x(0) + \phi'_x(0)u + \frac{\phi''_x(\xi)}{2} u^2,
\end{align*}
where $\xi \in [0,u]$. 
We have $\phi_x(0) = 1$, $\phi'_x(0) = \E{x} = 0$ and $\abs{\phi''_x(\xi)}
\le \E{ \abs{x}^2} = 1$ by the isotropic position assumption. Thus, for 
$u \in [-1/4,1/4]$, Lemma~\ref{lemma:errorestimate} gives us
\begin{align}\label{eqn:truncation}
  \abs{[\log( \phi_x)]^{(k)}(u)} \le  \E{ \abs{x}^k} k^k.
\end{align}

\subsection{Eigenvalue spacings}\label{sec:spacings}
We will apply Theorem~\ref{thm:anti} to the truncated Taylor
polynomials $p_i$ of the functions $\psi_i''$, and bound the
truncation error to be at most a constant fraction of the desired
anti-concentration.

\begin{theorem}\label{thm:spacings}
  Let $s \in \R^n$ be a random vector with indpendent components. For $t \in \R^n$, let $\psi(t) = \log\E{e^{it^Ts}}$ 
be the second characteristic function of $s$. Suppose we are given the 
following data and conditions: 
  \begin{enumerate}
\item Integer $k > 2$ such that $\E{\abs{s_i}^{k}}$ exists for all $i \in [n]$.
    \item $\Delta > 0$ such that for each $i \in [n]$, there exists $2 < k_i < k$ such 
that $\abs{\cum_{k_i}(s_i)} \ge \Delta$.
\item $M_2 >0$ such that $\E{s_i^2} \leq M_2$ for $i \in [n]$.
    \item $M_k > 0$ such that $\E{\abs{s_i}^{k}} \le M_k$ for $i \in [n]$.
\item $g_i(t_i) := \frac{\partial^2 \psi(t)}{\partial t_i^2}$.
\item $\tau\sim N(0, \sigma^2 I_n)$ where $\sigma = \min(1, \frac{1}{2\sqrt{2 M_2 \log{1/q}}}, \sigma')$ and
 \begin{align}\label{eqn:sigma-prime}
 \sigma' =  \left(\frac{3}{8}\right)^{k+1} \cdot \frac{k-1}{k!} \cdot 
\left(\frac{\sqrt{2\pi}}{4k}\right)^{k-2} \cdot \frac{q^{k-2}}{(\sqrt{2 \log(1/q)})^{k-1}} 
\cdot \frac{\Delta}{M_k},
  \end{align}
and $0 < q \leq 1/3$. 
  \end{enumerate}
 Then with probability at least $1 - n^2 q$, for all distinct $i, j$ we have
  \begin{align*}
    \abs{ g_i( \tau_i ) - g_j( \tau_j )} \ge \frac{\Delta}{2(k-2)!}
    \left( \frac{\sqrt{2\pi} \sigma q}{4k}\right)^{k-2}.
  \end{align*}
\end{theorem}
\begin{proof}
  We will argue about the spacing $\abs{g_1(\tau_1)-g_2(\tau_2)}$, and then use the union bound to get that none of
the spacings is small with high probability.
  Since $s_1$ has first  $k$ moments, we can apply Taylor's theorem with
  remainder (Actually one needs more care as that theorem was stated for functions
of type $\R \to \R$, whereas our function here is of type $\R \to \C$. To this end, we can consider the real and imaginary 
parts of the function separately and apply Theorem~\ref{thm:taylor} to each part; we omit the details.)
Applying Theorem~\ref{thm:taylor} gives
  \begin{align*}
    g_1(t_1) = - \sum_{l=2}^{k_1} \cum_l(s_1) \frac{ (it_1)^{l-2}}{(l-2)!}
    + R_{1}(t_1) \frac{(it_1)^{k_1-1}}{(k_1-1)!}. 
  \end{align*}
  Truncating $g_1$ after the degree $(k_1-2)$ term yields a
  polynomial $p_1(t_1)$. Denote the truncation error by $\rho_1(t_1)$. Then, fixing $t_2$ arbitrarily and
  setting $z = g_2(t_2)$ for brevity, we have 
  \begin{align*}
    \abs{ g_1(t_1) - g_2(t_2)} &= \abs{ p_1(t_1) + \rho_1(t_1) - z} \\
    & \ge \abs{p_1(t_1) - z} - \abs{\rho_1(t_1)}.
  \end{align*}

We will show that $\abs{p_1(t_1) - z}$ is likely to be large and $\abs{\rho_1(t_1)}$ is likely to be small.
 Noting that $\frac{(k_1-2)!}{i^{k_1-2}\cum_{k_1}(s_1)} p_1(t_1)$ is monic of degree $k_1-2$ 
(but with coefficients from $\C$), 
we apply our anti-concentration result in Theorem~\ref{thm:anti}.
Again, although that theorem was proven for polynomials with real coefficients, its application to the present
situation is easily seen to go through without altering the bound by considering the real and imaginary parts
separately. In the following, the probability is for $t_1 \sim N(0, \sigma^2)$. 
  \begin{align*}
    \prob{ \abs{ p_1(t_1) - z} \le \eps_1} & \le \frac{4(k_1-2)}{\sigma
      \sqrt{2 \pi}} \left( \frac{\eps_1 (k_1-2)!}{\abs{\cum_{k_1}(s_1)}}
    \right)^{1/(k_1-2)} \le \frac{4k_1}{\sigma
      \sqrt{2 \pi}} \left( \frac{\eps_1 (k_1-2)!}{\Delta}
    \right)^{1/(k_1-2)}.
  \end{align*}
Setting 
\begin{align} \label{eqn:eps_setting}
\eps_1
  := \frac{\Delta}{(k_1-2)!} \left( \frac{ \sqrt{2 \pi} \sigma q}{4k_1} \right)^{(k_1-2)}
\leq \frac{\Delta}{(k-2)!} \left( \frac{ \sqrt{2 \pi} \sigma q}{4k} \right)^{(k-2)}
\end{align}
we have
  \begin{align*}
    \prob{ \abs{p_1(t_1) - z} \le \eps} \le q. 
  \end{align*}
  Next we bound the truncation error and show that $\abs{\rho_1(t_1)} \le \eps / 2$ with probability at least 
$1 - \frac{q}{\sqrt{\pi \log{1/q}}}$. Applying Lemma~\ref{lemma:errorestimate}, the error introduced is
  \begin{align*}
    \abs{\rho_1(t_1)} & \le \frac{k_1! \; 2^{k_1} \E{\abs{s_1}^{k_1+1}}}{\abs{\phi_1(t_1)}^{k_1+1}}\cdot \frac{t_1^{k_1-1}}{(k_1-1)!}.
  \end{align*}

We now lower bound the probability that $\abs{t_1}$ is small when $t_1 \sim N(0,\sigma^2)$:
\begin{align*}
\prob{\abs{t_1} \leq \sigma\sqrt{2\log{1/q}}} \geq 1 - \frac{q}{\sqrt{\pi \log{1/q}}}.
\end{align*}
The computation above used Claim~\ref{claim:gaussian_concentration}. 

Thus with probability at least $1 - \frac{q}{\sqrt{\pi \log{1/q}}}$ we have

\begin{align} \label{eq:eps_1_bound}
\abs{\rho_1(t_1)} \leq \frac{k_1! \; 2^{k_1} M_k}{(3/4)^{k_1+1} (k_1-1)!}\cdot (\sigma \sqrt{2 \log{1/q}})^{k_1-1},
\end{align}
here we used that by our choice of $\sigma$ we have $\sigma\sqrt{2\log{1/q}} \leq \frac{1}{2\sqrt{M_2}}$, hence
Lemma~\ref{lem:phi_nonvanishing} gives that $\abs{\phi(t_1)} \geq 3/4$.

Now for $\abs{t_1} \leq \sigma\sqrt{2\log{1/q}}$ we want 
\begin{align*}
\abs{\rho_1(t_1)} \leq \epsilon_1/2.
\end{align*}

This is seen to be true by plugging in the value of $\eps_1$ from \eqref{eqn:eps_setting} and the bound 
on $\rho_1(t_1)$ from \eqref{eq:eps_1_bound} and our choice of $\sigma$.


Thus we have proven that $\abs{g_1(t_1)-g_2(t_2)} \geq \eps/2$ with probability at least 
$1- (q + \frac{q}{\sqrt{\pi \log{1/q}}}) \geq 1 - 2q$ (using $q \in (0,1/3]$). Now applying the union bound over all pairs 
  we get the required bound.
\end{proof}

\subsection{Sample complexity}\label{sec:samples}
In this section, we bound the sample complexity of the
algorithm. First we will show how many samples are necessary to
estimate accurately the desired Fourier transforms $\E{ e^{iu^Tx}}$,
$\E{x e^{iu^Tx}}$ and $\E{xx^T e^{iu^Tx}}$.
\begin{lemma}\label{lemma:sample0}
  Let $x \in\R^n$ be a random vector. Fix $\eps > 0$ and a vector $t
  \in \R^n$. Let $x^{(j)}$ be i.i.d. samples drawn according to $x$ then
  \begin{align*}
    \abs{ \frac{1}{m} \sum_{j=1}^m e^{i u^T x^{(j)}} - \E{ e^{iu^Tx}}} \le \eps,
  \end{align*}
  with probability at least $ 1- 4 e^{- m \eps^2 / 2}$.
\end{lemma}
\begin{proof}
  Note that the random variables $e^{iu^Tx}$ are bounded in magnitude
  by 1. We separate out the real and imaginary components of
  $e^{iu^Tx}$ and separately apply the Chernoff inequality.
\end{proof}

In the most general setting, all we can do is bound the variance of
our sample covariance matrix, and this will give a polynomial bound on
the sample complexity.

\begin{lemma}\label{lemma:sample1}
  Suppose that the random vector $x \in\R^n$ is drawn from an
  isotropic distribution $F$. Then
  \begin{align*}
\Var( x_j e^{iu^Tx} ) &\le 1 \mbox{ for } 1 \le j \le n, \\ 
\Var(x_j^2 e^{iu^Tx}) &\le \E{x_j^4}, \\ 
\Var(x_i x_j e^{iu^Tx} )&\le 1 \mbox{ for } i \neq j.
\end{align*}
\end{lemma}

\begin{proof}
  \[
   \mathrm{Var} ( x_j
    e^{iu^Tx} ) =  \E{x_j^2} - \abs{\E{x_j e^{iu^Tx}}}^2
     \le 1.
  \]
The other parts are similar, with the last inequality using isotropy.
\end{proof}

We can combine these concentration results for the Fourier derivatives
to obtain the final sample complexity bound. Recall from 
\eqref{eqn:truncation} that we have in the interval $u \in [-1/4,1/4]$
\begin{align*}
  \abs{ g(u)} \le  \E{ \abs{x}^2} k^k \le 16
\end{align*}
We can now give the sample complexity of the algorithm.
\begin{corollary}\label{cor:sampling}
  Let $x=As$ be an ICA model where $A \in \R^{n \times n}$ is a
  unitary matrix. Suppose that the random vector $s \in\R^n$ is drawn from
  an isotropic distribution, and that for each $s_i$, we have $\E{s_i^4} \le
  M$. Fix $\eps > 0$ and a vector $u \in \R^n$ where $\norm{u} \le
  1/4$. Let $\hat{\Sigma}_u$ be the matrix estimated from $m$
  independent samples of $x^i = A s^i$, then
  \begin{align*}
    \norm{ \hat{\Sigma}_u - \Sigma_u}_F \le \eps
  \end{align*}
  with probability at least $1-1/n$ for $m \ge \poly(n,M) / \eps^2$.
\end{corollary}
\begin{proof}
  Apply Chebyshev's inequality along with the variance bounds.  Since
  the Frobenius norm is unitarily invariant, we can consider the error
  in the basis corresponding to $s$. In this basis:
  \begin{align*}
    &\norm{\E{e^{iu^Ts} (s- \tilde{\mu})(s- \tilde{\mu})^T -
        \tilde{\Sigma}_u }}  \\
    &\le \norm{\E{ss^T e^{iu^Ts}}-\sum_{i=1}^m (s^i)(s^i)^T e^{i
      u^T s^i}} + 2\norm{\E{s \tilde{\mu}^T e^{iu^T
        s}}-\sum_{i=1}^m x
      \hat{\tilde{\mu}}^T e^{iu^Ts}} \\
    &\quad+ \norm{\E{\tilde{\mu} \tilde{\mu}^T} - \hat{\tilde{\mu}}
      \hat{\tilde{\mu}}^T}
    \abs{\E{e^{iu^Ts}}} \\
    & \le \eps
  \end{align*}
  where the last bound is derived by apportioning $\eps/5$ error to
  each term.  Finally, we conclude by noting that by our choice of
  $t$, we have $\abs{\E{e^{iu^Tx}}} \ge 29/32$, and the multiplicative
  error due to the scaling by $1/\E{e^{iu^Tx}}$ is lower order in
  comparison to $\eps$.
\end{proof}

For more structured distributions, e.g., logconcave distributions, or
more generally distributions with subexpoential tails, much sharper
bounds are known on the sample complexity of covariance estimation,
see e.g., \cite{rudelson, vershynin, sriver, adamczak}.

\subsection{Proof  of Theorem \ref{thm:ICA}}\label{sec:proofica}
In this section we give the proof of correctness for our algorithm for
ICA.
\begin{proof}[Proof of Theorem \ref{thm:ICA}]
  In the exact case, the diagonal entries are given by
  $g_i( (A^T u)_i)$. Since $A$ is orthonormal, for any pair $(A^Tu )_i
  = A_i^T u$ and $(A^Tu)_j = A_j^T u$ have orthogonal $A_i$ and $A_j$,
  hence the arguments of $g_i$ and $g_j$ are independent Gaussians and
  Theorem \ref{thm:spacings} gives us the eigenvalue spacings of
  $\Sigma_u$ to be used in Lemma \ref{lemma:stability}.

  In particular, the spacings are at least $\xi = \frac{\Delta}{2k!}
  \left( \frac{\sqrt{2\pi} \sigma}{4(k-1)^{n^2}} \right)^k$. Thus,
  with desired accuracy $\eps$ in Lemma \ref{lemma:stability}, then we
  require the sampling error (in operator norm, which we upper bound
  using Frobenius norm) to be $\norm{E}_F \le \eps \xi / (\xi +
  \eps)$. We can then substitute this directly into Corollary
  \ref{cor:sampling} which gives the sample complexity.
\end{proof}

\subsection{Gaussian noise}\label{subsec:noise}
The Gaussian function has several nice properties with respect to the
Fourier transform, and we can exploit these to cancel out independent
Gaussian noise in the problems that we study. To deal with Gaussian
noise, when the observed signal $x=As+\eta$ where $\eta$ is from an
unknown Gaussian $N(\mu_\eta, R_\eta)$ which is independent of $s$, we
can use the following modified algorithm.
\begin{center}
\fbox{\parbox{\textwidth}{
\begin{minipage}{5in}
\begin{enumerate}
\item Pick two different random Gaussian vectors $u,v$. 
\item Compute $\Sigma=\Sigma_0, \Sigma_u$ and $\Sigma_v$ as in the previous algorithm. 
\item Output the eigenvectors of $(\Sigma_u - \Sigma)(\Sigma_v-\Sigma)^{-1}$.\\
\end{enumerate}
\end{minipage}
}}
\end{center}
\medskip

\begin{theorem}\label{thm:gaussian}
  Let $x \in \R^n$ be given be a noisy independent components model $x
  = As+\eta$, where $A \in \R^{n \times n}$ is a full rank matrix, and
  the noise vector $\eta$ has a Gaussian distribution.  With
  sufficiently many samples, the modified algorithm outputs $A$.
\end{theorem}
\begin{proof}
  When $x=As+\eta$, the function $\psi(u) = \log \left( \E{e^{i u^T
        x}} \right)$ can be written as
\begin{align*}
  \psi(u) = \log \left( \E{ e^{i u^T x}} \right) + \log \left( \E{e^{
        i u^T \eta}}\right)
\end{align*}
Therefore,
\begin{align*}
  D^2 \psi_u &= A\diag{ \psi''_i( A^T_i u)}A^T + \frac{\E{e^{i u^T
        \eta}(\eta-\mu_\eta)(\eta-\mu_\eta)^T}}{\E{e^{i
        u^T \eta}}} \\
  & = A\diag{ \psi''_i( A^T_i u)}A^T +
  \E{(\eta-\mu_\eta)(\eta-\mu_\eta)^T} \\
  & = A\diag{ \psi''_i( A^T_i u)}A^T +R_\eta
\end{align*}
where $\eta \sim N(\mu_\eta, R_\eta)$.  Therefore,
\begin{align*}
\Sigma_u - \Sigma = A(D_u - D)A^T
\end{align*}
with $D$ being the covariance matrix of $s$ and 
\begin{align*}
(\Sigma_u - \Sigma)(\Sigma_v-\Sigma)^{-1} = A(D_u-D)(D_v-D)^{-1}A^{-1}.
\end{align*}
The eigenvectors of the above matrix are the columns of $A$.
\end{proof}
For a complete robustness analysis, one needs to control
the spectral perturbations of the matrix $ A(D_u-D)(D_v-D)^{-1}A^{-1}$
under sampling error. We omit this proof, but note that it follows
easily using the techniques we develop for underdetermined ICA.

\section{Efficient tensor decompositions}\label{sec:decompositions}

In this section we analyze the tensor decomposition algorithm,  
which will be our main tool for the underdetermined ICA problem. 

\subsection{Algorithm}
Recall that our algorithm works by flattening tensors $T_{\mu}$ and $T_{\lambda}$ into 
matrices $M_\mu$ and $M_\lambda$ and then observing that the eigenvectors of
$M_\mu M_\lambda^{-1}$ are vectors corresponding to flattened
(vectorized) $A_i^{\otimes d/2}$. 

\begin{figure}[hbtp]
\begin{center}
\fbox{\parbox{\textwidth}{
\begin{minipage}{6in}
\vspace{0.1in}
{\bf Diagonalize$(M_{\mu}, M_{\lambda})$}
\begin{enumerate}
  \item Compute the SVD of $M_\mu = V \Sigma U^\ast$, and let the $W$ be the
    left singular vectors (columns of $V$) corresponding to the $m$ largest
    singular values. Compute the matrix $M = (W^\ast M_\mu W)(W^\ast M_\lambda W)^{-1}$.
  \item Compute the eigenvector decomposition $M = PDP^{-1}$.
  \item Output columns of $WP$. 
\end{enumerate}
\end{minipage}
}}
\end{center}
\end{figure}

\begin{figure}[hbtp]
\begin{center}
\fbox{\parbox{\textwidth}{
\begin{minipage}{6in}
\vspace{0.1in}
{\bf Tensor Decomposition$(T_{\mu}, T_{\lambda})$}
\begin{enumerate}
  \item Flatten the tensors to obtain $M_\mu = \tau^{-1}(T_\mu)$ and $M_\lambda
    = \tau^{-1}(T_\lambda)$.
  \item $WP = Diagonalize(M_\mu, M_\lambda)$. 
  \item For each column $C_i$ of $WP$, let $C'_i := \re{e^{i \theta^\ast}
      C_i}/\norm{\re{e^{i \theta^\ast} C_i}}$ where $\theta^\ast =
    \argmax_{\theta \in [0,2 \pi]}\left( \norm{ \re{ e^{i \theta} C_i}}\right)$. 
  \item For each column $C'_i$, let $v_i \in \R^n$ be such that $v_i^{\otimes d/2}$ is the best 
rank-1 approximation to $\tau(C'_i)$.
\end{enumerate}
\end{minipage}
}}
\end{center}
\end{figure}


In Step 3 of Tensor Decomposition, we get an approximation of $v_i^{\odot d/2}$ up to 
a phase factor. We first correct the phase by maximizing the projection onto $\R^n$. To this
end we prove
\begin{lemma}\label{lem:complextoreal}
  Let $v \in \C^n$ and $u \in \R^n$ be unit vectors such that for some
  $\varphi \in [0,2\pi]$ we have $\norm{e^{i \varphi} v - u} \le
  \eps$ for $0 \leq \eps \leq 1/2$. Let $\theta^\ast = \argmax_{\theta \in [0,2\pi]}( \norm{
    \re{e^{i \theta} v}})$ and $u' = \re{ e^{i
      \theta^\ast}v}/\norm{ \re{e^{i \theta^\ast}v}}$. Then there is a sign $\alpha \in {-1,1}$ such
that 
  \begin{align*}
    \norm{\alpha u - u'} \le 11 \sqrt{\eps}.
  \end{align*}
\end{lemma}
\begin{proof}
  Without loss of generality, we will assume that $\varphi = 0$, hence
  $\norm{v - u} \le \eps$. By the optimality of $\theta^\ast$
  \begin{align*}
    \norm{ \re{ e^{i \theta^\ast } v}} \ge \norm{ \re{v}} \ge 1-\eps.
  \end{align*}
  Let us denote $v' = e^{i \theta^\ast }v$, then we have
  $\norm{\re{v'}}^2 + \norm{\im{v'}}^2 = 1$ which implies that
  $\norm{\im{v'}}^2 \le 2 \eps - \eps^2 < 2 \eps$. Now using $\eps \leq 1/2$ we have
  \begin{align*}
    \norm{ v' -u'}  &\le \norm{ \re{v'} - u'} + \norm{ \im{v'}}  \\
&= \norm{\re{v'} - \frac{\re{v'}}{\norm{\re{v'}}}} +  \norm{ \im{v'}} \\
&\leq \norm{\re{v'}}\left(\frac{1}{1-\eps} - 1\right) + \norm{ \im{v'}} \\
&\leq 2\eps  + \sqrt{2 \eps} \le 4 \sqrt{\eps},
  \end{align*}
and
  \begin{align*}
    \norm{ u' - e^{i \theta^\ast}u} \le \norm{ u' - e^{i \theta^\ast}v} + \norm{e^{i \theta^\ast}v - e^{i \theta^\ast}u} = \norm{u'-v'}+ \norm{u-v} <  5\sqrt{\eps}.
  \end{align*}
  This implies $\abs{ \re{ e^{ i \theta^\ast}}}\ge 1 - 5\sqrt{ \eps}$.
Hence there is a sign $\alpha \in {-1,1}$ such that $\abs{e^{ i \theta^\ast}-\alpha} \leq 10 \sqrt{\eps}$ (we omit some routine computations). 
Finally, 
\begin{align*}
\norm{u'-\alpha u} \leq \norm{u'-e^{i \theta^\ast}u} + \norm{e^{i \theta^\ast}u- \alpha u} 
\leq 5 \sqrt{\eps} + 10 \sqrt{\eps} = 15 \sqrt{\eps}.  
\end{align*}

\end{proof}

\begin{lemma} \label{lem:tensor-root}
For unit vector $v \in \R^n$ and a positive integer $d$, given $v^{\odot d} + E$, where
$E$ is an error vector, we can recover $v''$ such that for some $\alpha \in \set{-1,1}$ 
we have 
\begin{align*}
\norm{v- \alpha v''}_2  \leq  \frac{2 \norm{E}_2}{\beta-\norm{E}_2},
\end{align*}
where $\beta = \frac{1}{n^{d/2-1/2}}$.
\end{lemma}
\begin{proof}
Let's for a moment work with $v^{\odot d}$ (so there is no error), and then we will 
take the error into account. In this case we can essentially read $v$ off from $v^{\odot d}$. 
Each one-dimensional slice of $v^{\otimes d}$ (Note 
that as vectors, $v^{\odot d}$ and $v^{\otimes d}$ are the same; they differ only in
how their entries are arranged: In the former, they are in a linear array and in the latter
they are in an $n \times n \times \ldots \times n$ array. We will use
them interchangeably, and we will also talk about $v^{\otimes d} + E$ which has the obvious 
meaning.) is a scaled
 copy of $v$. Let us choose the copy with the maximum norm. Since $\norm{v}=1$, there
is a coordinate $v(i)$ such that $\abs{v(i)} \geq 1/\sqrt{n}$. Thus there is a
one-dimensional slice of $v^{\otimes d}$ with norm at least $\frac{1}{n^{d/2-1/2}} = \beta$. 
Scaling this slice to norm $1$ would result in $\alpha v$ for some $\alpha \in \set{-1,1}$. 
Now, when we do have error and get $v^{\otimes d} + E$, 
then we must have a one-dimensional slice $v'$ of $v^{\otimes d} + E$ with norm at least
$\beta-\norm{E}_2$. Then after normalizing $v'$ to $v''$, one can check that 
$\norm{\alpha v''-v} \leq \frac{2 \norm{E}_2}{\beta-\norm{E}_2}$ for some 
$\alpha \in \set{-1,1}$. 
\end{proof}


\subsection{Exact analysis}
We begin with the proof of the tensor decomposition theorem with
access to exact tensors as stated in Theorem
\ref{thm:decomposition}. This is essentially a structural results that
says we can recover the rank-1 components when the ratios $\mu_i /
\lambda_i$ are unique.

We first note that for a tensor $T_\mu$ with a rank-1 decomposition as
in \eqref{eqn:tensor}, that the flattened matrix version$M_\mu
= \tau^{-1}(T_\mu)$ can be written as
\begin{align*}
  M_\mu = (A^{\odot d/2})\diag{\mu_i} (A^{\odot d/2})^T.
\end{align*}
We will argue that the diagonalisation step works correctly (we will write $B =  A^{\odot d/2}$ in what follows). 
The
recovery of $A_i$ from the columns of $B$ follows by Lemma~\ref{lem:complextoreal} above.

Our theorem is as follows (note that the first condition below is
simply a normalisation of the eigenvectors):
\begin{theorem}\label{thm:exact}
Let $M_\mu , M_\lambda \in \C^{p \times p}$ such that:
\begin{align*}
M_\mu = B \diag{\mu_i} B^T, \qquad \text{and} \qquad M_\lambda = B \diag{\lambda_i} B^T,
\end{align*}
where $B \in \R^{p \times m}$ and $\mu, \lambda \in \C^m$ for some $m
\le p$. Suppose that the following hold:
\begin{enumerate} 
\item For each column $B_i \in \R^m$ of $B$, $\norm{B_i}_2 = 1$,
\item $\sigma_m(B) > 0$, and
\item $\mu_i, \lambda_i \neq 0$ for all $i$, and $\abs{\frac{\mu_i}{\lambda_i}-\frac{\mu_j}{\lambda_j}} > 0$ for 
all $i \neq j$.
\end{enumerate}
Then \textbf{Diagonalize$(M_{\mu}, M_{\lambda})$} outputs the columns of $B$ up to sign and permutation. 
\end{theorem}
\begin{proof}
  By our assumptions, the image of
  $M_\lambda$ has dimension $m$ and the matrix $W$ computed in 
\textbf{Diagonalize$(M_{\mu}, M_{\lambda})$} satisfies
 $\colspan{W} = \colspan{B}$. Moreover, we could choose $W$ to have all 
entries real because $B$ is a real matrix; this will give that the 
ambiguities in the recovery of $B$ are in signs and not in phase. Since
  the columns of $W$ are orthonormal, the columns of $P := W^T B$ all
  have unit norm and it is a full rank $m \times m$ matrix. So we can
  write
  \begin{align*} 
    W^T M_\mu W &= P \diag{ \mu_i} P^T, \\
    (W^T M_\lambda W)^{-1} &= (P^T)^{-1} \diag{ \lambda_i^{-1}} P^{-1}.
  \end{align*}

Which gives
  \begin{align*}
    ( W^T M_\mu W) (W^T M_\lambda W)^{-1} = P \diag{ \mu_i / \lambda_i} P^{-1}.
  \end{align*} 

Thus the  colums of $P$ are the eigenvectors of $( W^T M_\mu W) (W^T M_\lambda W)^{-1}$, and
  thus our algorithm is able to recover the columns of $P$ up to sign and permutation. Let's call the matrix so recovered $P'$.
Denote by $P_1, \ldots, P_m$ the columns of $P$, and similarly for $P'$ and $B$. Then $P'$ is given by 
$P'_{\pi(j)} = \alpha_j P_j$ where $\pi: [m] \rightarrow [m]$ is a permutation and $\alpha_j \in \set{-1, +1}$.

We now claim that $WP = WW^TB = B$. To see this, let $\hat{W} = [W, W']$ be an orthonormal basis that completes $W$. 
Then $\hat{W}^T\hat{W} = \hat{W}\hat{W}^T = I$. Also, $\hat{W}\hat{W}^T = WW^T+W'W'^T$. For any vector $v$ in the span of the columns of $W$, we have 
$v = \hat{W}\hat{W}^Tv = (WW^T + W'W'^T)v = WW^Tv$. In other words, 
$W$ acts as orthonormal matrix restricted to its image, and
  thus $WW^T$ acts as the identity. In particular, $WP = WW^T B =
  B$. 

  Our algorithm has access to $P'$ as defined above rather than to
  $P$. The algorithm will form the product $WP'$.  But now it's clear
  from $WP = B$ that $WP'_{\pi(j)} = \alpha_j B_j$. Thus the algorithm
  will recover $B$ up to sign and permutation.
\end{proof}

\subsection{Diagonalizability and robust
  analysis} \label{subsec:correctness} 

In applications of our tensor decomposition algorithm, we do not have
access to the true underlying tensors $T_\mu$ and $T_\lambda$ but
rather slightly perturbed versions. We prove now that under suitably
small perturbations $R_\mu$ and $R_\lambda$, we are able to recover
the correct rank 1 components with good accuracy. The statement of the
robust version of this theorem closely follows that of the exact
version: we merely need to add some assumptions on the magnitude of
the perturbations relative to the quotients $\mu_i / \lambda_i$ in
conditions 4 and 5.

\begin{theorem} \label{thm:robustdecomposition}
Let $M_\mu, M_\lambda \in \C^{p \times p}$ such that 
\begin{align*}
M_\mu = B \diag{\mu_i} B^T, \qquad M_\lambda = B \diag{\lambda_i} B^T,
\end{align*}
where $B \in \R^{p \times m}$ , and $\mu, \lambda \in \C^m$ for some
$m \leq p$. For error matrices $R_\mu, R_\lambda \in \C^{p \times p}$, let $M_\mu + R_\mu$ and 
$M_\lambda + R_\lambda$ be perturbed versions of $M_\mu$ and $M_\lambda$. Let $0 < \eps < 1$. Suppose
that the following conditions and data are given:
\begin{enumerate}
\item For each column $B_i \in \R^m$ of $B$, $\norm{B_i}_2 = 1$.
\item $\sigma_m(B) > 0$.
\item $\mu_i, \lambda_i \neq 0$ for all $i$,
  $\abs{\frac{\mu_i}{\lambda_i}-\frac{\mu_j}{\lambda_j}} \ge \Omega > 0$ for all
  $i \neq j$.
\item  $0 < K_L \leq \abs{\mu_i}, \abs{\lambda_i} \leq K_U$.
\item $\norm{R_\mu}_F, \norm{R_\lambda}_F \leq K_1 \leq \frac{\epsilon K_L^2 \sigma_m(B)^3}{2^{11} \kappa(B)^3 K_U m^2} \min(\Omega, 1).$ \label{assumption:upper_bound}
\end{enumerate}
Then \textbf{Diagonalize} applied to $M_\mu + R_\mu$ and
  $M_\lambda + R_\lambda$ outputs $\tilde{B}$ such that there exists a 
  permutation $\pi: [m] \to [m]$ and phases $\alpha_j$ (a phase $\alpha$ is a scalar in $\C$ with 
$\abs{\alpha}=1$) such that 
  \begin{align*}
    \norm{B_j - \alpha_j\tilde{B}_{\pi(j)}} \le \eps.
  \end{align*}
The running time of the algorithm is 
$\text{poly}(p, \frac{1}{\Omega}, \frac{1}{K_L}, \frac{1}{\sigma_{\min}(B)}, \frac{1}{\epsilon})$.
\end{theorem}
\begin{proof}
  We begin with an informal outline
  of the proof. We basically implement the proof for the exact case, however because
of the perturbations, various equalities now are true only approximately and this leads to 
somewhat lengthy and technical details, but the intuitive outline remains the same as for the
exact case. 

The algorithm constructs an orthonormal
  basis of the left singular space of $\bMm := M_\mu + R_\mu$; denote
  by $Y$ the matrix with this basis as its columns.  The fact that
  $\bMm$ is close to
  $M_\mu$ 
  gives by Wedin's theorem (Theorem \ref{thm:Wedin}) that the left singular spaces of $\bMm$ and
  $M_\mu$ are close. More specifically, this means that there are two
  matrices, $W$ with columns forming an orthonormal basis for the left
  singular space of $M_\mu$, and $X$ with columns forming an
  orthonormal basis for the left singular space of $\bMm$ such that $W$ and
  $X$ are close in the entrywise sense. This implies that $W^TB$ and
  $X^TB$ are close. This can be used to show that under appropriate
  conditions $X^TB$ is nonsingular. Now using the fact that the
  columns of $Y$ and of $X$ span the same space, it follows that
  $\bar{P} := Y^TB$ is nonsingular. In the next step, we show by
  virtue of $\norm{R_\mu}$ being small that the matrix $Y^T\bMm Y$
  constructed by the algorithm is close to $\bar{P} \diag{\mu_i}
  \bar{P}^T$ where the $\mu_i$ are the eigenvalues of $M_\mu$; and
  similarly for $Y^T \bMl Y$.  We then show that $(Y^T\bMm Y)(Y^T\bMl
  Y)^{-1}$ is diagonalizable and the diagonalization provides a matrix
  $\tilde{P}$ close to $\bar{P}$, and so $\tilde{B} = Y \tilde{P}$ gives the
  columns of $B$ up to phase factors and permutation and small error.

\emph{A note on the running time.} Algorithm Diagonalize uses SVD and eigenvector decomposition of diagonalizable
(but not normal) matrices as subroutines. There are well-known algorithms for these as discussed earlier. 
The outputs of these algorithms are not exact and have a quantifiable error:
The computation of SVD of $M \in \C^{n \times n}$ within error $\epsilon$ (for any reasonable notion of error, say 
$\norm{M - V\Sigma U^T}_F$ where $V\Sigma U^T$ is the SVD output by the algorithm on input $M$) can be 
done in time $\text{poly}\left(n, \frac{1}{\epsilon}, \frac{1}{\sigma_{\min}(M)}\right)$. Similarly, for the eigenvector 
decomposition of a diagonalizable matrix $M \in \C^{n \times n}$ with eigenvalues 
$\abs{\lambda_i-\lambda_j} \geq \Omega > 0$ for $i \neq j$, we can compute the decomposition within error 
$\epsilon$ in time $\text{poly}(n, \frac{1}{\Omega}, \frac{1}{\epsilon}, \frac{1}{\min_i \abs{\lambda_i}})$. 

In the analysis below, we ignore the errors from these computations as they can be controlled and will be of smaller order than the error from the main analysis. This can be made rigorous but we omit
the details in the interest of brevity.
Combining the running time of the two subroutines one can check easily that the overall running time is what is 
claimed in the statement of the theorem.


We now proceed with the formal proof. 
The proof is broken into 7 steps. 
\paragraph{Step 1.} $W^TB \approx X^TB$. \\

Let $\bMm := M_\mu + R_\mu$ and $\bMl := M_\lambda + R_\lambda$. Now the fact that $\norm{R_\mu}_F$ is small implies by Wedin's
theorem (Theorem \ref{thm:Wedin}) that the left singular spaces of $M_\mu$ and $\bMm$ are close: 
Specifically, by Theorem IV.1.8 
in \cite{bhatia1997matrix} about canonical angles between subspaces, we have: There exists an orthonormal basis of the left singular 
space of $M_\mu$ (given by the columns $w_1, \ldots w_m$ of $W \in \C^{p\times m}$)
and an orthonormal basis of the left singular space of $\bMm$ 
(given by the columns $x_1, \ldots, x_m$ of $X \in \C^{p\times m}$) such that 
\begin{align*}
x_j = c_j w_j + s_j z_j, \;\; \text{for all} \; j,
\end{align*}
where $0\leq c_1 \leq \ldots \leq c_m \leq 1$, and $1 \geq s_1 \geq \ldots \geq s_m \geq 0$, and $c_j^2 + s_j^2 =1 $ for all $j$; vectors $w_1, \ldots, w_m; z_1, \ldots, z_m$ form an orthonormal basis. (For the last condition to 
hold we need $p \geq 2m$. A similar representation can be derived when this condition does not hold and the 
following computation will still be valid. We omit full discussion of this other case for brevity; in any case,
we could arrange so that $p \geq 2m$ without any great penalty in the parameters achieved.)
We now apply Wedin's theorem~\ref{thm:Wedin} to $M_\mu$ and $\bMm$ to upper bound $s_j$. To this end, first 
note that by Claim~\ref{claim:singular_value_inequality_product} we have $\sigma_m(M_\mu) \geq K_L \sigma_m(B)^2$;
and second, by Weyl's inequality for singular values (Lemma~\ref{lem:Weyl_singular_values}) we have 
$\abs{\sigma_j(\bMm)-\sigma_j(M_\mu)}\leq \sigma_1(R_\mu) \leq \placeholder_1$ for all $j$. Thus in 
Theorem~\ref{thm:Wedin}, with $\Sigma_1$
corresponding to non-zero singular values of $M_\mu$, we have $\max{\sigma(\Sigma_2)}=0$.
And we can choose a corresponding conformal SVD of $\bMm$ so that 
$\min{\sigma(\bar{\Sigma}_1)} \geq K_L\sigma_m(B)^2-\placeholder_1$.
Which gives, 
$\norm{\sin{\Phi}}_2 \leq \placeholder_1/(K_L\sigma_m(B)^2-\placeholder_1) =: \placeholder_2$, 
where $\Phi$ is the matrix of canonical angles between $\colspan{W}$ and $\colspan{X}$. Thus we have 
\begin{align}\label{eq:sin_ubd}
s_j \leq \placeholder_2,
\end{align}
for all $j$.

Now we can show that $X^TB$ is close to $W^TB$: The $(i,j)$'th entry of $W^TB-X^TB$ is $(1-c_i) w_i^Tb_j - s_i z_i^Tb_j$. 
Using \eqref{eq:sin_ubd} and $\norm{w_i}, \norm{b_j}, \norm{z_i} \leq 1$, we have 
\begin{align*}
(1-c_i) w_i^Tb_j - s_i z_i^Tb_j \leq s_i^2 + s_i \leq 2 \placeholder_2.
\end{align*}

And so $\norm{W^TB-X^TB}_F \leq 2 m^2 \placeholder_2$. Hence by Lemma~\ref{lem:Weyl_singular_values} we have 
$\abs{\sigma_j(W^TB)-\sigma_j(X^TB)} \leq 2m^2\placeholder_2$ for all $j$.

\paragraph{Step 2.} \emph{ $\bar{P} := Y^TB$ is full rank.} \\
The singular values of $W^TB$ are the same as those of $B$. Briefly,
this is because $W^T$ acts as an isometry on $\colspan{B}$.  Also
observe that the singular values of $Y^TB$ are the same as those of
$X^TB$. Briefly, this is because $Y^T$ and $X^T$ act as isometries on
$\colspan{X} = \colspan{Y}$. These two facts together with the
closeness of the singular values of $W^TB$ and $X^TB$ as just shown
imply that
\begin{align} \label{eqn:Weyl_sigma_j}
 \abs{\sigma_j(B)-\sigma_j(Y^TB)} \leq 2m^2\placeholder_2
\end{align}
for all $j$.
Now using that $2 m^2 \placeholder_2 < \sigma_m(B)/2$ (This follows by our
condition~\ref{assumption:upper_bound} in the theorem giving an upper bound on $K_1$: 
$K_1 \leq \epsilon \frac{K_L}{K_U} \frac{K_L \sigma_m(B)^3}{2^{11} \kappa(B)^3 m^2}$
which gives $\placeholder_1 \leq \frac{K_L \sigma_m(B)^3}{8m^2}$. This in turn implies 
$2 m^2 \placeholder_2 < \sigma_m(B)/2$ using $\sigma_m(B) \leq 1$; 
we omit easy verification.) we
get that $\sigma_m(Y^TB)>0$ and hence $Y^TB$ is full rank. We note some consequences
of \eqref{eqn:Weyl_sigma_j} for use in later steps:

\begin{align} \label{eq:kappa_P_B}
\kappa(\bar{P}) \leq 4\kappa(B).
\end{align}
This follows from $\kappa(\bar{P}) \leq \frac{\sigma_1(B)+2m^2\placeholder_2}{\sigma_m(B)-2m^2\placeholder_2} \leq 4 \kappa(B)$, 
because $2m^2\placeholder_2 < \sigma_m(B)/2$. 
\begin{align}\label{eq:sigma_m_B_P_upper}  
\sigma_m(\bP) \leq \sigma_m(B) + 2m^2K_2 < 2 \sigma_m(B).
\end{align}

\begin{align}\label{eq:sigma_m_B_P_lower}
\sigma_m(\bP) \geq \sigma_m(B) - 2m^2K_2 <  \sigma_m(B)/2.
\end{align}

\paragraph{Step 3.} \emph{$Y^T \bMm Y \approx \bar{P} \diag{\mu_i} \bar{P}^T$ and $Y^T \bMl Y \approx \bar{P} \diag{\lambda_i} \bar{P}^T$.} \\ 
More precisely, let $E_\mu := Y^T \bMm Y - \bP \diag{\mu_i}\bP^T$, then $\norm{E_\mu}_F \leq m^2 \norm{R_\mu}_F$; and similarly for $\bMl, E_\lambda := Y^T \bMl Y - \bP \diag{\lambda_i}\bP^T$. The proof is short: We have 
$Y^T \bMm Y = Y^T (M_\mu + R_\mu) Y = Y^T M_\mu Y + Y^T R_\mu Y = \bP \diag{\mu_i} \bP^T + Y^T R_\mu Y.$
Hence $\norm{E_\mu}_F = \norm{Y^T R_\mu Y}_F \leq \norm{R_\mu}_F$.

\paragraph{Step 4.} \emph{$(Y^T\bMm Y)(Y^T\bMl Y)^{-1}$ is diagonalizable.} \\

This is because Theorem~\ref{thm:diagonalizable2} is applicable to 
$\tilde{N}:=(Y^T\bMm Y)(Y^T\bMl Y)^{-1} = 
(\bP \diag{\mu_i} \bP^T + E_\mu)(\bP \diag{\lambda_i} \bP^T+E_\lambda)^{-1}$:
using $\norm{E_\mu}_F \leq \norm{R_\mu}_F$, the two condition to verify are
\begin{itemize}
\item  $ \frac{6 \kappa(\bP)^3 m K_U}{K^2_L \sigma_m(\bP)^2} K_1 \leq \Omega.$ \\
This follows from our condition~\ref{assumption:upper_bound} using 
\eqref{eq:kappa_P_B}, \eqref{eq:sigma_m_B_P_lower} and
$\sigma_m(B) \leq 1$.
\item $K_1 \leq \sigma_m(\bP)^2K_L/2.$ \\
This also follows from condition~\ref{assumption:upper_bound}, using \eqref{eq:sigma_m_B_P_upper}
and $\epsilon \leq 1$. 
\end{itemize}
Hence $\tilde{N}$ is diagonalizable: 
$\tilde{N}=\tilde{P}\diag{\tilde{\gamma}_i}\tilde{P}^{-1}$.

\paragraph{Step 5.} \emph{The eigenvalues of $\tilde{N}$ are close to the eigenvalues of 
$\bP \diag{\mu_i/\lambda_i} \bP^T$.} 
This follows from our application of Theorem~\ref{thm:diagonalizable2} in the previous step
(specifically from \eqref{eqn:application_Bauer_Fike}) and gives a permutation $\pi:[m] \to [m]$ such that 
\begin{align*}
\abs{\frac{\mu_i}{\lambda_i} - \tilde{\gamma}_{\pi(i)}} < \Omega/2,
\end{align*}
where the $\tilde{\gamma}_i$ are the eigenvalues of $\tilde{N}$. 

In the next step we show that there exist phases $\alpha_i$ such that $\tilde{P}^{\pi, \alpha} := [\alpha_1 \tilde{P}_{\pi(1)}, \alpha_2 \tilde{P}_{\pi(2)}, \ldots, \alpha_m \tilde{P}_{\pi(m)}]$ is close to $\bar{P}$.

\paragraph{Step 6.} \emph{$\bar{P}$ is close to $\tilde{P}$ up to sign and permutation of columns.} \\
We upper bound the angle $\theta$ between the corresponding eigenpairs 
$(\frac{\mu_j}{\lambda_j}, \bar{P}_j)$ and $(\tilde{\gamma}_{\pi(j)}, \tilde{P}_{\pi(j)})$ of 
$N := \bar{P}\diag{\mu_i/\lambda_i}\bar{P}^{-1}$ and $\tilde{N}$. 
Theorem~\ref{thm:eisenstat} (a generalized version of the $\sin(\theta)$ eigenspace perturbation theorem for diagonalizable 
matrices) applied to $N$ and $\tilde{N}$ gives (with the notation derived from Theorem~\ref{thm:eisenstat}) 
\begin{align*}
\sin{\theta} \leq \kappa(Z_2) \frac{\norm{(N - \tilde{\gamma}_{\pi(j)}I)\tilde{P}_{\pi(j)}}_2}{\min_i\abs{(N_2)_{ii}-
\tilde{\gamma}_{\pi(j)}}}.
\end{align*}

To bound the RHS above, we will estimate each of the three terms.
The first term:
\begin{align*}
\kappa(Z_2) \leq \kappa(\bar{P}^{-1}) = \kappa(\bar{P}) \leq 4\kappa(B),
\end{align*}
where for the first inequality we used that the condition number of a submatrix can 
only be smaller~\cite{Thompson}; the second inequality is \eqref{eq:kappa_P_B}. 

Setting $\mathsf{Err} := N-\tilde{N}$, we bound the second term:
\begin{align}
  \norm{(N - \tilde{\gamma}_{\pi(j)}I)\tilde{P}_{\pi(j)}}_2 &=
  \norm{(\tilde{N} - \tilde{\gamma}_{\pi(j)}I) \tilde{P}_{\pi(j)} +
    \mathsf{Err}\:\tilde{P}_{\pi(j)}}_2 \nonumber \\
  & = \norm{\mathsf{Err}\:\tilde{P}_{\pi(j)}}_2  \nonumber \\
  & \leq \norm{\mathsf{Err}}_2 \nonumber \\
 & \leq \kappa(\bP)^2 \cdot \frac{K_U}{K_L} \cdot 2m \cdot \frac{K_1}{\sigma_m(\bP)^2 K_L} \;\;\;\text{(by \eqref{eqn:upper_bound_err_norm2} in Theorem~\ref{thm:diagonalizable2})} \nonumber \\
& \leq 2^6 \kappa(B)^2 m \frac{K_U}{K_L} \frac{K_1}{\sigma_m(B)^2 K_L} \;\;\;\text{(using \eqref{eq:kappa_P_B}, \eqref{eq:sigma_m_B_P_upper}).}
\label{eqn:error_bound}
\end{align}

And lastly, the third term:
\begin{align*}
\min_i\abs{(N_2)_{ii}-\tilde{\gamma}_{\pi(j)}}
&\geq \min_{i:i \neq j}\abs{\frac{\mu_i}{\lambda_i}-\frac{\mu_j}{\lambda_j}} - 
\abs{\frac{\mu_j}{\lambda_j}-\tilde{\gamma}_{\pi(j)}}\\
&\geq \Omega -
\kappa(\bar{P})\norm{\mathsf{Err}}_2   \;\;\; \text{(using Lemma~\ref{thm:weyl})}\\
&\geq \Omega - 2^9 \kappa(B)^3 m \frac{K_U}{K_L} \frac{K_1}{\sigma_m(B)^2 K_L} \;\;\; \text{(using \eqref{eqn:error_bound} and \eqref{eq:kappa_P_B})}.
\end{align*}

To abbreviate things a bit, let's set 
$\epsilon':= 2^9 \kappa(B)^3 \frac{K_U}{K_L} m \frac{K_1}{\sigma_m(B)^2 K_L}$.
Then, putting things together we get
\begin{align*}
\sin(\theta) \le \frac{\eps'}{\Omega - \eps'}.
\end{align*}

Now using the fact that the columns of $\tilde{P}$ and $\bP$ are unit length implies that 
there exist phases $\alpha_i$ such that 
\begin{align} \label{eqn:sin_to_length}
\norm{\alpha_j\tilde{P}_{\pi(j)}-\bP_j}_2 \leq \frac{\eps'}{\Omega - \eps'}.
\end{align}

\paragraph{Step 7.} \emph{$Y\tilde{P}$ gives $B$ approximately and up to phase factors and permutation of columns.} \\
This follows from two facts: (1) $ \tilde{P}^{\pi, \alpha} \approx \bar{P}$, so 
$Y \tilde{P}^{\pi,\alpha} \approx Y\bar{P}$ (we will prove this shortly); and (2) $Y\bar{P} = YY^T B$ (follows
by the definition of $\bar{P}$). 
Now note that the operator $YY^T$ is projection to $\colspan{Y}$; since the angle between $\colspan{Y}$ 
and $\colspan{B}$ is small as we showed in Step 1, we get that $YY^TB \approx B$. 

Formally, we have
\begin{align*}
\norm{Y\alpha_j\tilde{P}_{\pi(j)}-Y\bar{P}_j}_2 \leq \norm{Y}_2 \norm{\alpha_j\tilde{P}_{\pi(j)}-\bar{P}_j}_2 \leq \frac{\eps'}{\Omega - \eps'},
\end{align*}
using \eqref{eqn:sin_to_length}. And
\begin{align*}
\norm{b_j - YY^Tb_j}_2 \leq K_2,
\end{align*}
where the last inequality used that the sine of the angle between $\colspan{Y}$ and 
$\colspan{W} = \colspan{B}$ is at most $K_2$ as proved in Step 1. 

Putting these together we get

\begin{align*}
\norm{Y\alpha_j \tilde{P}_{\pi(j)}-b_j}_2 
\leq \norm{b_j - YY^Tb_j}_2 + \norm{Y\alpha_j\tilde{P}_{\pi(j)}-Y\bar{P}_j}_2
\leq \frac{\eps'}{\Omega - \eps'} + K_2.
\end{align*}

Letting $\tilde{B} = Y \tilde{P}$ gives
\begin{align*}
\norm{\alpha_j \tilde{B}_{\pi(j)} - b_j}_2 \leq \frac{\eps'}{\Omega - \eps'} + K_2 \leq \epsilon.
\end{align*}
The last inequality follows from our condition~\ref{assumption:upper_bound}, which implies that
$\frac{\eps'}{\Omega - \eps'} \leq \epsilon/2$ and $K_2 \leq \epsilon/2$.


\end{proof}

\begin{theorem}[Diagonalizability of perturbed matrices]\label{thm:diagonalizable2}
  Let $N_\mu, N_\lambda \in \C^{m \times m}$ be full rank complex
  matrices such that $N_\mu = Q\diag{\mu_i} Q^T, N_\lambda =
  Q\diag{\lambda_i}Q^T$ for some $Q \in \R^{m \times m}$ and $\mu,
  \lambda \in \C^m$.  Suppose we also have the following conditions and data:
\begin{enumerate}
\item $0 < K_L \le \abs{\mu_i}, \abs{\lambda_i} \le K_U$.
\item $ \abs{\mu_i / \lambda_i -  \mu_j / \lambda_j} > \Omega > 0$ for all pairs $i \neq j$.
\item $0 < K <1$ and $E_\mu, E_\lambda \in \C^{m \times m}$ such that
  $\norm{E_\mu}_F, \norm{E_\lambda}_F \leq K$.
\item $6 \kappa(Q)^3 \cdot \frac{K_U}{K_L} \cdot m \cdot \frac{K}{\sigma_m(Q)^2 K_L} \leq \Omega$.
\item $    K \leq \sigma_m(Q)^2 K_L/2.$
\end{enumerate}
Then $(N_\mu + E_\mu) (N_\lambda + E_\lambda)^{-1}$ is diagonalizable and hence has $n$ eigenvectors. 
\end{theorem}

\begin{proof}
Defining $F_\mu := (Q\diag{\mu_i}Q^T)^{-1} E_\mu$, and similarly $F_\lambda$, we have

\begin{align}
(N_\mu + E_\mu) (N_\lambda + E_\lambda)^{-1} 
&= (Q \diag{\mu_i} Q^T + E_\mu) (Q \diag{\lambda_i} Q^T + E_\lambda)^{-1} \nonumber \\
&= Q\diag{\mu_i}Q^T (I+F_\mu) (I+F_\lambda)^{-1} (Q\diag{\lambda_i} Q^T)^{-1} \nonumber \\
&= Q\diag{\mu_i}Q^T (I+F_\mu) (I+G_\lambda) (Q\diag{\lambda_i} Q^T)^{-1} \label{eq:inversion_approx}\\
&= Q\diag{\mu_i/\lambda_i}Q^{-1} + \mathsf{Err}. \nonumber
\end{align}
In \eqref{eq:inversion_approx} above $G_\lambda = (I+F_\lambda)^{-1}-I$; hence by Claim~\ref{claim:error_matrix_inversion} (which is applicable because 
$\norm{F_\lambda}_F \leq \frac{K}{\sigma_m(Q)^2 K_L} \leq 1/2$, by our assumption) we have 
$\norm{G_\lambda}_F \leq (m+1)\norm{F_\lambda}_F$. 
The norm of $\mathsf{Err}$ 
then satisfies 

\begin{align} \label{eqn:upper_bound_err_norm2}
\norm{\mathsf{Err}}_F 
&\leq \frac{\sigma_1(Q)^2}{\sigma_m(Q)^2}\cdot\frac{K_U}{K_L} \left(\norm{F_\mu}_F + 
(m+1)\norm{F_\lambda}_F + (m+1)\norm{F_\mu}_F\cdot\norm{F_\lambda}_F\right) \nonumber\\
&\leq  \kappa(Q)^2 \cdot \frac{K_U}{K_L} \cdot 2m \cdot \frac{K}{\sigma_m(Q)^2 K_L}.
\end{align}

Now note that 
$3 \kappa(Q) \norm{\mathsf{Err}}_2 \leq 6 \kappa(Q)^3 \cdot \frac{K_U}{K_L} \cdot m \cdot \frac{K}{\sigma_m(Q)^2 K_L} \leq \Omega$ by our assumption and so Lemma~\ref{thm:weyl} is applicable with matrices 
$Q\diag{\mu_i/\lambda_i}Q^{-1}$, $Q$, and $\mathsf{Err}$ playing the roles of 
$A$, $X$, and $E$, resp. 
Lemma~\ref{thm:weyl} gives us a permutation $\pi: [m] \to [m]$ such that  
\begin{align} \label{eqn:application_Bauer_Fike}
\abs{\nu_{\pi(i)}(Q\diag{\mu_i/\lambda_i}Q^{-1}+\mathsf{Err})-\nu_i(Q\diag{\mu_i/\lambda_i}Q^{-1})}  \leq \kappa(Q)\norm{\mathsf{Err}}_2 < \Omega/2,
\end{align}
where $\nu_i(M)$ denotes an eigenvalue of $M$.   

Hence all the eigenvalues of $(N_\mu + E_\mu) (N_\lambda + E_\lambda)^{-1}$ are distinct. 
By Lemma~\ref{lemma:diagonalizable}, it has $n$ linearly independent eigenvectors $\{ v_1, \ldots, v_n \}$. 
\end{proof}

\section{Underdetermined ICA}\label{sec:underdetermined}
In this section we give our algorithm for the underdetermined ICA problem and 
analyze it. In the
underdetermined case, there are more independent source variables than
there are measurements, thus $A$ has fewer rows than columns. We have
to be more careful about fixing the normalization and scaling of the
model than in the fully determined case where isotropic position
provides a convenient normalization for $x, A$ and $s$.
\begin{problem}[Underdetermined ICA] Fix $n,m \in \N$ such that $n
  \le m$. We say that $x \in \R^n$ is generated by an underdetermined
  ICA model if $x = As$ for some fixed matrix $A \in \R^{n \times m}$
  where $A$ has full row rank and $s \in \R^m$ is a fully independent
  random vector.  In addition, we fix the normalization so that each
  column $A_i$ has unit norm.  The problem is to recover the columns
  of $A$ from independent samples $x$ modulo phase factors.
\end{problem}
Additional assumptions are needed for the essentially unique
identifiability of this model. For example, suppose that columns $A_i$
and $A_j$ are parallel i.e., $A_i = c A_j$, then one could replace the
variables $s_i$ and $s_j$ with $s_i + c s_j$ and $0$ and the model
would still be consistent. We introduce the following sufficient
condition for identifiability: we require that the $m$ column vectors
of $A^{\odot k}$ be linearly independent for some $k > 0$ (smaller $k$
would be better for the efficiency of the algorithm).  We make this
quantitative by requiring that the $m$'th singular value satisfy
$\sigma_m (A^{\odot k}) >0$.

Our approach to the underdetermined ICA problem is to apply our tensor
decomposition to a pair of carefully-chosen tensors that arise from
the distribution. The tensors we use are the derivative tensors of the
second charateristic function $\psi_x(u) = \log\left( \E{
  e^{iu^Tx}}\right)$. 

This method generalises the fourth moment methods for ICA where one
computes the local optima of the following quartic form:
\begin{align*}
  f(u) = \E{(x^Tu)^4} - 3 \E{ (x^Tu)^2}^2.
\end{align*}
An equivalent formulation of this is to consider the tensor $T \in
\R^{n \times n \times n \times n}$ which represents this quartic form
(just as in the matrix case where symmetric matrices represent
quadratic forms, symmetric tensors of order $4$ represent quartic forms). Let
us denote our overall tensor representing $f(u)$ by $T$ where $f(u) =
T(u,u,u,u)$. By a relatively straightforward calculation, one can
verify that $T(u,u,u,u)$ is the fourth
derivative of the second characteristic function of $x$ evalauted at 0:
\begin{align*}
  T = D^4_u \psi_x(0).
\end{align*}
On the other hand, one can also verify that $T$ has the following
decomposition (see for example \cite{aghkt12}):
\begin{align*}
  T = \sum_{j=1}^m \left( \E{ s_i^4} - 3\E{s_i^2}^2\right) A_i \otimes A_i
  \otimes A_i \otimes A_i
\end{align*}
So in fact, one can view the fourth moment tensor methods as
performing the tensor decomposition of only one tensor -- the fourth
derivative of $\psi$ evaluated at 0!

Our method also generalises the algorithm we gave for the fully
determined case in Section \ref{sec:ica}. We can view that case as
simply being the second derivative version of the more general
algorithm. The techniques used in this section are generalisations and
refinements of those used in the fully determined case, though
replacing the easy matrix decomposition arguments with the
corresponding (harder) tensor arguments.

A property of the second characteristic function that is central for our algorithm
is that the second characteristic function of a random vector with independent components
factorizes into the sum of the second characteristic functions of each component:
\begin{align*}
  \log\left(\E{e^{i u^Ts}}\right) = \sum_{j=1}^m \log\left(\E{e^{iu_j s_j}}\right), 
\end{align*}
and now every mixed partial derivative (with respect to $u_j$ and
$u_{j'}$) is 0, as each term in the sum depends only on one component
of $u$. Taking the derivative tensor will result in a diagonal tensor
where the offdiagonal terms are all 0. In the case when we're
interested in $x=As$, we simply need to perform the change of basis
via $A$ very carefully for the derivative tensors via the chain
rule. One could also try to perform this analysis with the moment
generating function $\E{e^{u^T x}}$ without the complex phase. The
difficulty here is that the moment generating functions exists only if
all moments of $x$ exist, and thus a moment generating function
approach would not be able to deal with heavy tailed
distributions. Moreover, using a real exponential leads us to estimate
exponentially large quantities from samples, and it is difficult to
get good bounds on the sample complexity. Using the complex
exponential avoids these problems as all quantities have modulus 1.

We will then apply our tensor
decomposition framework: as before we show that the eigenvalues
of the flattened derivative tensors are well spaced
in Section~\ref{sec:underspacings}. We then study the sample
complexity in Section~\ref{subsec:sample} and assembling these
components in Section~\ref{sec:proof}.

\subsection{Algorithm}\label{subsec:icaalgorithm-detailed}
For underdetermined ICA we compute the higher derivative
tensors of the second characteristic function $\psi_x(u) = \log(
\phi_x(u))$ at two random points and run the tensor decomposition algorithm from 
the previous section. 

\begin{center}
\fbox{\parbox{\textwidth}{
\begin{minipage}{5in}
\vspace{0.1in}
{\bf Underdetermined ICA}($\sigma$)

\begin{enumerate}

\item (Compute derivative tensor) Pick independent random vectors 
$\alpha, \beta \sim N(0, \sigma^2 I_n)$. For even $d$, estimate the $d^{th}$ derivative tensors 
 of $\psi_x(u)$ at $\alpha$ and $\beta$ as $T_\alpha = D^{d}_u\psi_x(\alpha)$ and
$T_\beta = D^{d}_u\psi_x(\beta)$. 

\item (Tensor decomposition) Run \textbf{Tensor Decomposition$(T_\alpha,T_\beta)$}.





\end{enumerate}
\end{minipage}
}}
\end{center} 
To estimate the
$2d^{th}$ derivative tensor of $\psi_x(u)$ empirically, one simply
writes down the expression for the derivative tensor, and then
estimates each entry from samples using the naive estimator.

More precisely, we can use 
\begin{align*}
  \frac{ \partial \phi(u)}{\partial u_i} = \E{ (ix_i) e^{i u^T x}}.
\end{align*}
This states that differentiation in the Fourier space is equivalent to
multiplication in the original space, thus it suffices to estimate
monomials of $x$ reweighted by complex exponentials. To estimate the
$d^{th}$ derivative tensor of $\log (\phi(u))$ empirically, one simply
writes down the expression for the derivative tensor, and then
estimates each entry from samples using the naive estimator. Note that
the derivatives can be somewhat complicated, for example, at fourth
order we have
\begin{align*}
  & [D^4 \psi_u]_{i_1,i_2,i_3,i_4} \\
  & = \frac{1}{\phi(u)^4} \left[ \E{ (ix_{i_1}) (ix_{i_2}) (ix_{i_3}) (ix_{i_4}) \exp(i
      u^T x)} \phi(u)^3 \right. \\
  & \qquad -\E{ (ix_{i_2}) (ix_{i_3}) (ix_{i_4}) \exp(i u^T x)} \E{
    (ix_{i_1}) \exp(i u^T x)} \phi(u)^2 \\
  & \qquad -\E{ (ix_{i_2}) (ix_{i_3}) \exp(i u^T x)} \E{ (ix_{i_1}) (ix_{i_4}) \exp(i
    u^T x)} \phi(u)^2 \\
  & \qquad - \E{ (ix_{i_2}) (ix_{i_4}) \exp(i u^T x)} \E{ (ix_{i_1}) (ix_{i_3}) \exp(i
    u^T x)} \phi(u)^2 \\
  & \qquad -\E{ (ix_{i_2}) \exp(i u^T x)} \E{ (ix_{i_1}) (ix_{i_3}) (ix_{i_4}) \exp(i
    u^T x)} \phi(u)^2 \\
  & \qquad -\E{ (ix_{i_3}) (ix_{i_4}) \exp(i u^T x)} \E{ (ix_{i_1}) (ix_{i_2}) \exp(i
    u^T x)}  \phi(u)^2 \\
  & \qquad -\E{ (ix_{i_3}) \exp(i u^T x)} \E{ (ix_{i_1}) (ix_{i_2}) (ix_{i_4}) \exp(i
    u^T x)} \phi(u)^2 \\
  & \qquad -\E{ (ix_{i_4}) \exp(i u^T x)} \E{ (ix_{i_1}) (ix_{i_2}) (ix_{i_3})
    \exp(i u^T x)} \phi(u)^2 \\
  & \qquad +2 \E{ (ix_{i_3}) (ix_{i_4}) \exp(i u^T x)} \E{ (ix_{i_2}) \exp(i u^T
    x)} \E{ (ix_{i_1}) \exp(i u^T x)}
  \phi(u) \\
  & \qquad +2 \E{ (ix_{i_3}) \exp(i u^T x)} \E{ (ix_{i_2}) (ix_{i_4}) \exp(i u^T
    x)} \E{ (ix_{i_1}) \exp(i u^T x)} \phi(u)\\
  & \qquad+2 \E{ (ix_{i_4}) \exp(i u^T x)} \E{ (ix_{i_2}) (ix_{i_3}) \exp(i u^T
    x)} \E{ (ix_{i_1}) \exp(i u^T x)}
  \phi(u) \\
  & \qquad +2 \E{ (ix_{i_3}) \exp(i u^T x)} \E{ (ix_{i_2}) \exp(i u^T x)} \E{
    (ix_{i_1}) (ix_{i_4}) \exp(i u^T x)} \phi(u)\\
  & \qquad +2 \E{ (ix_{i_4}) \exp(i u^T x)} \E{ (ix_{i_2}) \exp(i u^T x)} \E{
    (ix_{i_1}) (ix_{i_3}) \exp(i
    u^T x)} \phi(u)  \\
  & \qquad +2 \E{ (ix_{i_4}) \exp(i u^T x)} \E{ (ix_{i_3}) \exp(i u^T
    x)} \E{ (ix_{i_1}) (ix_{i_2}) \exp(i u^T x)} \phi(u) \\
  & \qquad \left.-6 \E{ (ix_{i_1}) \exp(i u^T x)}\E{ (ix_{i_2}) \exp(i u^T
      x)}\E{ (ix_{i_3}) \exp(i u^T x)}\E{ (ix_{i_4}) \exp(i u^T x)}\right].
\end{align*}
The salient points are described in Lemma \ref{lemma:errorestimate}
and Claim \ref{claim:errorestimate}: there are at most $2^{d-1}
(d-1)!$ terms (counting multiplicities), and no term has combined
exponents of $x_i$ in all it factors higher than $d$. We will give a
rigorous analysis of the sampling error incurred in Section
\ref{subsec:sample}.

\subsection{Truncation error}

\begin{lemma}\label{lemma:multivariate_errorestimate}
  Let $s = (s_1, \ldots, s_m) \in \R^m$ be a random vector with indpendent components 
each with finite $k$ absolute moments, and for $t \in \R^m$ let $\phi(t) = \E{e^{it^Ts}}$ be the 
associated characteristic function. Then for $k \geq 1$ and  $i_1, \ldots i_k \in [m]$ we have
  \begin{align*}
\abs{\partial_{i_1, \ldots, i_k} \log \phi(t)} \le \frac{2^{k-1}
      (k-1)! \max_{j \in [m]} \E{ \abs{s_j}^k}}{\abs{ \phi(t)}^k}.
\end{align*}
\end{lemma}

\begin{proof}
  To compute the derivatives of $\log \phi(t)$ we proceed
  inductively with $\partial_{i_1} \log \phi(t) = (\partial_{i_1} \phi(t))/\phi(t)$ as our base
  case. For $d < k$, write $\partial_{i_1, \ldots, i_d} (\log \phi)$ as $N_d(t)/\phi(t)^d$. Then we have

  \begin{align}  \label{eqn:Ndt} \begin{split}
    \partial_{i_1, \ldots, i_{d+1}}\log \phi(t) &= \partial_{i_{d+1}} \left(\frac{N_d(t)}{\phi(t)^d}\right) \\
&= \frac{ (\partial_{i_{d+1}} N_d(t)) \phi(t)^d - N_d(t) d \phi(t)^{d-1} \partial_{i_{d+1}} \phi(t)}{\phi(t)^{2d}} \\
&= \frac{ (\partial_{i_{d+1}} N_d(t)) \phi(t) - d N_d(t) \partial_{i_{d+1}} \phi(t)}{\phi(t)^{d+1}}. \end{split} \end{align}

  We make the following claim about $N_d(t)$:
  \begin{claim}\label{claim:multivariate_errorestimate}
    The functions $N_d(t)$ is the sum of terms of the form 
$C_{S_1, \ldots, S_d} \partial_{S_1}\ldots \partial_{S_d} \phi(t)$ where multisets $S_1, \ldots, S_d \subseteq \{i_1, 
\ldots, i_d\}$ (this is a multiset) satisfy $S_1 \cup \ldots \cup S_d = \{i_1, \ldots, i_d\}$, and $C_{S_1, \ldots, S_d}$ are integer
coefficients with $\sum \abs{C_{S_1, \ldots, S_d}} \leq 2^{d-1} (d-1)!$.
  \end{claim} 
  \begin{proof}
The first part follows via induction on $d$ and \eqref{eqn:Ndt}. For the second part, let $T(d)$ denote 
$\sum \abs{C_{S_1, \ldots, S_d}}$.  Note that $T(1)=1$. Then by \eqref{eqn:Ndt}, we have 
$T(d+1) \leq d T(d) + d T(d)$, which gives $T(d) \leq 2^{d-1} (d-1)!$. 
\end{proof}

  Returning to the proof of Lemma~\ref{lemma:multivariate_errorestimate}, we
  observe that for any multiset $S$ with elements from $[m]$ and size at most $k$, we have
\begin{align*}
\abs{\partial_S \phi(t)} = \abs{\E{i^{\abs{S}} s_S e^{i t^T s}}} \leq \E{\abs{s_S}}.
\end{align*}
  For $\ell \in [m]$, let $p_\ell$ be the number of times $\ell$ occurs in the multiset $\{i_1, \ldots, i_d\}$. 
For each term of $N_d(t)$ we have 
 \begin{align}\label{eq:partial_ubd}\begin{split}
\abs{\prod_{j=1}^d \partial_{S_j}\phi} &= \prod_{j=1}^d \abs{\partial_{S_j} \phi} \\
& \leq \prod_{j=1}^d \E{\abs{s_{S_j}}} \\
& = \prod_{\ell=1}^m \E{\abs{s_\ell}^{p_\ell}} \\
& \leq \prod_{\ell=1}^m \left(\E{\abs{s_\ell}^d} \right)^{p_\ell/d} \\
& \leq \max_{\ell \in [m]} \E{\abs{s_\ell}^d}.
\end{split}\end{align}
The second equality above uses the independence of the $s_\ell$, and the second inequality uses the 
first part of Fact~\ref{fact:holder}.

Thus $\abs{N_d(t)} \leq 2^{(d-1)}(d-1)! \max_{\ell \in [m]} \E{\abs{s_\ell}^d}$, which when divided by $\phi(t)^d$ gives
the required bound. 

\end{proof}

\subsection{Eigenvalue spacings}\label{sec:underspacings}
In this subsection we examine the anti-concentration of the diagonal entries $\psi_i^{(d)}((A^Tu)_i)$. 
The analysis has similarities to the fully-determined case but there are also some
 major differences: in the fully-determined case, $A_i^Tu$
and $A_j^T u$ are independent Gaussians because the columns of $A$
are orthogonal by isotropic position (recall that we defined $A_i^T$ to mean $(A_i)^T$). 
We can not make $A$ an
orthonormal matrix in the underdetermined case, so we have to exploit
the more limited randomness. An additional complication is that we are working with anti-concentration
of the quotients of such diagonal entries rather than the entries themselves.

\begin{theorem}\label{thm:anti-concentration} 
  Let $s \in \R^m$ be a random vector with independent components. For
  $t \in \R^m$ and $d \in 2\N$ let $\psi_a(t) := \log \E{e^{it_a
      s_a}}$, and $g_a(t_a) := \frac{d^d \psi_a(t_a)}{d t_a^d}$ for
  all $a \in [m]$.  Let $0 < \delta$. Suppose that the following data and conditions are 
given:
  \begin{enumerate}
    \item $\E{ s_a} = 0$, $\E{ s_a^2} \le M_2$ and $\E{ s_a^d} \le M_d$ for $a \in [m]$ and $M_2 < M_d$.
    \item $k \geq 2$ and for all $a \in [m]$, $k_a \in \N$ 
     where $d < k_a < k$, such
      that $\abs{\cum_{k_a}(s_a)} \ge \Delta$.
    \item $\E{\abs{s_a}^{k_a + 1}} \le M_{k}$ for $a \in [m]$ and $M_2 < M_k$.
\item $A \in \R^{n \times m}$ be a full row
  rank matrix whose columns all have unit norm and $1 - \ip{A_a}{A_b}^2
  \ge L^2$ for all pairs of columns.
\item $u, v \sim N(0, \sigma^2 I_n)$ sampled independently where
\begin{align*}
\sigma \leq \min\left(1, \frac{1}{2\sqrt{2M_2 \log 1/q}}, \sigma'\right),
\end{align*}
and
\begin{align*}
\sigma' = \Delta \frac{k-d+1}{k!} \left(\frac{3}{8}\right)^k \frac{1}{M_k} 
\left(\frac{Lq \sqrt{2\pi}}{4(k-d)} \right)^{k-d} \left(\frac{1}{\sqrt{2\log 1/q}}\right)^{k-d}
\end{align*}
  \end{enumerate}
and $0 < q < 1/3$.
  Then with probability at least $1- 3 {m \choose 2}q$ we have
  \begin{align}\label{eqn:spacing}
\abs{ \frac{g_b( A_b^T u) }{ g_b( A_b^T v)} - \frac{g_a( A_a^T u)
      }{ g_a( A_a^T v)} } \geq \Delta \frac{1}{(k-d)!(d-1)!}\left(\frac{3}{8}\right)^d \frac{1}{M_d}
\left(\frac{\sigma L q \sqrt{2\pi}}{4(k-d)} \right)^{(k-d)},
      \end{align}
      for all distinct $a, b \in [m]$. (We count the small probability
      case where $g_b( A_b^T v)=0$ or $g_a( A_a^T v)=0$ as violating
      the event in \eqref{eqn:spacing}.) 
\end{theorem} 

\begin{proof}
  Fix  $a \neq b \in [m]$ and
  show that the spacing in \eqref{eqn:spacing} is lower bounded for
  this pair with high probability. We will then take a union bound
  over all $\binom{m}{2}$ pairs, which will give the desired result.

 For random choice of $v$, the events 
\begin{align} \label{eqn:good_events}
g_a( A_a^T v) \neq 0 \text{ and } g_b( A_b^T v) \neq 0
\end{align}
 have
 probability $1$. Thus in the following we will assume that these events are true. 

We will need concentration of $(A_a^T u)$ and of $(A_a^T v)$.
\begin{align*}
\probsub{u \sim N(0,\sigma^2)}{\abs{A_a^T u} > \tau} \leq \sqrt{\frac{2}{\pi}} \sigma^2 \norm{r}^2 \frac{1}{\tau} e^{-\frac{\tau^2}{2 \sigma^2 \norm{r}^2}} \leq \sqrt{\frac{2}{\pi}} \sigma^2 \frac{1}{\tau} e^{-\frac{\tau^2}{2 \sigma^2}},
\end{align*}
here the first inequality used Claim~\ref{claim:gaussian_concentration} and the second
used the fact that $\norm{r} \leq 1$.  Substituting 
$\tau =  \sigma\sqrt{2\log{1/q}}$ we get
\begin{align*}
\prob{\abs{A_a^T u} \leq \sigma\sqrt{2\log{1/q}}} 
\geq 1 - \frac{\sigma q}{\sqrt{\pi \log{1/q}}} 
\geq  1 - \frac{q}{\sqrt{\pi \log{1/q}}}.
\end{align*}

By the union bound we have
\begin{align} \label{eqn:concentrationAuAv}
\prob{\abs{A_a^T u}, \abs{A_a^T v} \leq \sigma\sqrt{2\log{1/q}}} 
\geq 1 - \frac{2q}{\sqrt{\pi \log{1/q}}}.
\end{align}
In the sequel we will assume that the event in the previous expression happens.

Under the assumption that $\abs{A_a^T v} \leq \sigma\sqrt{2\log 1/q}$ we have
\begin{align} \label{eq:denominator}
\abs{g_a(A_a^T v)} = \abs{\psi^{(d)}(A_a^T v)} \leq \frac{2^{d-1} (d-1)! M_d}{(3/4)^d},
\end{align}
where to upper bound $\abs{\psi^{(d)}(A_a^T u)}$ we used Lemma~\ref{lemma:errorestimate},
Lemma~\ref{lem:phi_nonvanishing}, and the condition $\sigma \sqrt{2 \log 1/q} \leq
\frac{1}{2\sqrt{M_2}}$ which follows from our assumption on $\sigma$.

  To compute the probability that the spacing is at least $\eps_a$, where $\eps_a>0$ will
be chosen later, we
  condition on fixing of $A_b^T u = z$ and any fixing of $v$:

  \begin{align*}
    \prob{\abs{ \frac{g_a( A_a^T u) }{ g_a( A_a^T v)} - \frac{g_b(
          A_b^T u) }{ g_b( A_b^T v)} }\le \eps_a} = \int \prob{ \abs{
        \frac{g_a( A_a^T u) }{ g_a( A_a^T v)} - \frac{g_b(z)}{g_b( A_b^T
          v)}} \le \eps_a \suchthat A_b^T u = z } \prob{A_b^T u = z}dz.
  \end{align*}
  We will bound the conditional probability term.  As in
  the proof of Theorem~\ref{thm:anti-concentration}, applying Taylor's
  theorem with remainder (Theorem~\ref{thm:taylor}) gives
  \begin{align*}
    g_a(A_a^T u) = i^d \sum_{l=d}^{k_a} \cum_l(s_a) \frac{
      (i(A_a^Tu))^{l-d}}{(l-d)!}  + R_{k_a+1}(A_a^T u) \frac{(A_a^T
      u)^{k_a-d+1}}{(k_a-d+1)!}.
  \end{align*}
  Truncating $g_a$ after the degree $(k_a-d)$ term yields 
  polynomial $p_a(A_a^Tu)$. Denote the truncation error by $\rho_a(A_a^T u)$.

Then setting $h = \frac{g_b(A_b^T u) g_a( A_a^T v)}{g_b( A_b^T v)}$ which is a constant under our 
conditioning, we have
\begin{align*}
\abs{ \frac{g_a( A_a^T u) }{ g_a( A_a^T v)} - \frac{g_b(A_b^T u) }{ g_b( A_b^T v)} }
&= \frac{1}{\abs{g_a( A_a^T v)}} \abs{g_a( A_a^T u) - \frac{g_b(A_b^T u) g_a( A_a^T v)}{g_b( A_b^T v)}} \\
&=  \frac{1}{\abs{g_a( A_a^T v)}} \abs{g_a( A_a^T u) - h}  \\
&= \frac{1}{\abs{g_a( A_a^T v)}} \abs{p_a(A_a^Tu) + \rho_a(A_a^Tu) -h} \\
& \geq \frac{1}{\abs{g_a( A_a^T v)}} \abs{p_a(A_a^Tu) -h} - \frac{1}{\abs{g_a( A_a^T v)}} \abs{\rho_a(A_a^Tu)}. 
\end{align*}

Now we are going to show that the first term above is likely 
to be large and the second one is likely to be small. 
 
We have $A_a^T u = \ip{A_a}{A_b} A_b^T u + r^T u$ where $r$ is a vector orthogonal
  to $A_b$. Our hypothesis, namely $1 - \ip{A_a}{A_b}^2
  \ge L^2$, gives $\norm{r}^2 \ge L^2$.
The orthogonality of $r$ and $A_b$ implies that the univariate
  Gaussian $r^T u $ is independent of $A_b^T u$.  

  Now we apply our anti-concentration inequality from
  Theorem~\ref{thm:anti} to obtain (for $u \sim N(0, \sigma^2 I_n)$ and fixed $v$ satisfying 
\eqref{eqn:good_events})
  \begin{align} \nonumber
    \prob{ \abs{p_a(A_a^T u) - h} \le \epsilon_a
     \suchthat A_b^T u = z}
& \le
    \frac{4(k_a-d)}{\sigma \norm{r} \sqrt{2 \pi}} \left( \frac{\eps_a (k_a-d)!}{\abs{\cum_{k_a}(s_a)}}
   \right)^{1/(k_a-d)} \\
    & \le \frac{4(k_a-d)}{\sigma L \sqrt{2 \pi}} \left( \frac{
        \eps_a (k_a-d)!}{\Delta} \right)^{1/(k_a-d)}.
 \label{eqn:anticoncentration_quotient}
  \end{align} 
  We choose
  \begin{align*}
   \eps_a := \frac{\Delta}{(k_a-d)!}\left(\frac{\sigma L q \sqrt{2\pi}}{4(k_a-d)}\right)^{k_a - d} \geq  \frac{\Delta}{(k-d)!}\left(\frac{\sigma L q \sqrt{2\pi}}{4(k-d)}\right)^{k - d} =: \epsilon,
  \end{align*}
making RHS of \eqref{eqn:anticoncentration_quotient} equal to $q$. Recall that this was
for fixed $v$ satisfying \eqref{eqn:good_events}.


For the truncation error, assuming that the event $\abs{A_a^T u} \leq \sigma\sqrt{2\log{1/q}}$ happens, 
we have
\begin{align*}
  \abs{\rho_a(A_a^T u)}
  &\leq \abs{\psi^{(k_a+1)}(A_a^T u)} \cdot \frac{(A_a^T u)^{k_a-d+1}}{(k_a-d+1)!} \\
  &\leq \frac{2^{k_a}M_{k_a+1}}{(3/4)^{k_a+1}}\cdot \frac{k_a!}{(k_a-d+1)!}\cdot \left(\sigma
    \sqrt{2 \log{1/q}}\right)^{k_a-d+1} \\
&\leq \epsilon_a/2,
\end{align*}
where to upper bound $\abs{\psi^{(k_a+1)}(A_a^T u)}$ we used Lemma~\ref{lemma:errorestimate},
Lemma~\ref{lem:phi_nonvanishing}, and the condition $\sigma \sqrt{2 \log(1/q)} \leq
\frac{1}{2\sqrt{M_2}}$, which  holds given our upper bound on $\sigma$. And the final inequality
follows from our condition $\sigma \leq \sigma'$. 



Thus with probability at least $1- \frac{2q}{\sqrt{\pi \log 1/q}} - q$ we have
$\abs{p_a(A_a^Tu)-h}-\abs{\rho_a(A_a^Tu)} \geq \epsilon_a/2$ under the condtion that $A_b^Tu=z$ and 
$v$ fixed. Now since this holds for any $z$ and any fixing of $v$, it also holds without the conditioning 
on the event $A_b^Tu=z$ and fixing of $v$.

Hence using \eqref{eq:denominator}, with probability at least 
$1- \frac{2q}{\sqrt{\pi \log 1/q}} - q \geq 1-3q$  
we have
\begin{align*}
\frac{1}{\abs{g_a(A_a^Tv)}}\left(\abs{p_a(A_a^Tu)-h}-\abs{\rho_a(A_a^Tu)} \right) 
\geq \epsilon_a \cdot (3/8)^d \frac{1}{(d-1)! M_d} 
\geq \epsilon \cdot (3/8)^d \frac{1}{(d-1)! M_d}.
\end{align*}

To summarize, with probability at least $1 - 3q$ the spacing is at least $\epsilon$. By the union bound,
with probability at least $1- 3 {m \choose 2} q$ all the spacings are at least $\epsilon$.


\end{proof}

The following is a straightforward corollary of the proof of the previous theorem.
\begin{corollary} \label{cor:K_L}
In the setting of Theorem~\ref{thm:anti-concentration} we have with probability at least
$1 - 6mq$ that
\begin{align*}
\abs{g_a(A^Tu)}, \abs{g_a(A^Tv)} \geq \frac{\Delta_0}{2 (k-d)!}\left(\frac{\sigma q \sqrt{2\pi}}{4 (k-d)} \right)^{k-d}
\end{align*}
for all $a \in [m]$. 
\end{corollary}

An important part of the proof is to give a lower bound on the quantity
$ 1 - \ip{A_i}{A_j}^2 \ge L^2$ so that the ICA model remains
identifiable. At order $d$, we will give our bounds in terms of
$\sigma_m \left( A^{\odot j} \right)$ for $j=1,\ldots,d/2$.
\begin{lemma}\label{lemma:identifiable}
  Fix $m,n \in \N$ such that $ n \le m$. Let $A \in \C^{n \times m}$
  be a full row rank matrix whose columns $A_i$ have unit norm. Then
  \begin{align*}
    1 - \ip{A_i}{A_j}^2 \ge \frac{2}{k} \sigma_m \left( A^{\odot k} \right)^2,
  \end{align*}
  for all $k \in \N$ where $k \ge 2$.
\end{lemma}
\begin{proof}
  Consider the matrix $B = A^{\odot 2}$, observe that $\ip{A_i}{A_j}^2 = \ip{B_i}{B_j}$. 
Define the matrix $C = A^{\odot k}$. Observe that $\norm{C_i} = \norm{A_i}^k =1$ and that
  \begin{align}\label{eqn:BC}
    1 - \ip{B_i}{B_j} =  1 - \abs{\ip{C_i}{C_j}}^{2/k}
  \end{align}
  for $k \ge 2$. 
Recall that
  \begin{align*}
    \sigma_m(C) = \min_{\norm{x} = 1} \norm{Cx}.
  \end{align*}
  In particular, if we consider the vector $x =
  \frac{1}{\sqrt{2}}(e_i \pm e_j)$ we have
  \begin{align*}
    \norm{Cx}^2 = \frac{1}{2} \left( \norm{C_i}^2 + \norm{C_j}^2 \pm 2
      \ip{C_i}{C_j}\right) = 1 \pm \ip{C_i}{C_j} \ge \sigma_m(C)^2.
  \end{align*}
  Hence we must have $ 1 - \abs{\ip{C_i}{C_j}} \ge \sigma_m^2(C)$.  
Combining this with \eqref{eqn:BC}, we obtain
    \begin{align*}
      1 - \ip{B_i}{B_j} & = 1 - \abs{ \ip{C_i}{C_j}}^{2/k} \\
    & \ge 1 - ( 1 - \sigma_m(C)^2)^{2/k} \\
    & \ge \frac{2}{k} \sigma_m(C)^2,
    \end{align*}
    where the last step follows from noting that all derivatives of
    the function $f(x) = ( 1 - x)^t$ for $t \in (0,1)$ are negative in
    the interval $x \in [0,1]$
  \end{proof}

\subsection{Sample complexity}\label{subsec:sample}

To understand the complexity of our algorithm, we must bound how many
samples are needed to estimate the matrices $M_u$ and $M_v$
accurately. Throughout this section, we estimate $\E{ f(x)}$ for some function $f(x)$, 
using independent samples $x^i$ via
\begin{align*}
  \bar{E}(f(x)) := \frac{1}{N} \sum_{i=1}^N f(x^i) \to \E{f(x)}.
\end{align*}
More generally, we will estimate derivative tensors as follows. 
As before, $\phi(t) = \E{e^{i t^Ts}}$ and define 
the empirical version of the characteristic function in the natural way
$\bar{\phi}(t):= \frac{1}{N} \sum_{i=1}^N e^{it^T s^i}$.
As we will see, for a multiset $S \subseteq [m]$ the derivative of $\bar{\phi}(t)$ behaves nicely
and will serve as an approximation of $\phi(t)$. 
Note that
\begin{align*}
\partial_S\bar{\phi}(t) = \bE{s_S e^{i t^Ts}},
\end{align*}
where $\bE{\cdot}$ denotes empirical expectation over $N$ i.i.d. samples of $s$.
Similarly, we estimate $\partial_S\log \phi(t)$ by
\begin{align} \label{eqn:partial_estimate}
\partial_S\log\bar{\phi}(t)= \bar{N}_d(t)/\bar{\phi}(t)^d, 
\end{align}
where by Claim~\ref{claim:multivariate_errorestimate} $\bar{N}_d(t)$ is a sum of the form 
$\sum_{S_1, \ldots, S_d} C_{S_1, \ldots, S_d} (\partial_{S_1}\bar{\phi}(t))\ldots (\partial_{S_d} \bar{\phi}(t))$, as described in Claim~\ref{claim:multivariate_errorestimate}.
Thus to show that $\partial_S\bar{\phi}(t)$ is a good approximation of $\partial_S\phi(t)$ we show that
$\abs{\frac{N_d(t)}{\phi(t)^d}-\frac{\bar{N}_d(t)}{\bar{\phi}(t)^d}} = 
\frac{\abs{\bar{\phi}(t)^dN_d(t)-\phi(t)\bar{N}_d(t)}}{\phi(t)^d\bar{\phi}(t)^d}$ is small. 

The notion of empirical estimate of a derivative tensor now follows immediately from 
\eqref{eqn:partial_estimate} which gives how to estimate individual entries of the tensor.


\begin{lemma} \label{lem:underdetermined_sample_approx_entry} 
  Let $s \in \R^m$ be a random vector with independent components. For
  $t \in \R^m$ let $\psi_s(t) = \phi_s(t) = \log \E{e^{it^ts}}$ be the
  second characteristic function of $s$. Consider the $d$th derivative
  tensor of $\psi_s(t)$ (it contains $m^d$ entries). Let $M_2, M_{2d} > 0$ be such that
  $\E{s_i^2} \leq M_2$ and $\E{\abs{s_i}^{2d}} \leq M_{2d}$. Fix $0 < \eps,
  \delta < 1/4$, and let $\norm{t} \leq \frac{1}{\sqrt{2M_2}}$.
  Suppose we take $N$ samples then with probability at least 
\begin{align*}
 1 -  \binom{m+d-1}{d} \frac{2^d M_{2d}}{\eps^2 \delta} \left[ \frac{2^d
      d (M_d+2)^{d-1} (d-1)!}{(3/4)^d(1/2)^d}  \right]^2,
\end{align*}
every entry of the empirical tensor will be within $\eps$ of the corresponding entry of the 
derivative tensor.
\end{lemma}

\begin{proof}
In light of \eqref{eqn:partial_estimate} we will prove that each term in the expression for $\bar{N}_d(t)$ 
(it's a product of several $\partial_S \bar{\phi}(t)$) is a good approximation of 
the corresponding term in the expression for $N_d(t)$ by showing that the corresponding factors in the product are
close. Finally, we show that the whole sum is a good approximation. 
For complex-valued r.v. $X$ with mean $\mu$, note that 
$\Vr{X} = \E{(X-\mu)(X-\mu)^*} = \E{XX^*} - \mu\mu^* \leq \E{\abs{X}^2}$.

We use the second moment method to prove that $\partial_S \bar{\phi}(t)$ is a good approximation of $\partial_S \phi(t)$
with high probability.

For multiset $S \subseteq \{i_1, \ldots, i_d\}$, with $p_j$ being the frequency of $j$ in $S$, by the same arguments as 
in \eqref{eq:partial_ubd} we have

\begin{align*}
\Vr{s_S e^{i t^T s}} &\leq \E{s_S^2}  
 \leq \prod_{j=1}^m \E{s_j^{2p_j}} 
 \leq \prod_{j=1}^m \left(\E{s_j^{2d}}\right)^{p_j/d} 
  \leq \left( \max_j \E{s_j^{2d}}\right)^{\abs{S}/d} 
  \leq M_{2d}^{\abs{S}/d}. 
\end{align*}

Thus
\begin{align*}
\Vr{\partial_S \bar{\phi}(t)} \leq \frac{M_{2d}^{\abs{S}/d}}{N} \le \frac{M_{2d}}{N}.
\end{align*}

Chebyshev's inequality (which remains unchanged for complex-valued
r.v.s) for $\eps' > 0$ yields
\begin{align} \label{eqn:app_Chebyshev}
\prob{\abs{\partial_S \bar{\phi}(t) - \partial_S \phi(t)} \geq \epsilon'} \leq 
\frac{M_{2d}}{\epsilon'^2 N}.
\end{align}
We will choose a value of $\eps'$ shortly.

Now we bound the difference between the corresponding summands in the decompositions of $N_d(t)$ 
and $\bar{N}_d(t)$ as sums. 
Specifically, with probability at most $\frac{(d+1)M_{2d}}{\epsilon' N}$ 
(this comes from the union bound: we want the 
event in \eqref{eqn:app_Chebyshev} to hold for all $S_j$ for $j \in [d]$ as well as for $S = \emptyset$, 
corresponding to $\phi(t)$) we have  
\begin{align*}
  \abs{\left(\bar{\phi}(t)^d\prod_{j=1}^d \partial_{S_j}\phi(t)\right)-\left(\phi(t)^d\prod_{j=1}^d \partial_{S_j}\bar{\phi}(t)\right)}
  &\leq \abs{\prod_{j=1}^d \partial_{S_j}\phi(t)}
  \abs{\bar{\phi}(t)^d-\phi(t)^d}
  + \abs{\phi(t)^d} \abs{\prod_{j=1}^d \partial_{S_j}\phi(t)-\prod_{j=1}^d \partial_{S_j}\bar{\phi}(t)} \\
  &\leq M_d \abs{\bar{\phi}(t)^d-\phi(t)^d} + (M_d+\epsilon')^d - M_d^d \\
  &\leq \epsilon' d\:M_d + \epsilon' d (M_d+\epsilon')^{d-1} \\
  &\leq 2 \epsilon' d (M_d+1+\epsilon')^{d-1},
\end{align*}
where the second inequality used \eqref{eq:partial_ubd},
Lemma~\ref{lem:prod_difference} and $\abs{\phi(t)} \leq 1$ and
$\abs{\bar{\phi}(t)} \leq 1$.

Now using the expression for $N_d(t)$ as a sum given in
Claim~\ref{claim:multivariate_errorestimate}, with probability at most
$\frac{2^d M_{2d}}{\epsilon'^2 N}$ (the factor $2^d$ again comes from the union
bound: we want the event in \eqref{eqn:app_Chebyshev} to hold for all
(multi-) subsets of $\{i_1, \ldots, i_d\}$) we have
\begin{align}
\abs{\partial_S \bar{\phi}(t) - \partial_S \phi(t)} 
&=
\frac{\abs{\bar{\phi}(t)^dN_d(t)-\phi(t)\bar{N}_d(t)}}{\abs{\phi(t)^d\bar{\phi}(t)^d}} \\
&\leq \frac{2^d \epsilon' d (M_d+1+\epsilon')^{d-1} (d-1)! }{\abs{\phi(t)^d\bar{\phi}(t)^d}} \nonumber \\
&\leq \frac{2^d \epsilon' d (M_d+1+\epsilon')^{d-1} (d-1)!}{(3/4)^d(3/4-\epsilon')^d}, \label{eqn:sampling_one_term_tensor}
\\
& \le \eps, \notag
\end{align}
where the last inequality used Lemma~\ref{lem:phi_nonvanishing} and
\begin{align*}
  \eps' = \left[ \frac{2^d d (M_d+1+\epsilon')^{d-1}
      (d-1)!}{(3/4)^d(3/4-\epsilon')^d} \right]^{-1} \eps.
\end{align*}

Now if we want \eqref{eqn:sampling_one_term_tensor} to hold for all
multisets $S$ of size $d$, then the union bound needs to extended to all
such multisets (of which there are ${m+d-1 \choose d}$) giving that
error probability at most 
\begin{align*}
  \binom{m+d-1}{d} \frac{2^d M_{2d}}{\eps'^2 N} = \binom{m+d-1}{d}
  \frac{2^d M_{2d}}{\eps^2 N} \left[ \frac{2^d d (M_d+1+\epsilon')^{d-1}
      (d-1)!}{(3/4)^d(3/4-\epsilon')^d} \right]^2,
\end{align*}
as desired.
\end{proof}

\begin{lemma}[Sample
  Complexity] \label{lem:underdetermined_sample_approx} Let $x = As$
  be an ICA model with $A \in \R^{n \times m}, x \in \R^n, s \in
  \R^m$ and $d$ an even positive integer. Let $M_2, M_{2d} > 0$ be such that
  $\E{s_i^2} \leq M_2$ and $\E{\abs{s_i}^{2d}} \leq M_{2d}$. Let $v \in \R^n$ satisfy $\norm{v}_2 \leq
  \frac{1}{2\norm{A}_2\sqrt{2M_2}}$. Let $T_v = D_u^d \psi_x(v)$ be the $d$'th derivative
  tensor of $\psi_x(u) = \log \E{e^{iu^Tx}}$ at $v$. 
And let $\bar{T}_v = D_u^d \bar{\psi}_x(v)$ be its naive 
  estimate using $N$ independent samples of $x$ where 
\begin{align*}
  N \ge \binom{m+d-1}{d} \frac{1}{m^{d/2}\sigma_1(A)^d} \frac{2^d
    M_{2d}}{\eps^2 \delta} \left[ \frac{2^d d (M_d+2)^{d-1}
      (d-1)!}{(3/4)^d(1/2)^d} \right]^2.
\end{align*}
Then with probability at least $1-\delta$ we have
\begin{align*}
\norm{T_v - \bar{T}_v}_F \leq \eps.
\end{align*}
\end{lemma}
\begin{proof}
In the following all tensors are flattened into matrices. 
Let $x^j = A s^j$, $j \in [N]$ be i.i.d. samples. Letting $t = A^Tv$
we have $T_v = D_u^d \psi_x(u) = A^{\otimes d/2} D_t^d\psi_s(t) (A^{\otimes d/2})^T$, and
$\bar{T}_v = D_u^d \bar{\psi}_x(u) = A^{\otimes d/2} D_t^d\bar{\psi}_s(t) (A^{\otimes d/2})^T$.
(Note that we could also have written 
$D_u^d \psi_x(u) = A^{\odot d/2} \diag{\partial^d_{t_j}\psi_s(t)} (A^{\odot d/2})^T$ because the components of $s$ are
indpendent, however the corresponding empirical equation 
$D_u^d \bar{\psi}_x(u) = A^{\odot d/2} \diag{\partial^d_{t_j}\bar{\psi}_s(t)} (A^{\odot d/2})^T$ need not be 
true.)

Hence

\begin{align*}
\norm{\bar{T}_v-T_v}_F &= \norm{A^{\otimes d/2} D_t^d\bar{\psi}_s(t) (A^{\otimes d/2})^T - 
A^{\otimes d/2} D_t^d\psi_s(t) (A^{\otimes d/2})^T}_F \\
&= \norm{A^{\otimes d/2} (D_t^d\bar{\psi}_s(t)-D_t^d\psi_s(t)) (A^{\otimes d/2})^T}_F \\
& \leq \sigma_1(A^{\otimes d/2})^2 \norm{D_t^d\bar{\psi}_s(t)-D_t^d\psi_s(t)}_F \\
& =  \sigma_1(A)^d \norm{D_t^d\bar{\psi}_s(t)-D_t^d\psi_s(t)}_F \\
& \leq \eps,
\end{align*}

where the last inequality holds with probability at least $1-\delta$
by Lemma~\ref{lem:underdetermined_sample_approx_entry} which is applicable 
because $\norm{A^Tv}_2 \leq \norm{A}_2 \norm{v}_2 \leq
\frac{1}{2\sqrt{2M_2}}$.
\end{proof}


\subsection{Main theorem and proof}\label{sec:proof}
We are now ready to formally state and prove the main theorem. To get a success probability of $3/4$, we choose $q$ so that $20m^2q < 1/4$.
\begin{theorem}[Underdetermined ICA]\label{thm:main}
  Let $x \in \R^n$ be generated by an underdetermined ICA model $x=As$
  with $A \in \R^{n \times m}$ where $n \le m$. Suppose that the
  following data and conditions are given:
  \begin{enumerate}
  \item $d \in 2 \N$ such that $\sigma_m \left( A^{\odot d/2} \right) > 0$.
  \item $k$ such that for each $i$ there exists $k_i$, where $d < k_i < k$ such
    that $\abs{ \cum_{k_i}(s_i)} \ge \Delta_0$.
  \item Constants $M_2, M_d, M_{k}$ such that for each $s_i$ the following bounds hold
\begin{align*}
  \E{ s_i} = 0, \quad \E{ s_i^2} \le M_2, \quad \E{ s_i^d} \le M_d ,
  \quad \E{\abs{s_i}^{k_i + 1}} \le M_{k}, \quad \Delta_0 \leq M_d, \quad \E{\abs{s_i}^{2d}} \le M_{2d}.
\end{align*}

\item $0< \sigma \leq \min(1, \sigma_0, \frac{1}{4m} \sqrt{\frac{1}{6M_2 \ln(2/q)}})$ where 
\begin{align*} \sigma_0 = \Delta_0 \frac{k-d+1}{k!} \left(\frac{3}{8}\right)^k \frac{1}{M_k} 
\left(\frac{2 \sigma_m(A^{\odot d/2}) q \sqrt{2\pi}}{4(k-d)\sqrt{d}} \right)^{k-d} \left(\frac{1}{\sqrt{2\log 1/q}}\right)^{k-d}.
\end{align*} 
\end{enumerate} 
Then, with probability at least $1 - 20 m^2 q$, algorithm
\textbf{Underdetermined ICA$(\sigma)$} will return a matrix
$\tilde{B}$ such that there exist signs $\alpha_j \in \set{-1, 1}$ and
permutation $\pi : [m] \to [m]$ such that
\begin{align*}
\norm{B_j - \alpha_j \tilde{B}_{\pi(j)}} \leq \eps, 
\end{align*}
using $N$ samples where
\begin{align*} 
  N \ge \left(\frac{k m (M_d+2)}{\sigma q \sigma_m(A^{\odot d/2})}\right)^{ck} 
            \frac{\kappa(A^{\odot d/2})^6 M_{2d}}{\Delta_0^6 \epsilon^2},
\end{align*}
for some absolute constant $c$. The running time of the algorithm is $\text{poly}(N)$. 
\end{theorem} 

\begin{proof}
The proof involves putting together of various results we have proven. 
 We take $N$ independent samples of $x$ and form the
  flattened $d$th derivative tensors $\bar{M}_u, \bar{M}_{v}$ of
  $\psi(u)$ evaluated at $u$ and $v$ which are sampled from $N(0,
  \sigma_0^2)$. Recall that these are the matrices constructed by 
\textbf{Underdetermined ICA}$(\sigma_0)$) which
then invokes \textbf{Diagonalize}$(\cdot)$ which computes eigendecomposition of 
$\bar{M}_u\bar{M}_{v}^{-1}$. We will denote by $M_u, M_v$ the corresponding matrices 
without any sampling errors.
We will first use the
result about eigenvalue spacing Theorem~\ref{thm:anti-concentration} to get a bound on the spacings
of the eigenvalues of the matrix $M_u M_v^{-1}$, where
$u, v \sim N(0,\sigma_0^2 I_n)$ are random vectors. Next, we determine upper and lower bounds $K_U$ and $K_L$ on the 
eigenvalues of $M_u$ and $M_v$. We can then apply Theorem~\ref{thm:robustdecomposition} to show that 
if we have sufficiently good approximation of $M_u$ and $M_v$ then we will get a good reconstruction
of matrix $A$. Finally, we use Lemma~\ref{lem:underdetermined_sample_approx} to determine the number 
of samples needed to get the required approximation.

  \paragraph{Step 1.} First, we apply
  Theorem~\ref{thm:anti-concentration}. Note that our choice of
  $\sigma_0$ satisfies the constraints on $\sigma$ in
  Theorem~\ref{thm:anti-concentration}; thus except with probability
  $q$, we have
\begin{align}\label{eqn:Omega_0}
  \abs{ \frac{g_b( A_b^T u) }{ g_b( A_b^T v)} - \frac{g_a( A_a^T u) }{
      g_a( A_a^T v)} } \ge \Omega_0 := 
\frac{\Delta_0}{M_d}\left(\frac{3}{8}\right)^d \frac{1}{(d-1)!(k-d)!} 
\left(\frac{\sigma_0 q L \sqrt{2\pi}}{4(k-d)}\right)^{k-d},
\end{align}
for all pairs $a,b \in [m]$. Here $L$ as defined in Theorem~\ref{thm:anti-concentration}
is given by $2 \sigma_m( A^{\odot d/2})/\sqrt{d}$ by Lemma~\ref{lemma:identifiable}.

\paragraph{Step 2.}  
Next, we will show that $u$ and $v$ concentrate in norm. To do so, we
will apply the following concentration inequality for sub-exponential
random variables (this is standard in the proof of the
Johnson-Lindenstrauss Lemma, see \cite{vempala06,Dasgupta03} or
alternatively \cite{VershyninSingular} for a more general formulation).
\begin{lemma}[\cite{vempala06,Dasgupta03}]
  Let $z_i \sim N(0, 1)$ be i.i.d., then
  \begin{align*}
    \prob{ \sum_{i=1}^n z_i^2 \ge \beta n } \le e^{ \frac{n}{2} ( 1 -
      \beta + \log(\beta))}.
  \end{align*}
\end{lemma}
For $\beta \ge 6$, the bound only improves as $n$ increases. Thus, we have
the simplified bound 
\begin{align*}
  \prob{ \sum_{i=1}^n z_i^2 \ge \beta n} \le e^{-\frac{n \beta}{12}}.
\end{align*}
In particular, union bounding over both $u,v\in\R^n$, we have
\begin{align*}
  \prob{ \norm{u},\norm{v} \ge \frac{1}{2 \norm{A}_F \sqrt{2 M_2} }} & =
  \prob{\norm{u}^2,\norm{v}^2 \ge \left(
      \frac{1}{2 \norm{A}_F \sqrt{2 M_2} }\right)^2} \\
  & \le 2 \exp \left( -\frac{1}{12\sigma_0^2 } \left( \frac{1}{ 2 \norm{A}_F
      \sqrt{2 M_2}}\right)^2 \right),
\end{align*}
where in the second line, we took $\beta n = \frac{1}{12\sigma_0^2}
\left( \frac{1}{2 \norm{A}_F \sqrt{2M_2}} \right)^2$. Using $\norm{A}_F \le m$, and
our choice of $\sigma_0$ which gives
$\sigma_0 \le \frac{1}{4m} \sqrt{\frac{1}{6M_2 \ln(2/q)}}$ 
we obtain 
\begin{align*}
  \prob{ \norm{u},\norm{v} \ge \frac{1}{2 \norm{A}_F \sqrt{2 M_2} }} \le q.
\end{align*}
Thus except with probability $q$, norms $\norm{u}, \norm{v}$
satisfy the hypotheses of
Lemma~\ref{lem:underdetermined_sample_approx}.

\paragraph{Step 3.}
Now we determine the values of parameters $K_U$ and $K_L$ used in Theorem~\ref{thm:robustdecomposition}.
A bound for $K_U$ can be 
obtained from Lemma~\ref{lemma:errorestimate} and
Lemma~\ref{lem:phi_nonvanishing} to $\psi_s(t) = \psi_s(A^Tu)$. 
The latter lemma being applicable
because $\norm{A^Tu} \leq \norm{A}_F\norm{u} \leq \frac{1}{2 \sqrt{2 M_2}}$ and 
$\norm{A^Tv} \leq \norm{A}_F\norm{v} \leq \frac{1}{2 \sqrt{2 M_2}}$ from
Step 2:

\begin{align*}
K_U = \frac{(d-1)!2^{d-1} M_d}{(3/4)^d}. 
\end{align*}

For $K_L$, by Cor.~\ref{cor:K_L} we can set
\begin{align*}
  K_L = \frac{\Delta_0}{2(k-d)!} \left(\frac{\sigma_0 q \sqrt{2\pi}}{4(k-d)}\right)^{k-d},
\end{align*}
which holds with probability at least $1-6mq$. 

\paragraph{Step 4.}
We now fix $K_1$ which is the upper bound on 
$\norm{M_u - \bar{M}_u}_F$ and $\norm{M_v - \bar{M}_v}_F$ 
needed in Theorem~\ref{thm:robustdecomposition} 
(the role of these two quantities is played by $\norm{R_\mu}_F$ and $\norm{R_\lambda}_F$ in that theorem).
Our assumption $\Delta_0 \leq M_d$ gives that $\Omega_0 \leq 1$ by \eqref{eqn:Omega_0}.
And hence the bound required in Theorem~\ref{thm:robustdecomposition} becomes
\begin{align} \label{eqn:K_1}
K_1 = \frac{\epsilon K_L^2 \sigma_m(B)^3}{2^{11} \kappa(B)^3 K_U m^2} \Omega_0,
\end{align}
where $B = A^{\odot d/2}$. 

For this $K_1$ by Theorem~\ref{thm:robustdecomposition}  the algorithm
recovers $\tilde{B}$ with the property that there are signs $\alpha_j
\in \set{-1, 1}$ and permutation $[m] \to [m]$ such that

\begin{align*}
\norm{B_j - \alpha_j \tilde{B}_{\pi(j)}} \leq \eps.
\end{align*}

\paragraph{Step 5.}
It remains to determine the number of samples needed to achieve $\norm{M_u - \bar{M_u}}_F \leq K_1$ and
$\norm{M_v - \bar{M_v}}_F \leq K_1$. 

By Step 2 above, we satisfy the hypotheses of Lemma~\ref{lem:underdetermined_sample_approx}. 
Hence by that lemma, for $N$ at least the quantity below
\begin{align*}
\binom{m+d-1}{d} \frac{1}{m^{d/2}\sigma_1(A)^d} \frac{2^d M_{2d}}{K_1^2 q}
    \left( \frac{16^d d (M_d+2)^{d-1} (d-1)!}{3^d} \right)^2 
\leq 11^{2d} m^{d/2} d^{2(d+1)} M_{2d} (M_d+2)^{2(d-1)} \frac{1}{K_1^2 q}
\end{align*}
we have 
\begin{align*}
  \norm{M_u - \bar{M_u}}_F \le K_1,
\end{align*}
except with probability $q$, and similarly for $\norm{M_u - \bar{M_u}}_F$.
Subtistuting the value of $K_1$ from \eqref{eqn:K_1} and in turn of $K_U$, $K_L$ and 
$\Omega_0$ above 
and simplifying (we omit the straightforward but tedious details) gives that it suffices
to take
\begin{align*}
N \geq 
\frac{2^{4k+6d+26}}{3^{2d}} d^{6d+2} (k-d)^{2(k-d)} m^{d/2+4} \frac{M_d^2 M_{2d} (M_d+2)^{2d}}{\Delta_0^6} 
\frac{\kappa(B)^6}{\sigma_m(B)^{k-d+6}} \frac{1}{\sigma^{5(k-d)} q^{5(k-d)+1}} \frac{1}{\epsilon^2}. 
\end{align*}

Accounting for the probability of all possible bad events enumerated in the proof via 
the union bound we see that with probability at least
$1 - q -3{m \choose 2} q - 6mq -q > 1 - 20m^2q$ no bad events happen. The running time computation
involves empirical estimates of derivate tensors and SVD and eigenvalue computations; we skip the 
routine check that the running time is $\text{poly}(N)$.
\end{proof}

\subsection{Underdetermined ICA with unknown Gaussian noise}\label{sec:underdetermined-noise}
Theorem~\ref{thm:main} just proved is the detailed version of Theorem~\ref{thm:UICA_noisy}
without Gaussian noise. 
In this section we indicate how to extend this proof when there is Gaussian noise thus 
proving Theorem~\ref{thm:UICA_noisy} in full.  
Our algorithm for the noiseless case applies essentially unaltered to the case when the input has 
unknown Gaussian noise if $d > 2$. We comment on the case $d=2$ at the end of this section.
More precisely, the ICA model now is 
\begin{align*}
x' = x + \eta = As + \eta,
\end{align*}
where $\eta \sim N(0, \Sigma)$ where $\sigma \in \R^{n \times n}$ is unknown covariance matrix and $\eta$ 
is independent of $s$. 
Using the independence of $\eta$ and $s$  and the standard expression for the second characteristic of 
the Gaussian we have
\begin{align} \label{eqn:characteristic_Gaussian_noise}
\psi_{x'}(u) = \E{e^{iu^T x'}} = \E{e^{iu^T x + iu^T \eta}} = \psi_x(u) + \psi_\eta(u) = \psi_x(u) - 
\frac{1}{2} u^T\Sigma u.
\end{align}

Our algorithm works with (estimate of) the $d$th derivative tensor of $\psi_{x'}(u)$. For $d > 2$, 
we have 
$D_u^d \psi_{x'}(u) = D_u^d \psi_{x}(u)$ as in \eqref{eqn:characteristic_Gaussian_noise} 
the component of the second characteristic function 
corresponding to the Guassian noise is quadratic and vanishes for third and higher derivatives. Therefore, but for the estimation errors, the Gaussian noise makes
no difference and the algorithm would still recover $A$ as before. Since the algorithm works only with 
estimates of these derivatives, we have to account for how much our estimate of $D_u^d \psi_{x}(u)$ changes 
due to the extra additive term involving the derivative of the estimate of the second characteristic of the 
Guassian.

If $\Sigma$ is such that the moments of the Gaussian noise also satisfy the conditions we imposed on 
the moments of the $s_i$ in the Theorem~\ref{thm:main}, then we can complete the proof with little
extra work. The only thing that changes in the proof of the main theorem is that instead of getting 
the bound
$\norm{M_u - \bar{M}_u} \leq \eps'$ we get the bound $\norm{M_u - \bar{M}_u} \leq 2\eps'$. If we increase
the number of samples by a factor of 4 then this bound becomes $\norm{M_u - \bar{M}_u} \leq \eps'$, and
so the proof can be completed without any other change.

\emph{The $d=2$ case.} When $d=2$, the second derivative of the component of the second characteristic 
function corresponding to the noise in \eqref{eqn:characteristic_Gaussian_noise} is a constant matrix
independent of $u$. Thus if we take derivatives at two different points and subtract them, then this 
constant matrix disappears. This is analogous to the algorithm we gave for fully-determined ICA with 
noise in Sec.~\ref{subsec:noise}. The error analysis can still proceed along the above lines; we omit
the details. 

\section{Mixtures of spherical Gaussians}\label{sec:gaussians}
Here we apply Fourier PCA to the classical problem of learning a
mixture of Gaussians, assuming each Gaussian is spherical.  More
precisely, we get samples $x+\eta$, where $x$ is from a distribution
that is a mixture of $k$ unknown Gaussians, with $i$'th component
having mixing weight $w_i$ and distribution $F_i =N(\mu_i,\sigma_i^2
I)$; the noise $\eta$ is drawn from $N(\mu_\eta, \Sigma_\eta)$ and is
not necessarily spherical. The problem is to estimate the unknown
parameters $w_i, \mu_i, \sigma_i$. Our method parallels the Fourier
PCA approach to ICA, but here, because the structure of the problem is
additive (rather than multiplicative as in ICA), we can directly use
the matrix $D^2 \phi$ rather than $D^2 \psi = D^2 \log( \phi)$. It is
easy to show that $D^2 \phi = \Sigma_u$ in the description of our
algorithm.

For any integrable function $f:\C^n \to \C$, we observe that for a
mixture $F = \sum_{i=1}^k w_i F_i$:
\begin{align*}
\EE{F}{(f(x+\eta))} = \sum_{i=1}^k w_i \EE{F_i}{f(x+\eta)}.
\end{align*}
We assume, without loss of generality, that the full mixture is
centered at zero, so that:
\begin{align*}
  \sum_{i=1}^k w_i \mu_i = 0
\end{align*}

\begin{center}
\fbox{\parbox{\textwidth}{
\begin{minipage}{5in}
{\bf Fourier PCA for Mixtures}
\begin{enumerate}
\item Pick $u$ independently from $N(0, I_n)$. 
\item Compute $M=\E{xx^T}$, let $V$ be the span of its top $k-1$ eigenvectors and $\bar{\sigma}^2$ be its $k$'th eigenvalue and
$v$ be its $k$'th eigenvector. Let $z$ be a vector orthogonal to $V$ and to $u$.
\item 
Compute 
\begin{align*}
\Sigma_u  &= \E{ xx^T e^{iu^Tx}}, \quad  \bar{\sigma}_u^2  = \E{(z^Tx)^2 e^{iu^Tx}},\\ 
\gamma_u &= \frac{1}{(u^Tv)^2}\left(-\E{(v^Tx)^2e^{iu^Tx}} + \bar{\sigma}_u^2\right), \quad
  \tilde{u} = \E{x(z^Tx)^2 e^{iu^Tx}}.
\end{align*} 
\item  Compute the matrices
\begin{align*}
  M = \E{xx^T} - \sigma^2 I \mbox{ and } M_u = \Sigma_u -
  \tilde{\sigma}_u^2 I - i\tilde{u}u^T - iu\tilde{u}^T - \gamma_u uu^T.
\end{align*}
\item Run \textbf{Tensor Decomposition$(M_u,M)$} to obtain eigenvectors $\tilde{\mu}_j$ and eigenvalues $\lambda_j$ of $M_uM^{-1}$ (in their original coordinate representation).\\
\item Estimate mixing weights by finding $w \ge 0$ that minimizes $\|\sum_{j=1}^k \sqrt{w_j}\tilde{\mu}_j\|$ s.t. $\sum_{j=1}^k w_j = 1$.  Then estimate means and variances as  
\[
\mu_j = \frac{1}{\sqrt{w_j}} \tilde{\mu}_j,  \qquad e^{-\frac{1}{2} \sigma_j^2 \|u\|^2 + iu^T\mu_j} = \lambda_j.
\] 
\end{enumerate}
\end{minipage}
}}
\end{center}

\begin{lemma}\label{lemma:translation}
  For any $f:\C^n \to \C$, and $x \sim N(\mu, \Sigma)$ where $\Sigma$ is
  positive definite:
\begin{align*}
  \E{ f(x) e^{iu^Tx}} = e^{i u^T\mu - \frac{1}{2} u^T \Sigma u}
  \E{f(x+i \Sigma u)}.
\end{align*}
\end{lemma}
\begin{proof}
  The proof is via a standard completing the square argument; consider
  the exponent:
\begin{align*}
  & -\frac{1}{2} \left[ (x- \mu)^T \Sigma^{-1} ( x - \mu) \right] +
  iu^T x \\
  &= - \frac{1}{2} \left[ x^T \Sigma^{-1} x + \mu^T \Sigma^{-1} \mu
    - x^T \Sigma^{-1} \mu - \mu^T \Sigma^{-1} x \right] + i u^T x\\
  & = -\frac{1}{2} \left[ x^T \Sigma^{-1} x - x^T \Sigma^{-1} (\mu + i
    \Sigma u) - (\mu + i \Sigma u)^T \Sigma^{-1} x + (\mu + i \Sigma
    u)^T  \Sigma^{-1} ( \mu + i \Sigma u)\right] \\
  & \qquad + i u^T \mu+ \frac{1}{2} (\Sigma u)^T \Sigma^{-1} (\Sigma
  u) \\
  & = -\frac{1}{2} (x-(\mu + i \Sigma u))^T \Sigma^{-1} (x-(\mu + i
  \Sigma u)) + i u^T \mu - \frac{1}{2} u^T \Sigma u
\end{align*}
Thus:
\begin{align*}
  & \E{ f(x) e^{ iu^T x}} \\
  & = \frac{1}{\det{\Sigma}^{1/2} (2 \pi)^{n/2}} \int f(x) e^{ iu^Tx} e^{ 
      -\frac{1}{2} (x - \mu)^T \Sigma^{-1}
      (x - \mu) } dx \\
    & = \frac{1}{\det{\Sigma}^{1/2} (2 \pi)^{n/2}} \int f(x) e^{
    -\frac{1}{2} (x-(\mu + i \Sigma u))^T \Sigma^{-1} (x-(\mu +
      i \Sigma u)) } e^{ i u^T \mu - \frac{1}{2} u^T
      \Sigma u } dx
\end{align*}
Now with a change of variables $y = x -  i \Sigma u$, we
obtain:
\begin{align*}
  \E{ f(x) e^{ iu^T x}} &= \frac{1}{\det{\Sigma}^{1/2} (2 \pi)^{n/2}} e^{
   i u^T \mu - \frac{1}{2} u^T \Sigma u } \int f( y + i
  \Sigma u)) e^{ -\frac{1}{2} (y-\mu)^T \Sigma^{-1} (y-\mu)}
   dy \\
  & = e^{ i u^T \mu - \frac{1}{2} u^T \Sigma u } \E{ f(
    y + i \Sigma u))}
\end{align*}
\end{proof}
Note: technically we require that $\E{\abs{f(x)}} < \infty$ with
respect to a Gaussian measure so as to apply the dominated convergence
theorem, and an analytic extension of the Gaussian integral to complex
numbers, but these arguments are standard and we omit them (see for
example \cite{lang1998complex}).

\begin{lemma} \label{lemma:gaussian-expectations}
  Let $x \in \R^n $ be drawn from a mixture of $k$ spherical Gaussians in
  $\R^n$, $u, z \in \R^n$ as in the algorithm. Let $\hat{w}_j = w_j e^{iu^T\mu_j - \frac{1}{2}\sigma_j^2\|u\|^2}$. Then, 
\begin{align}
  \E{xx^T} &= \sum_{j=1}^k w_j \sigma_i^2 I + \sum_{j=1}^k w_j \mu_j
  \mu_j^T . \label{eqn:a1}\\
  \E{xx^Te^{iu^Tx}} &= \sum_{j=1}^k \hat{w}_j \sigma_j^2 I   + \sum_{j=1}^k \hat{w}_j 
  (\mu_j +i\sigma_j^2 u)(\mu_j+i\sigma_j^2 u)^T.  \label{eqn:a2}\\
 \E{x(z^Tx)^2e^{iu^Tx}} &= \sum_{j=1}^k \hat{w}_j \sigma_j^2(\mu_i +
 i \sigma_j^2 u) \ \label{eqn:a5}. 
\end{align}
\end{lemma}
\begin{proof}
  These are obtained by direct calculation and expanding out $x_i \sim
  N(\mu_i, \sigma^2 I_n)$. 
 For \eqref{eqn:a1}:
  \begin{align*}
    \E{xx^T} & = \sum_{j=1}^k w_j \EE{F_j}{xx^T} \\
    & = \sum_{j=1}^k w_j \E{ (x-\mu_j + \mu_j)(x-\mu_j+\mu_j)^T} \\
    & = \sum_{j=1}^k w_j \left[ \sigma_j^2 I_n + \mu_j \mu_j^T \right]
  \end{align*}
 \eqref{eqn:a2} follows by applying Lemma
  \ref{lemma:translation} and the previous result:
  \begin{align*}
    \E{xx^T e^{i u^T x}} &= \sum_{i=1}^k w_i e^{i u^T \mu_i -
      \frac{1}{2}\sigma_i^2 \norm{u}^2} \EE{F_i}{ (x_i + i \sigma_i^2
      u) (x_i + i \sigma_i^2 u)^T}\\
    & = \sum_{i=1}^k \hat{w}_i \left[ \sigma_i^2 I_n + ( \mu_i+ i
      \sigma_i^2 u)( \mu_i + i \sigma_i^2 u)^T \right]
  \end{align*}
To see \eqref{eqn:a5}, we write (noting that $z$ is orthogonal to $u$ and to each $\mu_j$),
\begin{align*}
\E{x(z^Tx)^2 e^{iu^Tx}} &= \sum_{j=1}^k \hat{w}_j \E{(x+i\sigma_j^2 u)(z^T(x+i\sigma_j^2 u))^2}\\
&= \sum_{j=1}^k \hat{w}_j \E{(x-\mu_j + \mu_j+i\sigma_j^2 u)(z^T(x-\mu_j))^2}\\
&= \sum_{j=1}^k \hat{w}_j \left(\E{(x-\mu_j)(z^T(x-\mu_j))^2}+\E{(\mu_j+i\sigma_j^2 u)(z^T(x-\mu_j))^2}\right)\\
&= \sum_{j=1}^k \hat{w}_j \sigma_j^2 (\mu_j + i \sigma_j^2 u).
\end{align*}

\end{proof}

Instead of polynomial anti-concentration under a Gaussian measure, we
require only a simpler lemma concerning the anti-concentration of
complex exponentials:
\begin{lemma}[Complex exponential
  anti-concentration]\label{lemma:exponential-anti}
  Let $\mu_i, \mu_j \in \R^n $ satisfy $\norm{ \mu_i - \mu_j} > 0$,
  then for $u \sim N(0, \sigma^2 I_n)$ where $\norm{\mu_i - \mu_j}^2
  \sigma^2 \le 2 \pi^2$. Then:
  \begin{align*}
    \prob{ \abs{ e^{ i \mu_i^T u} - e^{ i \mu_j^T u}} \le \eps}
    \le \frac{16 \eps}{\norm{\mu_i -\mu_j} \sigma \sqrt{2 \pi}}.
  \end{align*}
\end{lemma}
\begin{proof}
  First, note that it suffices to show anti-concentration of the
  complex exponential around 1:
  \begin{align*}
    \abs{ e^{ i \mu_i^T u} - e^{ i \mu_j^T u}} = \abs{ e^{ i
      \mu_i^T u}} \abs{ 1- e^{ i (\mu_i - \mu_j)^T u}} = \abs{ 1-
      e^{ i (\mu_i - \mu_j)^T u}}
  \end{align*}
  The exponent $(\mu_i-\mu_j)^T u$ is of course a random variable $z
  \in \R$ distributed according to $N(0, \sigma^2
  \norm{\mu_i-\mu_j}^2)$.  From plane geometry, we know that: $
  \abs{ e^{iz} -1 } > \eps$ in case 
  \begin{align*}
    z  \notin \cup_{k \in \Z} [2 \pi k -2 \eps, 2 \pi k + 2 \eps]
  \end{align*}
  We can bound this probability as follows:
  \begin{align*}
    \prob{ z \notin \cup_{k \in \Z} [2 \pi k -2 \eps, 2 \pi k + 2
      \eps]} & \le 2 \sum_{k=0}^{\infty} \frac{4 \eps}{\norm{\mu_i -
        \mu_j}\sigma \sqrt{2 \pi}} e^{ -\frac{(2\pi k)^2 }{2
        \norm{\mu_i - \mu_j}^2\sigma^2} } \\
    & \le \frac{8 \eps}{\norm{\mu_i - \mu_j}\sigma \sqrt{2 \pi}} \sum_{k
      = 0}^{\infty} e^{ -\frac{2 \pi^2 k }{
        \norm{\mu_i - \mu_j}^2\sigma^2} } \\
    & = \frac{8 \eps}{\norm{\mu_i - \mu_j}\sigma \sqrt{2 \pi}}
    \frac{1}{1 - e^{ -\frac{2 \pi^2 }{ \norm{\mu_i -
            \mu_j}^2\sigma^2} }} \\
    & \le \frac{16 \eps}{\norm{\mu_i - \mu_j}\sigma \sqrt{2 \pi}}
  \end{align*}
  where the last line follows from the assumption $\norm{\mu_i - \mu_j}^2
  \sigma^2 \le 2 \pi^2$.
\end{proof}

We can now prove that the algorithm is correct with sufficiently many
samples. Using PCA we can find the span of the means $\{\mu_1, \ldots,
\mu_k\}$, as the span of the top $k-1$ right singular vectors of the
matrix whose rows are sample points
\cite{vempala2004spectral}. Projecting to this space, the mixture
remains a mixture of spherical Gaussians. We assume that the $\mu_i$
are linearly independent (as in recent work \cite{HsuK13} with higher
moments).

\begin{proof}[Proof of Theorem \ref{thm:lin-ind-mixtures}.]
  From Lemma \ref{lemma:gaussian-expectations}, we observe that for any unit
  vector $v$,
\begin{align*}
  \E{(v^Tx)^2} = v^T\E{xx^T}v = \sum_{i=1}^k w_i \sigma_i^2 +
  \sum_{i=1}^k w_i (\mu_i^T v)^2.
\end{align*}
Without loss of generality, we may assume that the overall mean is 0,
hence $0 = \sum_i w_i \mu_i$ is 0 and therefore the $\mu_i$ are
linearly dependent, and there exist a $v$ orthogonal to all the
$\mu_i$. For such a $v$, the variance is $\sigma^2 = \sum_{i=1}^k w_i
\sigma_i^2$ while for $v$ in the span, the variance is strictly
higher. Therefore the value $\sigma^2$ is simply the $k$'th eigenvalue
of $\E{xx^T}$ (assuming $x$ is centered at $0$).

Thus, in the algorithm we have estimated the matrices 
\begin{align*}
  M = \sum_{i=1}^k w_i \mu_i \mu_i^T = AA^T \mbox{ and } M_u =
  \sum_{i=1}^k w_i e^{-\frac{1}{2}\|u\|^2\sigma_i^2+i u^T\mu_i }\mu_i
  \mu_i^T = AD_uA^T.
\end{align*}
with $(D_u)_{ii} = e^{- \frac{1}{2}\|u\|^2\sigma_i^2 + i u^T\mu_i
}$.  Thus,
\begin{align*}
M_uM^{-1} = AD_uA^{-1}
\end{align*}
and its eigenvectors are the columns of $A$, assuming the entries of
$D_u$ are distinct.  We note that the columns of $A$ are precisely
$\tilde{\mu}_j = \sqrt{w_j}\mu_j$. The eigenvalue corresponding to the eigenvector
$\tilde{\mu}_j$ is the $j$'th diagonal entry of $D_u$.

To prove the algorithm's correctness, we will once again apply Theorem
\ref{thm:robustdecomposition} for robust tensor decomposition by
verifying its five conditions. Condition 1 holds by our assumption on
the means of the Gaussian mixtures. Conditions 3 holds by taking
sufficiently many samples (the overall sample and time complexity will
be linear in $n$ and polynomial in $k$), Conditions 2 and 4 hold by
applying \ref{lemma:exponential-anti}.
\end{proof}

We can apply our observations regarding Gaussian noise from Section
\ref{subsec:noise}. Namely, the covariance of the reweighted Gaussian
is shifted by $\Sigma_\eta$, the covariance of the unknown
noise. Thus, by considering $\Sigma_u$ and the standard covariance,
and taking their difference, the contribution of the noise is removed
and we are left with a matrix that can be diagonalized using $A$.

\section{Perturbation bounds}
In this section, we collect all the eigenvalue decomposition bounds
that we require in our proofs. The generalized Weyl inequality we
derive in Theorem \ref{thm:weyl} surprisingly seems to be unknown in
the literature. 

\subsection{SVD perturbations}
In this section, we present two standard perturbation bounds on
singular vectors. These bounds will help determine the accuracy needed
in estimating the matrix with samples.
\begin{lemma}\label{lemma:stability}
  Let $A \in \C^{n \times n}$ and suppose that $\sigma_i(A) -
  \sigma_{i+1}(A) \ge \eps$ for all $i$. Let $E \in \C^{n \times n}$
  be a matrix where where $\norm{E}_2 \le \delta$. Denote by $v_i$ the
  right singular vectors of $A$ and $\hat{v}_i$ the right singular
  vectors of $A+E$, then:
  \begin{align*}
    \norm{v_i - \hat{v}_i} \le \frac{\sqrt{2} \delta}{\eps - \delta}
  \end{align*}
\end{lemma}
\begin{proof}
We first write:
\begin{align*}
  \norm{v_i-\hat{v}_i}^2 = \ip{v_i - \hat{v}_i}{v_i - \hat{v}_i} = 2 (1-
   \ip{v_i}{\hat{v}_i}) = 2( 1 - \cos (\theta)) \le 2( 1-
   \cos(\theta)^2) = 2 \sin( \theta)^2
\end{align*}
Next, we apply the following form of Wedin's Theorem from \cite{stewart1990matrix} where notions
such as the canonical angles etc. used in the statement below are also explained. 
  \begin{theorem} \label{thm:Wedin}
    Let $A, E \in \C^{m \times n}$ be complex matrices with $m \geq n$. Let $A$ have
singular value decomposition
\begin{align*}
A = [ U_1 U_2 U_3] \left( \begin{array}{cc} 
    \Sigma_1 & 0 \\
    0 & \Sigma_2 \\
    0 & 0 \\
  \end{array} \right) [ V_1^\ast V_2^\ast]
\end{align*}
and similarly for $\tilde{A}=A+E$ (with conformal decomposition using $\tilde{U}_1,
\tilde{\Sigma}_1$ etc).  Suppose there are numbers $\alpha, \beta >
0$ such that
\begin{enumerate}
  \item $\min \sigma( \tilde{\Sigma}_1) \ge \alpha + \beta$,
  \item $\max \sigma( \Sigma_2 ) \le \alpha$.
\end{enumerate}
Then
\begin{align*}
  \norm{ \sin (\Phi)}_2 , \norm{ \sin( \Theta)}_2 \le
 \frac{\norm{E}_2}{\beta}
\end{align*}
where $\Phi$ is the(diagonal) matrix of canonical angles between the ranges of $U_1$ and
$\tilde{U}_1$ and $\Theta$ denotes the matrix of canonical angles between the 
ranges of $U_2$ and $\tilde{U}_2$.
\end{theorem}
We also require the following form of Weyl's Inequality (see \cite{stewart1990matrix}):
\begin{lemma} \label{lem:Weyl_singular_values}
  Let $A, E \in \C^{m \times n}$, then
  \begin{align*}
    \abs{\sigma_i (A+E) - \sigma_i(A)} \le \sigma_1(E)
  \end{align*}
\end{lemma}
By Weyl's inequality, we know that $\abs{ \sigma(\Sigma_1) - \sigma(
  \tilde{\Sigma}_2)} \ge \eps - \delta$. Similarly for the smallest
singular value. By Wedin's theorem, we pick the partition $\Sigma_1$
to be the top $i$ singular values, with $\Sigma_2$ the remaining
ones. Thus, taking $\alpha = \sigma_{i+1} (A)$ and $\beta = \eps -
\delta$, we have 
\begin{align*}
  \abs{\sin (\theta)} \le \norm{\sin( \Phi)}_2 \le
  \frac{\delta}{\eps - \delta}
\end{align*}
as required.
\end{proof}


\subsection{Perturbations of complex diagonalizable
  matrices}\label{subsec:perturbation}
The main technical issue in giving a robust version of our algorithms
is that the stability of eigenvectors of general matrices is more
complicated than for Hermitian or normal matrices where the
$\sin(\theta)$ theorem of Davis and Kahan \cite{davis1970rotation}
describes the whole situation. 
Roughly speaking, the difficulty lies
in the fact that for a general matrix, the eigenvalue decomposition is
given by $A = PDP^{-1}$. Upon adding a perturbation $E$, it is not clear
\emph{a priori} that $A+E$ has a full set of eigenvectors---that is
to say, $A+E$ may no longer be diagonalizable. 
The goal of this
section is to establish that for a general matrix with well-spaced
eigenvalues, sufficiently small perturbations do not affect the
diagonalizability. 
We use Bauer-Fike theorem via a homotopy argument typically used in proving
strong versions of the Gershgorin Circle Theorem
\cite{wilkinson1965algebraic}.
\begin{theorem}[Bauer-Fike \cite{bauer1960norms}]\label{thm:bauer}
  Let $A \in \C^{n \times n}$ be a diagonalizable matrix such that $A
  = X \diag{\lambda_i} X^{-1}$. Then for any eigenvalue $\mu$ of $A+E \in \C^{n \times n}$ we have
  \begin{align*}
    \min_i \abs{ \lambda_i(A) - \mu} \le \kappa(X) \norm{E}_2.
  \end{align*}
 
\end{theorem}
Using this, we prove a weak version of Weyl's theorem for
diagonalizable matrices whose eigenvalues are well-spaced. We consider
this a spectral norm version of the strong Gershgorin Circle theorem
(which uses row-wise $L^1$ norms).
\begin{lemma}[Generalized Weyl inequality]\label{thm:weyl}
  Let $A \in \C^{n \times n}$ be a diagonalizable matrix such that $A
  = X \diag{\lambda_i} X^{-1}$. Let $E \in \C^{n \times n}$ be a
  matrix such that $\abs{ \lambda_i(A) - \lambda_j(A)} \ge 3 \kappa
  (X) \norm{E}_2$ for all $i \neq j$. Then there exists a permutation $\pi:
  [n] \to [n]$ such that
  \begin{align*}
    \abs{ \lambda_i (A+E) - \lambda_{\pi (i)} (A)} \le \kappa (X) \norm{E}_2.
  \end{align*}
\end{lemma}
\begin{proof}
  Consider the matrix $M(t) = A + t E$ for $t \in [0,1]$. By the
  Bauer-Fike theorem, every eigenvalue $\hat{\lambda}(t)$ of $M(t)$ is
  contained in $\B(\lambda_i, t \kappa(X) \norm{E}_2 )$ for some $i$ (for $\lambda\in \C$, $t \in \R$ we use
$\B(\lambda, t)$ to denote the ball in $\C$ of radius $t$ with center at $\lambda$). In
  particular, when $t=0$ we know that $\hat{\lambda}(0) = \lambda_i
  \in \B(\lambda_i, 0)$. 

  As we increase $t$, $\hat{\lambda}(t)$ is a continuous function of
  $t$, thus it traces a connected curve in $\C$. Suppose that
  $\hat{\lambda}(1) \in \B(\lambda_j, \kappa(X) \norm{E}_2)$ for some $j
  \ne i$, then for some $t^\ast$, we must have $\hat{\lambda}(t^\ast)
  \notin \bigcup_j \B(\lambda_i, \kappa(X) \norm{E}_2)$ as these balls
  are disjoint. This contradicts the Bauer-Fike theorem. Hence we must
  have $\hat{\lambda}(1) \in \B(\lambda_i, \kappa(X) \norm{E}_2)$ as
  desired. 
\end{proof}



The following is a sufficient condition for the diagonalizability of a
matrix. The result is well-known (Exercise V.8.1 in
\cite{lang2005undergraduate} for example). 
\begin{lemma}\label{lemma:diagonalizable}
  Let $A: V \to V$ be a linear operator over a finite dimensional
  vector space of dimension $n$. Suppose that all the eigenvalues of
  $A$ are distinct, i.e., $\lambda_i \ne \lambda_j$ for all pairs
  $i,j$. Then $A$ has $n$ linearly independent eigenvectors.
\end{lemma}



We require the following generalisation of the Davis-Kahan $\sin(
\theta)$ theorem \cite{davis1970rotation} for general diagonalizable
matrices due to Eisenstat and Ipsen \cite{eisenstat1998relative}:
\begin{theorem}[Generalized $\sin (\theta)$ Theorem]\label{thm:eisenstat}
  Let $A, A+E \in \C^{n \times n}$ be diagonalizable matrices. Let
  $\hat{\gamma}$ be an eigenvalue of $A+E$ with associated
  eigenvector $\hat{x}$. 
Let 
\begin{align*}
A = [X_1, X_2] \left( \begin{array}{cc} 
    \Gamma_1 & 0 \\
    0 & \Gamma_2 \\
  \end{array} \right) [ X_1, X_2]^{-1},
\end{align*}
be an eigendecomposition of $A$. Here $\Gamma_1$ consists of eigenvalues of $A$ closest to
$\hat{\gamma}$, i.e. $\norm{\Gamma_1 - \hat{\gamma}I}_2 = \min_i
    \abs{\gamma_i - \hat{\gamma}}$, with associated matrix of eigenvectors $X_1$. And $\Gamma_2$
contains the remaining eigenvalues and the associated eigenvectors are in $X_2$. 
Also let $[X_1, X_2]^{-1} =: \left( \begin{array}{c} Z_1^{\ast} \\ Z_2^{\ast}
       \end{array} \right)$.

  Then the angle between $\hat{x}$ and the subspace spanned by the eigenvectors
  associated with $\Gamma_1$ is given by
  \begin{align*}
    \sin( \theta) \le \kappa(Z_2) \frac{\norm{ (A - \hat{\gamma}I)\hat{x}}_2}{\min_i \abs{
        (\Gamma_2)_{ii} - \hat{\gamma}}}.
  \end{align*}
\end{theorem}

\section{General position and average case analysis} \label{sec:general_position}
Our necessary condition for identifiability is satisfied almost surely by randomly chosen vectors for a fairly 
general class of distributions. For simplicity we restrict ourselves to the case of $d=2$ and Gaussian distribution
in the following theorem; the proof of a more general statement would be similar.

\begin{theorem} \label{thm:genericity}
Let $v_1, \ldots, v_m \in \R^n$ be standard Gaussian i.i.d. random vectors, with $m \leq {n+1 \choose 2}$. Then
$v_1^{\otimes 2}, \ldots, v_m^{\otimes 2}$ are linearly independent almost surely.  
\end{theorem}
\begin{proof}[Proof Sketch]
Let's take $m = {n+1 \choose 2}$ without loss of generality. Consider vextors $w_1, \ldots, w_m$, where $w_i$ is
obtained from $v_i^{\otimes 2}$ by removing duplicate components; e.g., for $v_1 \in \R^2$, 
we have $v_1^{\oplus 2} = (v_1(1)^2, v_1(2)^2, v_1(1)v_1(2), v_2(1)v_1(2))$ and 
$w_1 = (v_1(1)^2, v_1(2)^2, v_2(1)v_1(2))$. Thus $v_i \in \R^{{n+1 \choose 2}}$. Now consider the determinant
of the ${n+1 \choose 2}\times{n+1 \choose 2}$ matrix with the $w_i$ as columns. As a formal multivariate polynomial with the
components of the $v_i$ as variables, this determinant is not identically $0$. This is because, for example, it
can be checked that
the monomial $w_1(1)^2\ldots w_n(n)^2 w_{n+1}(\rho(n+1)) \ldots w_{m}(\rho(m))$ occurs precisely once in the expansion of 
the determinant as a sum of monomials (here $\rho:\set{n+1, \ldots, m} \to {[n] \choose 2}$ is an arbitrary bijection).  The proof can now be completed along the lines of the well-known Schwartz--Zippel lemma. 
\end{proof}

We now show that the condition number of the Khatri--Rao power of a random matrix behaves well in certain situations.
For simplicity we will 
deal with the case where the entries of the base matrix $M$ are chosen from $\set{-1,1}$ uniformly at random; the case of 
Gaussian 
entries also gives a similar though slightly weaker result, but would require some extra work.

We define a notion of $d$'th power of a matrix $M \in \R^{n \times m}$ which is similar to the Khatri--Rao power except that we only
keep the non-redundant multilinear part resulting in ${n \choose d} \times m$ matrix. Working with 
this multilinear part will simplify things. Formally,  
$M^{\ominus d} := [M_1^{\ominus d}, \ldots, M_m^{\ominus d}]$, where for a column vector $C \in \R^n$, define 
$C^{\ominus d} \in \R^{n \choose d}$ with entries given by $C_S := C_{i_1} C_{i_2} \ldots C_{i_d}$ where 
where $1 \leq i_1 < i_2 < \ldots < i_d \leq n$ and $S = \set{i_1, \ldots, i_d} \in {[n] \choose d}$.

The following theorem is stated for the case when the base matrix $M \in \R^{n \times n^2}$. This choice is  
to keep the statement and proof of the theorem simple; generalization to more general parameterization is 
straightforward. While the theorem below is proved for submatrices $M^{\ominus d}$ of the Khatri--Rao power 
$M^{\odot d}$, similar results hold for $M^{\odot d}$ by the interlacing properties of the singular values of 
submatrices~\cite{Thompson}.
\begin{theorem} \label{thm:vershynin_condition}
Let $M \in \R^{n \times m}$ be chosen by sampling each entry iid uniformly at random from $\set{-1,1}$. For $m = n^2$,
integer $d \geq 3$, and $N = {n \choose d}$, and $A = M^{\ominus d}$ we have 
\begin{align*}
\nE \max_{j \leq n^2} \abs{\sigma_j(A) - \sqrt{N}} < N^{1/2-\Omega(1)}.
\end{align*}
\end{theorem}

\begin{proof}

We are going to use Theorem 5.62 of Vershynin~\cite{VershyninSingular} which we state here essentially verbatim:
\begin{theorem}[\cite{VershyninSingular}] \label{thm:VershyninSingular}
Let $A$ be an $N \times m$ matrix ($N \geq m$) whose columns $A_j$ are independent isotropic random vectors in $\R^N$
with $\norm{A_j}_2 = \sqrt{N}$ almost surely. Consider the incoherence parameter
\begin{align*}
\mu := \frac{1}{N} \nE \max_{j \leq m} \sum_{k \in [m], k \neq j} \angles{A_j, A_k}^2. 
\end{align*}
Then for absolute constants $C, C_0$ we have 
$\nE \norm{\frac{1}{N}A^*A-I} \leq C_0 \sqrt{\frac{\mu \log m}{N}}$. In particular, 
\begin{align*}
\nE \max_{j \leq m} \abs{\sigma_j(A) - \sqrt{N}} < C \sqrt{\mu \log m}.
\end{align*} 
\end{theorem}

Our matrix $A = M^{\ominus d}$ will play the role of matrix $A$ in Theorem~\ref{thm:VershyninSingular}. 
Note that for a column $A_j$ we have
$\nE A_j \otimes A_j = I$, so the $A_j$ are isotropic. Also note that 
$\norm{A_j}_2 = \sqrt{N}$ always. 

We now bound the incoherence parameter $\mu$. To this end, we first prove a concentration bound for 
$\angles{A_j, A_k}$, for fixed $j, k$. We use a concentration inequality for polynomials of random variables. 
Specifically, 
we use Theorem 23 at (http://www.contrib.andrew.cmu.edu/\mytilde ryanod/?p=1472). Let us restate that theorem here.

\begin{theorem} \label{thm:hypercontractivity}
Let $f : \set{-1,1}^n \to \R$ be a polynomial of degree at most $k$. Then for any $t \geq (2e)^{k/2}$ we have 
\begin{align*}
\Pr_{x \sim \set{-1,1}}[\abs{f(x)} \geq t \norm{f}_2] \leq \expn{-\frac{k}{2e} t^{2/k}}.
\end{align*}
\end{theorem}

Here $\norm{f}_2 := [\nE_x f(x)^2]^{1/2}$. For our application to $\angles{A_j, A_k}$, we first fix $A_j$ 
arbitrarily. Then $\angles{A_j, A_k}$, which will play the role of $f(x)$ in the above theorem, can be 
written as $\sum_{S \in {[n] \choose d}}  c_S x_S$ where the choice of the coefficients $c_S = \pm 1$ comes
from the fixing of $A_j$ and the entries of $A_k$ are of the form $x_S$, where $S \in {[n] \choose d}$. Now 

\begin{align*}
\nE_x \angles{A_j, A_k}^2 
&= \sum_{S, S' \in {[n] \choose d}} c_S c_{s'} \nE_x x_S x_{S'} \\ 
&= \sum_{S \in {[n] \choose d}} c_S^2 \nE_x {x_S^2} + \sum_{S, S' \in {[n] \choose d}: S \neq S'} c_S c_{S'} \nE_x {x_S x_{S'}} \\
&= {n \choose d} \\
&= N.
\end{align*}

In other words, for our choice of $f$ we have $\norm{f}_2 = \sqrt{N}$.

Applying Theorem~\ref{thm:hypercontractivity} with $t \geq (2e)^{d/2}$ and $\lambda = t \sqrt{N}$ we have
\begin{align} \label{eqn:bad_event}
\Pr_{x \sim \set{-1,1}}[\abs{\angles{A_j,A_k}} \geq \lambda] \leq \expn{-\frac{d}{2e} t^{2/d}} 
= \expn{-\frac{d}{2e} \frac{\lambda^{2/d}}{N^{1/d}}}.
\end{align}

Note that we proved the above inequality for any fixed $A_j$, so clearly it also follows when $A_j$ is also 
random. 


We now estimate parameter $\mu$. Note that $\angles{A_j, A_k}^2 \leq N^2$ always. When the union of the event in
\eqref{eqn:bad_event} over all $j \neq k$, which we denote by $B$, does not hold, we will use the bound just 
mentioned. 
For the following computation recall
that the number of columns $m$ in $A$ is $n^2$. 

\begin{align}
\mu &\leq \frac{1}{N} m \lambda^2 \prob{\bar{B}} + \frac{1}{N} m N^2 \prob{B} \nonumber \\
&\leq \frac{m \lambda^2}{N} + \frac{m {m \choose 2} N^2}{N} \expn{-\frac{d}{2e} \frac{\lambda^{2/d}}{N^{1/d}}} \nonumber \\
&\leq \frac{n^2 \lambda^2}{N} + n^6 N \expn{-\frac{d}{2e} \frac{\lambda^{2/d}}{N^{1/d}}}. \label{eqn:mu_bound}
\end{align}

Now choose $\lambda := N^{1/2 + \epsilon}$ for a small $\epsilon > 0$. Then the expression in \eqref{eqn:mu_bound} 
is bounded by 
\begin{align*} 
\eqref{eqn:mu_bound} \leq n^2 N^{2 \epsilon} + n^6 N \expn{-\frac{d}{2e} N^{2\epsilon/d}}.
\end{align*}

It's now clear that for a sufficiently small choice of $\epsilon$ (say $0.05$) and sufficiently large 
$n$ (cepending on $d$ and $\epsilon$), only the first term above is significant and using our assumption
$d > 2$ gives

\begin{align*}
\mu < 2 n^2 N^{2 \epsilon} < 2 d! N^{2/d + \epsilon} << N.
\end{align*}

Therefore by Theorem~\ref{thm:VershyninSingular} we have
\begin{align*}
\nE\norm{\frac{1}{N}A^*A -I} \leq C_0 \sqrt{\frac{\mu \log n^2}{N}} < 1/N^{\Omega(1)},
\end{align*}

which gives 
\begin{align*}
\nE \max_{j \leq n^2} \abs{\sigma_j(A) - \sqrt{N}} < N^{1/2-\Omega(1)}.
\end{align*}

\end{proof}

In particular, setting $s_{\min}(A) := s_{n^2}(A)$ we have 
\begin{align*}
\nE \abs{\sigma_{\min}(A) - \sqrt{N}} < 1/N^{1/2-\Omega(1)}.
\end{align*}

Using Markov this also gives probability bounds.

\section{Technical lemmas}
In this section we collect some of the technical claims needed in the paper. 

\begin{lemma}[Nonvanishing of $\phi(t)$] \label{lem:phi_nonvanishing}
Let $s$ be a real-valued random vector in $\R^m$ with independent components and $\E{s}=0$. Also let
$\E{\abs{s_j}}$ and $\E{\abs{s_j^2}}$
exist and $\E{\abs{s_j^2}} \leq M_2$ for all $j$ for $M_2 >0$. Then for $t \in \R^m$ with 
$\norm{t}_2 \leq \frac{1}{2\sqrt{M_2}}$ the characteristic function $\phi(\cdot)$ of $s$ satisfies $|\phi(t)| \geq 3/4$.  
\end{lemma}
\begin{proof}
Using Taylor's theorem~\ref{thm:taylor} for $\cos{y}$ and $\sin{y}$ gives
\begin{align*}
e^{iy} = \cos{y} + i \sin{y} = 1 + iy - \frac{(iy)^2}{2!}[\cos{(\theta_1 y)} + i \sin{(\theta_2 y)}],
\end{align*}
for $y, \theta_1, \theta_2 \in \R$ with $\abs{\theta_1} \leq 1, \abs{\theta_2} \leq 1$. Applying this to $y=t^Ts$,
taking expectation over $s$, and using the assumption of zero means on the $s_i$ we get
\begin{align*}
  \E{e^{it^Ts}} = 1 - \E{\frac{(i t^Ts)^2}{2}[\cos{(\theta_1 y)} + i \sin{(\theta_2 y)}]}, 
\end{align*}
which using the indpendence of the components of $s$ and the zero means assumption gives
\begin{align*}
\abs{\E{e^{it^Ts}}-1} = \abs{\phi(t)-1} 
&\leq \frac{1}{2} \E{(t^Ts)^2 \abs{\cos{(\theta_1 y)} + i \sin{(\theta_2 y)}}} \\
&\leq \E{(t^Ts)^2} \\
&= \sum_j t_j^2 \E{s_j^2} \\
&\leq R_2 \norm{t}_2^2 \\
&\leq 1/4.
\end{align*}
\end{proof}

\begin{lemma} \label{lem:prod_difference}
Let $a_1, \ldots, a_d, b_1, \ldots, b_d \in \C$ be such that $\abs{a_j - b_j} \leq \epsilon$ for real $\epsilon \geq 0$,
and $\abs{a_j} \leq R$ for $R > 0$. 
Then 
\begin{align*}
\abs{\prod_{j=1}^d a_j - \prod_{j=1}^d b_j} \leq (R+\epsilon)^d-R^d.
\end{align*}
\end{lemma}
\begin{proof}
For $0 < j < d$, define the $j$th elementary symmetric function in $d$ variables: 
$\sigma_j(x_1, \ldots, x_d) = \sum_{1 \leq i_1 \leq \ldots \leq i_j \leq d} x_{i_1}\ldots x_{i_j}$. We will use the 
following well-known inequality (see, e.g., \cite{Steele}) which holds for $x_\ell \geq 0$ for all $\ell$.
\begin{align} \label{eq:symmetric_means}
\left(\frac{\sigma_j(x_1, \ldots, x_d)}{{d \choose j}}\right)^{1/j} \leq \frac{\sigma_1(x_1, \ldots, x_d)}{d}.
\end{align}

Let $b_j = a_j + \epsilon_j$. Then
\begin{align*}
\abs{\prod_j (a_j+\epsilon_j)-\prod_j a_j} 
&\leq \epsilon\:\sigma_{d-1}(\abs{a_1}, \ldots, \abs{a_d}) + \epsilon^2 \sigma_{d-2}(\abs{a_1}, \ldots, \abs{a_d}) 
+ \ldots + \epsilon^{d-1}\sigma_1(\abs{a_1}, \ldots, \abs{a_d}) \\
&\leq d \epsilon R^{d-1} + {d \choose 2} \epsilon^2 R^{d-2} + \ldots + \epsilon^d \\
&= (R+\epsilon)^d - R^d,
\end{align*}
where the second inequality follows from \eqref{eq:symmetric_means}.
\end{proof}

\begin{claim} \label{claim:gaussian_concentration}
Let $u \in \R$ be sampled according to $N(0, \sigma^2)$. Then for $\tau > 0$ we have
\begin{align*}
\prob{\abs{u} > \tau} \leq \sqrt{\frac{2}{\pi}}\frac{\sigma^2}{\tau} e^{-\frac{\tau^2}{2 \sigma^2} } 
\end{align*}
\end{claim}
\begin{proof}
Follows from the well-known fact:
$\frac{1}{\sqrt{2\pi}}\int_a^\infty e^{-z^2/2}dz \leq \frac{1}{\sqrt{2\pi}}\cdot \frac{1}{a}\cdot e^{-a^2/2}$, for
$a > 0$
\end{proof}

We state the following easy claim without proof.
\begin{claim} \label{claim:singular_value_inequality_product}
Let $B \in \C^{p \times m}$ with $p \geq m$ and $\colspan{B}=m$. Let $D \in \C^{m \times m}$ be a diagonal matrix. Then
\begin{align*}
\sigma_m(BDB^T) \geq \sigma_m(B)^2 \sigma_m(D).
\end{align*}
\end{claim}

\begin{claim} \label{claim:error_matrix_inversion}
For $E \in \C^{m \times m}$ with $\norm{E}_F < 1/2$ we have 
\begin{align*}
(I-E)^{-1} = I + E + R,
\end{align*}
where $\norm{R}_F < m \norm{E}_F $.
\end{claim}

\begin{proof}
For $\norm{E}_F <  1/2$ we have 
\begin{align*}
(I-E)^{-1} = I + E + E^2 + \ldots. 
\end{align*} 
Hence 
\begin{align*}
\norm{(I-E)^{-1}-(I+E)}_F \leq \norm{E^2}_F \norm{(I-E)^{-1}}_F < m \norm{E}_F.
\end{align*}
\end{proof}

\begin{fact} \label{fact:holder}
For a real-valued random variable $x$ and for any $0 < p \leq q$ we have
\begin{align*}
\E{\abs{x}^p}^{1/p} &\leq \E{\abs{x}^q}^{1/q}, \\
\E{\abs{x}^p} \E{\abs{x}^q} &\leq \E{\abs{x}^{p+q}}.
\end{align*}
\end{fact}
\begin{proof}
H\"older's inequality implies that for $0 \le p \le q$ we have
  \begin{align*}
    \E{\abs{x}^p}^{1/p} \le \E{\abs{x}^q}^{1/q},
  \end{align*}
  and hence 
  \begin{align*}
    \E{\abs{x}^p} \E{\abs{x}^q} \le \E{\abs{x}^{p+q}}^{p/(p+q)} \E{\abs{x}^{p+q}}^{q/(p+q)}  
= \E{ \abs{x}^{p+q}}.
  \end{align*}
\end{proof}

\section{Conclusion}
We conclude with some open problems. (1) Our condition for
ICA to be possible required that there exist a $d$ such that $A^{\odot d}$ has full column 
rank. As mentioned before, the existence of such a $d$ turns out to be equivalent to 
the necessary and sufficient condition for ICA, namely, any two columns of $A$ are linearly
independent. Thus if $d$ is large for a matrix $A$ then our algorithm whose running time
is exponential in $d$ will be inefficient. This is inevitable to some extent as suggested
by the ICA lower bound in \cite{AndersonGMM}. However, the lower bound there requires that 
one of the $s_i$ be Gaussian. Can one prove the lower bound without this requirement?
(2) Give an efficient algorithm for independent subspace analysis. This is the problem where the $s_i$ are not all indendent but rather the set of
indices $[m]$ is partitioned into subsets. For any two distinct subsets $S_1$ and $S_2$ in
the partition $s_{S_1}$ is independent of $s_{S_2}$, where $s_{S_1}$ denotes the vector of the
$s_i$ with $i \in S_1$ etc. Clearly this problem is a generalization of ICA.

\bigskip
\noindent {\bf Acknowledgements.} We thank Sham Kakade for helpful
discussions and Yongshun Xiao for showing us the Gershgorin Circle
theorem and the continuous deformation technique used in
Lemma \ref{thm:weyl}.

\bibliographystyle{abbrv}
\bibliography{ICA_bibliography}

\end{document}